\documentclass[11pt]{article}

\usepackage{color}
\usepackage{fullpage}
\usepackage{wrapfig}
\usepackage{algorithm}[0.1]
\usepackage{algorithmic}
\usepackage{multirow}
\usepackage{tabularx}
\usepackage{array}
\usepackage{natbib}
\usepackage{amsmath,amsfonts}
\usepackage{graphicx}
\usepackage{rotating}
\usepackage{multirow}
\usepackage{smile}

\usepackage[usenames,dvipsnames,svgnames,table]{xcolor}
\usepackage[colorlinks=true,
            linkcolor=blue,
            urlcolor=blue,
            citecolor=blue]{hyperref}

\numberwithin{equation}{section}
\numberwithin{theorem}{section}
\newcommand{\A}{\mathcal{A}}

\newcommand{\tr}{\mathrm{tr}}

\newcommand{\M}{\mathcal{M}}
\renewcommand{\S}{\mathcal{S}}
\renewcommand{\A}{\mathcal{A}}
\newcommand{\G}{\mathcal{G}}
\newcommand{\F}{\mathcal{F}}

\renewcommand{\L}{\mathcal{L}}
\newcommand{\N}{\mathcal{N}}

\newcommand{\R}{\mathbb{R}}

\newcommand{\E}{\mathbb{E}}
\newcommand{\Pe}{\mathbb{P}}
\renewcommand{\P}{\mathcal{P}}

\newcommand{\I}{\mathbb{I}}

\newcommand{\Z}{\mathbb{Z}}

\newcommand{\norm}[1]{\left\|#1\right\|}

\newcommand{\bre}{\mathrm{br}}

\newcommand{\wo}{\overline{w}}
\newcommand{\wu}{\underline{w}}
\newcommand{\Vo}{\overline{V}}
\newcommand{\Vu}{\underline{V}}

\newcommand{\dist}{\operatorname{dist}}

\DeclareMathOperator*{\diag}{diag}
\newtheorem{proposition}{Proposition}[section]

\usepackage[title]{appendix}

\newcommand{\KL}{\textrm{KL}}

\title{
One Policy is Enough: Parallel Exploration with a Single Policy is Near-Optimal for Reward-Free Reinforcement Learning
}
\author{Pedro Cisneros-Velarde${}^{\star,\, 1}$\quad Boxiang Lyu${}^{\star,\,2}$ \quad Sanmi Koyejo${}^{3, 4}$ \quad Mladen Kolar${}^{2}$\\
\\
$^{1}${University of Illinois at Urbana-Champaign} \quad $^{2}${The University of Chicago}\\
  $^{3}${Stanford University} \quad $^{4}${Google Research}
}

\date{}

\begin{document}

\maketitle

\begin{abstract}
Although parallelism has been extensively used in reinforcement learning (RL), the quantitative effects of parallel exploration are not well understood theoretically. We study the benefits of simple parallel exploration for reward-free RL in linear Markov decision processes (MDPs) and two-player zero-sum Markov games (MGs). In contrast to the existing literature, which focuses on approaches that encourage agents to explore a diverse set of policies, we show that using a single policy to guide exploration across all agents is sufficient to obtain an almost-linear speedup in all cases compared to their fully sequential counterpart. Furthermore, we demonstrate that this simple procedure is near-minimax optimal in the reward-free setting for linear MDPs. From a practical perspective, our paper shows that a single policy is sufficient and provably near-optimal for incorporating parallelism during the exploration phase.
\end{abstract}

\let\thefootnote\relax\footnotetext{$^{\star}$ Equal Contribution. Names are listed in alphabetical order.}

\section{Introduction}
\label{sec:introduction}

Parallel methods in deep RL have been successfully applied to problems ranging from Go~\citep{alphaGo:16, alphaZero:17} and Starcraft~\citep{vinyals2019grandmaster} to Atari games~\citep{nair2015Massively} and unmanned aerial vehicles~\citep{pham2018cooperative}. In the episodic case, parallel methods use $P$ agents, each simultaneously interacting with the environment for $K$ episodes. At the beginning of each episode, each agent is assigned an exploration policy that it uses to collect a trajectory from the environment. Agents update their policies as they explore and learn in the environment, with aggregation and synchronization mechanisms that calculate the exploration policy for the next episode. For example, \cite{nair2015Massively} synchronized the value function estimates stored on each agent and used the same value function to guide subsequent exploration. As a result of having multiple agents collecting data simultaneously and exploring the environment, parallel methods are faster than their single-agent or sequential counterparts. They can learn near-optimal policies in a relatively short amount of time---only $K$ rounds are needed to collect and use a total of $KP$ trajectories. While synchronization of agents' policies is often used to ensure unbiased gradient estimates for the neural networks used to represent the value function, synchronization of the exploration policies is not necessary for unbiased gradient estimates. Alternative approaches have been proposed to improve the coordination of multiple agents for exploration in parallel RL.

Importantly, much of the literature has focused on the diversity of exploration policies---plausibly as a necessary condition for efficient exploration. \cite{dimakopoulou2018coordinated} proposed multiple sampling-based algorithms that provide agents with a diverse set of exploration policies. \cite{mahajan2019maven} proposed the use of mutual information to ensure that agents explore with a diverse set of policies---we remark that~\cite{mahajan2019maven} focused on the multi-agent RL (MARL) case and not on parallel exploration in RL. At a high level, these alternatives ensure that the agents' exploration policies are sufficiently different from one another and argue that the diversity speeds up the learning process.

Taken together, existing theory and practice motivate a central question: \emph{is the diversity of exploration (by different policies) always required for efficient parallel exploration in RL?} By \emph{efficient exploration}, we refer to exploration that results in a speed-up of the learning process. 

We consider this question in the \emph{reward-free} RL setting. This setting allows us to isolate the benefits of parallelism in exploration, as it separates the collection of trajectories from the learning of an optimal policy~\citep{RW-SD-LY-RS:20, SQ-JY-ZW-ZY:22, AW-YC-MS-SSD-KJ:22}. Thus, the reward-free setting provides a level playing field for comparing exploration procedures. Therefore, we can focus on studying which procedure explores the environment more efficiently, without requiring the procedure to balance the trade-off between exploration and exploitation---a challenge inherent in the online RL setting~\citep{CJ-ZY-ZW-MIJ:20}.

We provide a surprising negative answer to the central question by showing that \emph{diversity is not necessary for efficient parallel exploration}. In summary, we show that no sophisticated coordination is necessary and that a simple algorithm in which all agents use the same exploration policy to collect data is sufficient for efficient parallel exploration. Specifically, we prove that when the number of parallel agents is no larger than the number of rounds we interact with the environment, using a single policy to guide exploration is sufficient to achieve a near-minimax optimal rate. Even outside of this regime, we prove that our method has an almost-linear speedup in learning compared to a single agent setting. 

We demonstrate the surprising effectiveness of our proposed approach in solving both MDPs and two-player zero-sum Markov games (MGs). The latter is widely regarded as one of the simplest forms of multi-agent reinforcement learning (MARL)~\citep{zhang2021multi}, which has resulted in its increased popularity in recent years \citep{zhang2020model, chen2022almost, kozuno2021learning, zhu2020online}. Our results suggest that the success of our parallelization method is not dependent on the simplicity of the MDP, and can be extended to more complex MARL problems.

The paper is organized as follows. In Section~\ref{sec:preliminaries}, we formally introduce the setting. To provide insight into both the construction and analysis of our proposed algorithms for the reward-free setting, we consider the online RL setting in Section~\ref{sec:parallel_online_RL}. In Section~\ref{sec:paral-rf-sec}, we describe the algorithms used for reward-free exploration in both MDPs and two-player zero-sum MGs, and analyze their theoretical properties. 

\paragraph{Contributions} We summarize our contributions and their practical implications.
\begin{itemize}
    \item We introduce the first parallel reward-free RL algorithms for both linear MDPs and MGs, inspired by a straightforward parallel online RL algorithm. Specifically, all algorithms----both online and reward-free---are simple in nature and use the same exploration policy for all agents. Compared to their sequential counterparts, our algorithms demonstrate an almost-linear speedup in performance.
    
    \item We provide an information-theoretic lower bound on the ability of any parallel reward-free RL algorithm to achieve $\epsilon$-suboptimality. By comparing the lower and upper bounds, we demonstrate that our algorithm is nearly minimax optimal when the number of agents is $P = \cO(K)$ for $K$ episodes. We emphasize that this assumption accurately reflects real-world use cases where the number of collected trajectories is often much larger than the number of available agents~\citep{pham2018cooperative, alphaGo:16, alphaZero:17, vinyals2019grandmaster}. With regard to near-optimality, our upper bound fully matches the speed-up term from the lower bound (i.e., in $K$ and $P$) and is only greater than the lower bound by a low-degree polynomial factor on the other problem parameters. Consequently, even if different exploration policies are used for each agent, their performance can at most match the rate of our single-policy algorithm in terms of $K$ and $P$, which are arguably two of the most important parameters in parallel exploration. For a more detailed discussion, see Section~\ref{subsec:lower-bound}.
    
    \item In terms of practical implications, our work justifies the simple approach of employing the same exploration policy across all agents in parallel RL. This eliminates the necessity for complex parallel systems that result in additional computational and communication overhead required to explore with a diverse set of policies.
\end{itemize}

%================

\subsection{Related Works}

Our paper is related to the literature on reward-free, multi-agent, distributed, and deployment-efficient RL, as well as the literature on parallel and federated bandits.

{\par \textbf{Reward-free RL}} Reward-free RL addresses the problem initially proposed in~\citet{jin2020reward} in which agents attempt to explore the environment without knowledge of their individual reward functions. The setting focuses only on the exploration capabilities of different algorithms and serves as a level playing field when comparing different exploration strategies.~\citet{jin2020reward} studied reward-free exploration in tabular MDPs and~\citet{RW-SD-LY-RS:20} proposed a provably efficient reward-free exploration algorithm for linear MDPs.~\citet{SQ-JY-ZW-ZY:22} proposed a reward-free exploration algorithm when using a kernel function approximation that is efficient for both MDPs and two-player zero-sum MGs. More recently,~\citet{AW-YC-MS-SSD-KJ:22} provided a tighter reward-free exploration algorithm that matches the sample complexity of PAC RL for linear MDPs, while \citet{agarwal2020flambe} studied reward-free RL with unknown feature representation in low-rank MDPs.

{\par \textbf{Multi-agent RL}} Our paper is related to cooperative MARL~\citep{boutilier1996planning}. In the online cooperative setting, \citet{agarwal2021communication} proposed a communication efficient algorithm for tabular MDPs, and~\citet{AD-AP:21} derived a provably efficient algorithm in the linear MDP setting, allowing agents' reward functions to be heterogeneous. We note that the online RL setting is not suitable for isolating the effectiveness of exploration alone, as it requires exploration strategies to balance both exploration and exploitation.~\citet{zhang2019distributed, lin2019communication, suttle2020multi} studied distributed variants of the actor-critic algorithm~\citep{konda1999actor}, but did not explicitly characterize the benefits that arise from parallel exploration. Our work also contributes to the literature on two-player zero-sum MGs, which is another example of MARL~\citep{perolat2015approximate, zhao2021provably, zhu2020online, kozuno2021learning, chen2022almost}.

{\par \textbf{Parallel and Federated Bandits}} \citep{karbasi2021Parallelizing, JC-AP-NT-YSS-PB-MIJ:21} studied parallelized learning in contextual bandits, where the decision maker receives a batch of bandit feedback with a predetermined size at the beginning of each round. \citep{zhu2021federated,huang2021federated,shi2021federated,dubey2020differentially} focused on federated bandits and efficient parallel algorithms adapted to the federated learning setting.

{\par \textbf{Concurrent RL}} 
The works~\citep{bai2019provably} and~\citep{ZZ-ZY-XJ:20} also analyze the setting in which parallel agents can perform exploration with a single policy. They provide almost-linear speedup conditions, but their setting differs from ours: they only focus on the tabular case for online RL and do not study MGs. We note, however, that they are more interested in the problem of reducing the updating frequency of policies in RL.

{{\par \textbf{Deployment Efficient RL}}} Drawing inspiration from bandit learning with low switching costs~\citep{auer2002finite, cesa2013online}, a recent line of work has examined efficient algorithms for RL when the number of times the exploration policy changes is restricted.~\citet{bai2019provably} proposed a low switching cost $Q$ learning algorithm for tabular MDPs, and~\citet{MG-RX-SSD-LFY:21} developed a provably efficient low switching cost algorithm for linear MDPs. \citet{huang2021deployment} provided a mathematically rigorous definition for deployment efficiency, which is inherently linked to the concept of low switching cost, and presented two provably efficient algorithms for deployment efficient reward-free RL. As we will discuss later, the algorithms attain suboptimal sample complexity.

{\par \textbf{Transfer learning}} The work~\citep{tuynman2022TransferRL} examines the setting in which multiple agents explore a copy of the same tabular MDP, except that the reward may vary. They demonstrate that sharing trajectories among agents results in better total regret performance than each agent learning without sharing.

%%% Local Variables:
%%% mode: latex
%%% TeX-master: "main"
%%% End:\label{eq:function_class-MG}
%

\subsection{Notation} 
Let $\|\cdot\|$ be the Euclidean norm, and $\|v\|_A = \sqrt{v^TAv}$ for a positive semidefinite matrix $A$. Let $\preceq, \succeq$ be the matrix Loewner order. For a positive integer $k$, let $[k]=\{1, 2, \dots, k\}$. Let $I_m$ be the $m \times m$ identity matrix and let $\I[\,\cdot\,]$ be the indicator function. Let $\Delta(\A)$ be the probability simplex defined on a given finite set $\A$. We define the clipping operator as $\Pi_{[0,a]}[b]:= \min\{b, a\}^+ = \min\{\max\{b, 0\}, a\}$ for any $a>0$ and $b\in\R$. Given the big-O complexity notation $\cO$, we use $\tilde{\cO}$ to hide polylogarithmic terms in the quantities of interest.

%==================================================
%==================================================

\section{Preliminaries}
\label{sec:preliminaries}

{\par \textbf{Markov Decision Processes.}} An episodic Markov decision process (MDP) takes the form $\M=(\S,\A,H,\P,r)$ where $\S$ is the state space, $\A$ is the action space, $H\in\Z_{\geq 1}$ is the length or number of steps of each episode, $\P=\{\P_h\}_{h=1}^H$ are the transition probability measures where $\P_h(\cdot|x,a)$ is the transition kernel over the next states if action $a\in\A$ is taken for state $x\in\S$ at step $h\in[H]$, and $r=\{r_h\}_{h=1}^H$ with $r_h:\S\times\A\to[0,1]$ are the reward functions. For simplicity of presentation, we assume deterministic rewards, but remark that random rewards, as long as they are bounded, do not affect our results. We assume that $\A$ is a finite set of actions and the non-bandit case $H\geq 2$. A stochastic policy $\pi = \{\pi_h\}_{h = 1}^H$ consists of functions $\pi_h:\S\to \Delta(\A)$ that map the state $x\in\S$ at step $h$ of the MDP to a distribution over actions. The value function $V_{h}^{\pi}: \mathcal{S} \to \mathbb{R}$ is $V_{h}^{\pi}(x;r) = \mathbb{E}_{\pi}\left[\sum_{h' = h}^Hr_{h'}(s_{h'}, a_{h'}) | s_{h'} = x\right]$ and the action-value function or $Q$-function $Q_{h}^{\pi}: \mathcal{S} \times \mathcal{A} \to \mathbb{R}$ is $Q_{h}^{\pi}(x, a;r) = \mathbb{E}_{\pi}\left[\sum_{h' = h}^Hr_{h'}(s_{h'}, a_{h'}) | s_{h'} = x, a_{h'} = a)\right]$, where the expectation $\mathbb{E}_{\pi}$ is taken with respect to both the randomness in the transitions $\P$ and the randomness inherent in the policy $\pi$. When the given reward function is understood from the context, we will omit it from the notation of the value and action-value functions.

For any function $f: \mathcal{S} \to \mathbb{R}$, we define the transition operator as $(\mathbb{P}_{h}f)(x, a) =\E_{x'\sim\P_h(\cdot|x,a)}[f(x')]$ and the Bellman operator as $(\mathbb{B}_{h}f)(x, a) = r_{h} +  (\mathbb{P}_{h}f)(x, a)$ for each step $h \in [H]$. The Bellman equation associated with a policy $\pi$ is: $Q^\pi_h(x,a)=(r_h(x,a)+\Pe_h V^\pi_{h+1})(x,a)$, $V^\pi_h(x)= \E_{a\sim \pi_h(x)}[Q^\pi_h(x,a)]$, $V^\pi_{H+1}(x)=0$, for any $(x,a)\in\S\times\A$.
Let $\pi^* = \argmax_{\pi} V^\pi_1(x)$, $V^*_h=V^{\pi^*}_h(x)$, and $Q^*_h(x, a) = Q^{\pi^*}_h(x, a)$.

{\par \textbf{Zero-sum Markov Games.}}
An episodic zero-sum Markov game (MG) takes the form $\mathcal{M}_G=(\S, \A, \cB,  H, \P, r)$ where $\S$ is the state space, $\A$ is the action space for Player 1, $\cB$ is the action space for Player 2, $H$ is the number of steps per episode, $\P = \{\P_h\}_{h\in[H]}$ are transition probability measures and $\P_h(\cdot\mid x,a, b)$ denotes the transition kernel over the next step if Player 1 takes action $a\in\A$ and Player 2 takes action $b\in\cB$ for state $x\in\S$ at step $h\in[H]$, and reward functions $r=\{r_h\}_{h\in[H]}$ with $r_h : \S \times \A \times \cB \to [0, 1]$. We assume that both action spaces $\A$ and $\cB$ are finite sets. We denote Player 1's policy by $\pi = \{ \pi_h\}_{h=1}^H$ with $\pi: \cS \to \Delta(\A)$ and Player 2's by $\nu = \{ \nu_h\}_{h=1}^H$ with $\nu: \cS \to \Delta(\cB)$. We define the value function $V_h^{\pi, \nu}: \S \to \R$ at the $h$-th step as $V_h^{\pi, \nu}(x; r) = \E_{\pi,\nu}[\sum_{h'=h}^H r_{h'} (s_{h'}, a_{h'}, b_{h'})  \given s_h = x]$ and the Q function $Q_h^{\pi, \nu}: \S \times \A \times \cB \to \R$ as $Q_h^{\pi, \nu}(x,a,b; r):=\E_{\pi,\nu}[ \sum_{h'=h}^H r_{h'}(s_{h'}, a_{h'}, b_{h'} )  \given s_h = x, a_h = a, b_h = b]$. The Nash equilibrium (NE) is defined as any pair of policies $(\pi^\dag, \nu^\dag)$ that is a solution to $\max_{\pi\in\Delta(\A)} \min_{\nu\in\Delta(\cB)} V_1^{\pi, \nu}(x;r)$ for a given initial state $x$ --- the $\max$ and $\min$ can be interchanged~\citep{shapley1953SG}.  For simplicity, we let $V_h^{\dagger}(x; r) = V_h^{\pi^\dag, \nu^\dag}(x; r)$ and $Q_h^{\dagger}(x,a,b; r) = Q_h^{\pi^\dag, \nu^\dag}(x,a,b; r)$ denote the value function and Q-function under the NE $(\pi^\dag, \nu^\dag)$ at step $h\in[H]$. Given a state $x\in\S$, the best response against Player 1 with policy $\pi$ is defined as $\bre_2(\pi):=\argmin_\nu V_1^{\pi, \nu}(x; r)$ and the one against Player 2 with policy $\nu$ is defined as $\bre_1(\nu):=\argmax_\pi V_1^{\pi, \nu}(x; r)$. Note that $V_1^{\bre_1(\tilde{\nu}), \tilde{\nu}}(x; r) \geq V_1^{\dag}(x; r)  \geq V_1^{\tilde{\pi}, \bre_2(\tilde{\pi})}(x; r)$ holds for any policies $\tilde{\pi}$ and $\tilde{\nu}$ and $x\in\S$. We also introduce the notation $V^*_1(s; r) = \max_{\pi, \nu} V^{\pi, \nu}_1(s; r)$ and the associated optimal value function and optimal Q-function for the $h$-th step as $V^*_h(x; r)$ and $Q^*_h(x, a, b; r)$.

{\par \textbf{Linear MDP.}} Under a linear MDP setting,  there exists a known feature map $\phi: \mathcal{S} \times \mathcal{A} \to \mathbb{R}^d$ such that for every $h\in[H]$, there exist $d$ unknown (signed) measures $\mu_{h} = \left(\mu_{h}^{(1)}, \dots \mu_{h}^{(d)}\right)$ over $\mathcal{S}$ and an unknown vector $\theta_{h} \in \mathbb{R}^d$ such that $\P_{h}(x'| x, a) = \langle \phi(x, a), \mu_{h}(x')\rangle$, $r_{h}(x, a) = \langle \phi(x, a), \theta_{h}\rangle$ for all $(x, a, x') \in \mathcal{S} \times \mathcal{A} \times \mathcal{S}$. We assume the non-scalar case with $d\geq 2$ and that the feature map satisfies $\|\phi(x, a) \|\leq 1$ for all $(x, a) \in \mathcal{S} \times \A$ and $\max\{\|\mu_{h}(\S)\|, \|\theta_{h}\|\} \leq \sqrt{d}$ at each step $h \in [H]$, where (with an abuse of notation) $\|\mu_{h}(\mathcal{S})\| = \int_{\mathcal{S}}\|\mu_{h}(x)\|dx$. Note that the transition kernel $\P_h(\cdot|x,a)$ may have infinite degrees of freedom since the measure $\mu_h$ is unknown.

{\par \textbf{Linear MG.}} The linear MG is defined in a similar way to the linear MDP by simply adding an additional argument to the feature map $\phi$ to include the effect of the actions taken by the second player. 

%======

{\par \textbf{Parallelism in the Online Case and Performance Metric.}} Consider the case where we have $P$ agents living in an identical episodic Markov decision process and exploring it in a total of $K$ episodes. We assume that all agents start at some $s_0 \in \cS$ at the beginning of each episode. Given an agent $p\in[P]$ at a step $h\in[H]$ of an episode $k\in[K]$, let $x_h^{k,p}$ be the state, $a_h^{k,p}$ be the action taken according to some policy $\pi_h^{k,p}$, and $r_h^{k,p}:=r_h(x_h^{k,p},a_h^{k,p})$ be the reward obtained. Let $\pi^{k,p}:=\{\pi^{k,p}_h\}_{h=1}^H$. 
For a set of policies $\{\pi^{k,p}\}_{k\in[K],p\in[P]}$ provided by an online RL algorithm, performance is measured by parallel regret
\begin{equation}
    \label{eq:regret_parallel}
    \textnormal{Regret}(K,P) = \sum_{p=1}^P\sum_{k=1}^K (V_1^{\pi^*}(s_0) - V_1^{\pi^{k,p}}(s_0)).
\end{equation}

{\par \textbf{Parallelism in the Reward-free Case for MDP and its Performance Metric.}} In the \emph{exploration phase} of the reward-free RL problem, we consider the case where we have $P$ agents, each of whom live in an identical episodic Markov decision process, collecting trajectories in a total of $K$ episodes.  We assume that all agents start at some $s_0 \in \cS$ at the beginning of each episode. We use the same notation as introduced for the online case, with the difference that the reward function will be computed by the algorithm itself. In the \emph{planning phase}, a central server receives all the trajectories collected by the agents and, given a reward function $r=\{r_h\}_{h\in[H]}$ provided by the user, computes a policy $\pi=\{\pi_h\}_{h=1}^H$. Given the policy computed by the planning phase, the performance metric is the suboptimality metric
\begin{equation}
    \label{eq:subopt_RF_MDP}
    \textnormal{SubOpt}(\pi; r) = V_1^{\pi^*}(s_0; r) - V_1^{\pi}(s_0; r).
\end{equation}

{\par \textbf{Parallelism in the Reward-free Case for MG and its Performance Metric.}} We basically consider the same setting as in the MDP case, but with an underlying zero-sum Markov game instead. Given some policies $\pi$ for Player 1 and $\nu$ for Player 2 computed by the planning phase, we define our performance metric as 
\begin{equation}
    \label{eq:subopt_RF_MG}
    \textnormal{SubOpt}(\pi,\nu;r) = V_1^{\bre_1(\nu),\nu}(s_0; r) - V_1^{\pi, \bre_2(\pi)}(s_0; r). 
\end{equation}
Note that $\textnormal{SubOpt}(\pi,\nu;r)=0$ iff $(\pi,\nu)$ is a Nash equilibrium.  

In this paper, we consider linear MDPs and linear MGs.

%==================================================
%==================================================

\section{Warm Up: The Parallel Online RL Case}
\label{sec:parallel_online_RL}

We propose \emph{Parallel Optimistic Least-Squares Value Iteration} (POLSVI), as described in Algorithm~\ref{alg:main_LIN_UCB_LSVI}. Our algorithm is based on the LSVI-UCB algorithm~\citep{CJ-ZY-ZW-MIJ:20} and is parallelized as follows. At each episode, a central server aggregates the state-action trajectories collected by the parallel agents up to the previous episode in a covariance matrix $\Lambda^k_h$ (line 7 of Algorithm~\ref{alg:main_LIN_UCB_LSVI}), after which it computes an optimistic estimate of the Q-function (line 9) using the optimism bonus $\beta(\phi(\cdot,\cdot)^\top(\Lambda^k_h)^{-1}\phi(\cdot,\cdot))^{1/2}$. Each agent then simultaneously explores the environment by following a greedy policy using the common function computed by the central server. The state-action trajectories taken by all the $P$ agents will then be collected by the central server at the next episode and the whole process repeats. We have the following result for the POLSVI algorithm.

\begin{algorithm}[tb]
  \caption{Parallel Optimistic LSVI (POLSVI)}
  \label{alg:main_LIN_UCB_LSVI}
\begin{algorithmic}[1]
    \STATE {\bfseries Input:} $P$, $K$, $\beta$, $\lambda$
    \FOR{episode $k\in[K]$}
        % \STATE Receive initial state $s_0$
        \STATE $x_1^{k,p}\gets s_0$ for $p\in[P]$
        \STATE \# DONE BY CENTRAL SERVER:
         \STATE $Q_{H+1}^k(\cdot,\cdot)\gets 0$
        \FOR{$h=H,\dots,1$}
            \STATE $\Lambda_{h}^k\gets\lambda I_d + \sum^P_{p=1}\sum^{k-1}_{\tau=1}\phi(x_h^{\tau,p},a_h^{\tau,p})\phi(x_h^{\tau,p},a_h^{\tau,p})^\top$
            \STATE $w_{h}^k\gets(\Lambda_{h}^{k})^{-1} \sum^P_{p=1}\sum^{k-1}_{\tau=1}\phi(x_h^{\tau,p},a_h^{\tau,p})[r_h^{\tau,p}+\max_{a\in\A} Q^{k}_{h+1}(x_{h+1}^{\tau,p},a)]$
            \STATE $Q^k_{h}(\cdot,\cdot)\gets \min\{(w_{h}^k)^\top\phi(\cdot,\cdot)+\beta(\phi(\cdot,\cdot)^\top(\Lambda_{h}^{k})^{-1}\phi(\cdot,\cdot))^{1/2},H\}$
        \ENDFOR
        \STATE \# DONE BY EACH AGENT $p\in[P]$ IN PARALLEL:
            \FOR{$h\in[H]$}
                \STATE $a_h^{k,p} \in \arg\max_{a\in\A} Q^k_h(x_h^{k,p},a)$ %for $p\in[P]$ 
                \# GREEDY POLICY
                \STATE Observe $x_{h+1}^{k,p}$ 
            \ENDFOR
    \ENDFOR
\end{algorithmic}
\end{algorithm}

\begin{theorem}[Performance of the POLSVI algorithm]
\label{thm:main-paralel}
There exists an absolute constant $c_\beta>0$ such that, for any fixed $\delta\in(0,1)$, if we set $\lambda=1$ and $\beta=c_\beta dH\sqrt{\iota}$, with $\iota:=\log(dKHP/\delta)$, then, with probability at least $1-2\delta$, 
\begin{equation}
\label{eqn:regret-res2}
%\begin{aligned}
\textnormal{Regret}(K,P)
\leq 
\underbrace{\cO\bigg(\sqrt{KP}\sqrt{d^3H^4\iota^2}\bigg)}_{\text{Base term}} 
+ \underbrace{\cO\bigg(\sqrt{ d^4H^4\iota}P\log\left(1+\frac{KP}{d}\right)\bigg)}_{\text{Overhead term}}.
%\end{aligned}
\end{equation}
\end{theorem}
From Theorem~\ref{thm:main-paralel} we observe that our performance metric has two complexity terms: the \emph{base term} and the \emph{overhead term}, which we discuss next.

{\par \textbf{Speedup in the Base Term.}} The base term is the performance that would be obtained by the sequential counterpart of POLSVI in $KP$ episodes. The regret metric that we would employ for the sequential counterpart becomes $\textnormal{Regret}(K) = \sum_{k=1}^{K} V_1^{\pi^*}(x^k_1) - V_1^{\pi^{k}}(x^k_1)$, where $K$ is the number of episodes and $\pi^k$ is the (greedy) policy taken by the single agent at episode $k\in[K]$. \citep{CJ-ZY-ZW-MIJ:20} proved that with probability $1-\delta$: $\textnormal{Regret}(K)\leq O(\sqrt{K}\sqrt{d^3H^4\iota^2})$, where $\iota=\log(2dKH/\delta)$ (under $\lambda=1$ and $\beta=c_\beta d H\sqrt{\iota}$, $c_\beta$ being some absolute constant). Therefore, the base term in our learning regret~\eqref{eqn:regret-res2} indicates an almost linear speedup, because of the factor $\sqrt{KP}$ compared to the factor $\sqrt{K}$ in the performance of the sequential algorithm. In other words, in terms of the base term, there is a complexity \emph{equivalence} between performing the sequential algorithm for $KP$ episodes and performing the parallelized version with $P$ agents for $K$ episodes. 

{\par \textbf{Overhead Term: the Price of Parallelization.}} Given that the base term indicates a speedup in POLSVI with respect to its sequential counterpart, the overhead term adds an extra complexity term due to the use of parallel agents in the RL algorithm --- this term would be nonexistent if the algorithm was sequentially executed by a single agent. Following the proof of Theorem~\ref{thm:main-paralel}, the overhead term originates from the occurrence of the event ``$\Lambda_h^{k+1}\succ 2\Lambda_h^{k}$" across different steps $h\in[H]$ and episodes $k\in[K]$ --- we call this event a \emph{doubling round}, a term taken from~\citet{JC-AP-NT-YSS-PB-MIJ:21}, where a similar phenomenon occurs in bandits. The overhead term is a bound on the total number of doubling rounds in all steps from all episodes in which the algorithm is executed. Given a fixed step $h$, a doubling round occurs when the information in the covariance matrix between two consecutive episodes is \emph{too different}, that is, when aggregating the information collected by the parallel agents adds a \emph{considerable} amount of information or novelty to what has been obtained so far in the previous episode. In Appendix~\ref{App:explain-dr}, we further explain how the inefficiency of parallel exploration and the stochasticity of the environment result in doubling rounds.

{\par \textbf{POLSVI employs a \emph{single policy} during exploration}} On line 13 of Algorithm~\ref{alg:main_LIN_UCB_LSVI}, we observe that all agents are using the same function $Q^k_h$ in order to compute their greedy policies, instead of each agent constructing its own Q-function to define its respective greedy policy. If agents have different Q-functions, then each one may take very distinct (greedy) actions given the same state, i.e., there would be \emph{heterogeneous} policies.

\section{Parallelizing Reward-Free Exploration: A Surprisingly Simple Baseline}
\label{sec:paral-rf-sec}

We have shown that a simplistic incorporation of parallelism in the online RL setting results in nearly optimal regret up to logarithmic factors, without having the agents execute a diverse set of policies. We use this intuition to develop algorithms and results for the reward-free setting.

\subsection{Markov Decision Process}

Algorithm~\ref{alg:main_LIN_RF_POLSVI_exp} and Algorithm~\ref{alg:main_LIN_RF_POLSVI_plan} detail the exploration and planning phases of the \emph{Reward-Free Parallel Optimistic Least-Squares Value Iteration} (RF-POLSVI) algorithm. The \emph{exploration phase} is very similar to POLSVI, with the difference that the optimistic bonus term $\beta(\phi(\cdot,\cdot)^\top(\Lambda^k_h)^{-1}\phi(\cdot,\cdot))^{1/2}$ is used both in the construction of the reward function (see the term $r^k_h(\cdot,\cdot)$ in line 9) and in its usual role as an optimism bonus (see line 11). In the \emph{planning phase}, the central server computes the final greedy policy based on the information collected by the parallel agents in the exploration phase, using a user-specified reward. The performance of the RF-POLSVI algorithm is summarized in the following result.

\begin{theorem}[Performance of the RF-POLSVI algorithm]
\label{thm:main-reward-free}
There exists an absolute constant $c_\beta>0$ such that, for any fixed $\delta\in(0,1)$, if we set $\lambda=1$ and $\beta=c_\beta dH\sqrt{\iota}$, with $\iota:=\log(dKHP/\delta)$, then, with probability at least $1-3\delta$, 
\begin{equation}
\label{eqn:subopt-rf-1}
%\begin{aligned}
\textnormal{SubOpt}(\pi;r) \leq 
\underbrace{\cO\left(\sqrt{\frac{d^3H^6\iota^2}{KP}}\right)}_{\text{Base term}} + \underbrace{\cO\left(
\frac{\sqrt{d^4H^6\iota}}{K}\log\left(1+\frac{KP}{d}\right)\right)}_{\text{Overhead term}}.
%\end{aligned}
\end{equation}
\end{theorem}

We have the same observations on the overhead and base terms in the suboptimality performance metric~\eqref{eqn:subopt-rf-1}. Likewise, we observe that RF-POLSVI employs a single policy for exploration: line 15 of Algorithm~\ref{alg:main_LIN_RF_POLSVI_exp} shows each agent following the same greedy policy based on the function $Q^k_h$.  \citep{RW-SD-LY-RS:20} studied the sequential counterpart of RF-POLSVI and we observe that the base term indicates an almost linear speed-up with respect to its sequential counterpart. The overhead term is again an additional penalty term due to doubling rounds.
\begin{algorithm}[tb]
  \caption{Reward-Free POLSVI (RF-POLSVI) --- Exploration phase}
  \label{alg:main_LIN_RF_POLSVI_exp}
\begin{algorithmic}[1]
    \STATE {\bfseries Input:} $P$, $T$, $\beta$, $\lambda$
    \FOR{episode $k\in[K]$}
        % \STATE Receive initial state $s_0\sim \nu$
        \STATE $x_1^{k,p}\gets s_0$ for $p\in[P]$
        \STATE \# DONE BY CENTRAL SERVER:
        \STATE $Q_{H+1}^k(\cdot,\cdot)\gets 0$
        \FOR{$h=H,\dots,1$}
            \STATE $\Lambda_{h}^k\gets\lambda I_d + \sum^P_{p=1}\sum^{k-1}_{\tau=1}\phi(x_h^{\tau,p},a_h^{\tau,p})\phi(x_h^{\tau,p},a_h^{\tau,p})^\top$
            \STATE 
        $u^k_h(\cdot,\cdot)\gets\min\{\beta(\phi(\cdot,\cdot)^\top(\Lambda_{h}^{k})^{-1}\phi(\cdot,\cdot))^{1/2},H\}$
        \STATE $r^k_h(\cdot,\cdot)\gets u^k_h(\cdot,\cdot)/H$
        \STATE $w_{h}^k\gets(\Lambda_{h}^{k})^{-1} \sum^P_{p=1}\sum^{k-1}_{\tau=1}\phi(x_h^{\tau,p},a_h^{\tau,p})$        $\times\max_{a\in\A} Q^{k}_{h+1}(x_{h+1}^{\tau,p},a)$
            \STATE $Q^k_{h}(\cdot,\cdot)\gets \min\{(w_{h}^k)^\top\phi(\cdot,\cdot)+r^k_h(\cdot,\cdot)+u^k_h(\cdot,\cdot),H\}$
        \ENDFOR
        \STATE \# DONE BY EACH AGENT $p\in[P]$ IN PARALLEL:
        %\FOR{$p\in[P]$} %Not sure if put the p as in outer loop
            \FOR{$h\in[H]$}
                \STATE $a_h^{k,p} \in \arg\max_{a\in\A} Q^k_h(x_h^{k,p},a)$ \# GREEDY POLICY
                \STATE Observe $x_{h+1}^{k,p}$
            \ENDFOR   
    \ENDFOR
\STATE {\bfseries Return} $\{(x^{k,p}_h,a^{k,p}_h)\}_{(h,k,p)\in[H]\times[K]\times[P]}$ \# COLLECTED TRAJECTORIES % BY THE AGENTS
\end{algorithmic}
\end{algorithm}
\begin{algorithm}[tb]
  \caption{Reward-Free POLSVI (RF-POLSVI) --- Planning phase}
  \label{alg:main_LIN_RF_POLSVI_plan}
\begin{algorithmic}[1]
    \STATE {\bfseries Input:} $P$, $\beta$, $\lambda$, $\{(x^{k,p}_h,a^{k,p}_h    
    )\}_{(h,k,p)\in[H]\times[K]\times[P]}$, $r=\{r_h\}_{h\in[h]}$ 
    \STATE $\hat{Q}_{H+1}(\cdot,\cdot)\gets 0$
        \FOR{$h=H,\dots,1$}
            \STATE $\Lambda_{h}\gets\lambda I_d + \sum^P_{p=1}\sum^K_{\tau=1}\phi(x_h^{\tau,p},a_h^{\tau,p})\phi(x_h^{\tau,p},a_h^{\tau,p})^\top$
            \STATE 
        $u_h(\cdot,\cdot)\gets\min\{\beta(\phi(\cdot,\cdot)^\top(\Lambda_{h})^{-1}\phi(\cdot,\cdot))^{1/2},H\}$
            \STATE            $\hat{w}_{h}\gets(\Lambda_{h})^{-1} \sum^P_{p=1}\sum^{K}_{\tau=1}\phi(x_h^{\tau,p},a_h^{\tau,p})$ $\times\max_{a\in\A} \hat{Q}_{h+1}(x_{h+1}^{\tau,p},a)$
            \STATE $\hat{Q}_{h}(\cdot,\cdot)\gets \min\{(\hat{w}_{h})^\top\phi(\cdot,\cdot)+r_h(\cdot,\cdot)+u_h(\cdot,\cdot)
            ,H\}$
            \STATE $\pi_h(\cdot)\in  \arg\max_{a\in\A} \hat{Q}_h(\cdot,a)$
        \ENDFOR
\STATE {\bfseries Return} $\pi=\{\pi_h\}_{h\in[H]}$
\end{algorithmic}
\end{algorithm}

\subsection{Zero-Sum Markov Games}
We study reward-free RL with an underlying zero-sum Markov game to demonstrate the power of parallel exploration in the MARL context. We propose the \emph{Reward-Free Markov Game Parallel Optimistic Least-Squares Value Iteration} (RFMG-POLSVI) algorithm. The \emph{exploration phase}, described in Algorithm~\ref{alg:main_LIN_RFMG_POLSVI_exp} in Appendix~\ref{App:RFMG-POLSVI}, is basically the same as RF-POLSVI with the difference that the action space is extended to a product of action spaces corresponding to each of the two players in the MG --- and so there is exploration using a single policy. In the \emph{planning phase}, in Algorithm~\ref{alg:main_LIN_RFMG_POLSVI_plan}, the central server computes the policies for each player through the computation of two Nash Equilibria for two-player zero-sum (static) games at each step of the MG --- lines 11 and 12 of Algorithm~\ref{alg:main_LIN_RFMG_POLSVI_plan} are minimax problems. The following result summarizes the performance of the RFMG-POLSVI algorithm.

\begin{theorem}[Performance of the RFMG-POLSVI algorithm]
\label{thm:main-reward-free_MG}
There exists an absolute constant $c_\beta>0$ such that, for any fixed $\delta\in(0,1)$, if we set $\lambda=1$ and $\beta=c_\beta dH\sqrt{\iota}$, with $\iota:=\log(dKHP/\delta)$, then, with probability at least $1-3\delta$,
\begin{equation}
\label{eqn:subopt-rf-1-MG}
%\begin{aligned}
\textnormal{SubOpt}(\pi,\nu;r) \leq 
\underbrace{\cO\left(\sqrt{\frac{d^3H^6\iota^2}{KP}}\right)}_{\text{Base term}} + \underbrace{\cO\left(\frac{\sqrt{d^4H^6\iota}}{K}\log\left(1+\frac{KP}{d}\right)\right)}_{\text{Overhead term}}.
%\end{aligned}
\end{equation}
\end{theorem}

We are not aware of a sequential counterpart to RFMG-POLSVI in the literature. However, recently, \citet{SQ-JY-ZW-ZY:22} provided (sequential) algorithms for reward-free RL for MGs with theoretical guarantees under kernel function and neural network approximation --- indeed, RFMG-POLSVI is based on a parallel adaptation of the general structure of the algorithms by~\citet{SQ-JY-ZW-ZY:22}. Since the suboptimality in~\eqref{eqn:subopt-rf-1-MG} has the same form as in~\eqref{eqn:subopt-rf-1}, we can conclude that we also have an almost linear speed-up compared to the sequential counterpart of our algorithm. The overhead term is also polylogarithmic in $P$. 
\begin{algorithm}[tb]
  \caption{Reward-Free Markov Game POLSVI (RFMG-POLSVI) --- Planning phase}
  \label{alg:main_LIN_RFMG_POLSVI_plan}
\begin{algorithmic}[1]
    \STATE {\bfseries Input:} $P$, $\beta$, $\lambda$, % Are these inputs? 
    $\{(x^{k,p}_h,a^{k,p}_h,b^{k,p}_h)\}_{(h,k,p)\in[H]\times[K]\times[P]}$, % (from the exploration phase), 
    $r=\{r_h\}_{h\in[h]}$ 
    \STATE $\overline{Q}_{H+1}(\cdot,\cdot,\cdot)\gets 0$
    \STATE $\underline{Q}_{H+1}(\cdot,\cdot,\cdot)\gets 0$
    \FOR{$h=H,\dots,1$}
        \STATE $\Lambda_{h}\gets\lambda I_d + \sum^P_{p=1}\sum^K_{\tau=1}\phi(x_h^{\tau,p},a_h^{\tau,p},b_h^{\tau,p})\phi(x_h^{\tau,p},a_h^{\tau,p},b_h^{\tau,p})^\top$
        \STATE 
        $u_h(\cdot,\cdot,\cdot)\gets\min\{\beta(\phi(\cdot,\cdot,\cdot)^\top(\Lambda_{h})^{-1}\phi(\cdot,\cdot,\cdot))^{1/2},$ $H\}$
        \STATE            $\overline{w}_{h}\gets(\Lambda_{h})^{-1} \sum^P_{p=1}\sum^{K}_{\tau=1}\phi(x_h^{\tau,p},a_h^{\tau,p},b_h^{\tau,p})$ $\times\E_{a\sim\pi_{h+1}(x^{\tau,p}_{h+1}),b\sim\overline{D}(x^{\tau,p}_{h+1})}[ \overline{Q}_{h+1}(x_{h+1}^{\tau,p},a,b)]$
        \STATE            $\underline{w}_{h}\gets(\Lambda_{h})^{-1} \sum^P_{p=1}\sum^{K}_{\tau=1}\phi(x_h^{\tau,p},a_h^{\tau,p},b_h^{\tau,p})$ $\times\E_{a\sim\underline{D}(x^{\tau,p}_{h+1}),b\sim\nu_{h+1}(x^{\tau,p}_{h+1})}[ \underline{Q}_{h+1}(x_{h+1}^{\tau,p},a,b)]$
       \STATE $\overline{Q}_{h}(\cdot,\cdot,\cdot)\gets \Pi_{[0,H]}[(\overline{w}_{h})^\top\phi(\cdot,\cdot,\cdot)+r_h(\cdot,\cdot,\cdot)+u_h(\cdot,\cdot,\cdot)]$
       \STATE $\underline{Q}_{h}(\cdot,\cdot,\cdot)\gets \Pi_{[0,H]}[(\underline{w}_{h})^\top\phi(\cdot,\cdot,\cdot)+r_h(\cdot,\cdot,\cdot)-u_h(\cdot,\cdot,\cdot)]$
       \STATE
       $(\pi_h(x),\overline{D}(x))\in\text{Nash Equilibrium}\,(\overline{Q}_h(x,\cdot,\cdot))$ for any $x\in\S$
       \STATE
       $(\underline{D}(x),\nu_h(x))\in\text{Nash Equilibrium}\,(\underline{Q}_h(x,\cdot,\cdot))$ for any $x\in\S$
        \ENDFOR
\STATE {\bfseries Return} $\pi=\{\pi_h\}_{h\in[H]}$, $\nu=\{\nu_h\}_{h\in[H]}$
\end{algorithmic}
\end{algorithm}

\subsection{Lower Bound in Reward-Free Exploration}
\label{subsec:lower-bound}

We now present and discuss the lower bound for parallel reward-free exploration in linear MDPs. Our proof technique mimics that of~\citep{AW-YC-MS-SSD-KJ:22} as the lower bound constructed there is the tightest one for reward-free exploration in linear MDPs.

\begin{theorem}
\label{thm:lower_bound}
Let $\epsilon>0$, $P > 0$, $d >1$, and $KP \geq d^2$. Consider running a parallel algorithm with $P$ agents for $K$ episodes in a $(d + 1)$-dimensional linear MDP, where each agent is allowed to have a unique exploration policy. Suppose that the parallel algorithm stops at a possibly random stopping time $\tau$ and outputs a policy $\hat{\pi}$ which is a guess at an $\epsilon$-optimal policy. Then, there is a universal constant $c > 0$ such that unless $KP \geq c (dH/\epsilon)^2$, there exists a linear MDP $\cM$ for which $\Pr_{\cM}[\{\tau > K \textrm{ or }\hat{\pi} \textrm{ is not $\epsilon$-optimal}\}] \geq 0.1$; i.e., with constant probability either $\hat{\pi}$ is not $\epsilon$-optimal or more than $KP$ samples are collected.  
\end{theorem}

We remark that Theorem~\ref{thm:lower_bound} derives a lower bound for PAC RL~\citep{strehl2009reinforcement}, a setting that focuses on learning a policy with suboptimality at most $\epsilon$ for a \emph{given} reward function. As reward-free RL is capable of returning policies with suboptimality at most $\epsilon$ for \emph{arbitrary} reward functions, any reward-free RL algorithm may be used for PAC RL and, naturally, any lower bound for PAC RL holds for reward-free RL.

{\par \textbf{Near-Minimax Optimality.}} We highlight the fact that Theorem~\ref{thm:lower_bound} and Theorem~\ref{thm:main-reward-free}, in combination, show that RF-POLSVI is near-minimax optimal up to logarithmic factors when $P = \cO(K)$. More specifically, when $P = \cO(K)$, the suboptimality of Algorithm~\ref{alg:main_LIN_RF_POLSVI_exp} decreases at a rate of $\Tilde{\cO}(d^2H^3/\sqrt{KP})$, which translates to $K = {\Omega}(1/(Pd^4H^6\epsilon^2))$ (up to logarithmic factors) rounds of interaction with the environment, matching the lower bound derived in Theorem~\ref{thm:lower_bound} necessary to obtain a $(dH^2\epsilon)$-optimal policy --- i.e., our algorithm matches the minimax lower bound up to a factor $dH^2$. We highlight that it may be possible that minimax optimality may be attained by heterogeneous policies; however, as our results demonstrate, these heterogeneous policies can at most match the rate of our single policy algorithm in terms of $K$ and $P$, arguably two of the most important parameters in parallel exploration. Moreover, the extra $dH^2$ factor in our rate is benign considering that we only have a single exploration policy, and yet we avoid extra high-degree or even exponential factors on the problem parameters. 

As discussed in Section~\ref{sec:introduction}, we remark that the assumption $P=\cO(K)$ for near-optimality holds even in large-scale real-world parallel RL applications, such as AlphaZero~\citep{alphaZero:17}, which used 700,000 batches of trajectories generated using only 5,000 parallel processors.

We end the section by noting that the proposed parallel reward-free exploration algorithms are surprisingly hard to beat. In particular, we show that simple adaptations of state-of-the-art results in reward-free exploration cannot outperform our proposed methods.

{\par \textbf{Comparison to~\citet{huang2021deployment}.}} \citet{huang2021deployment} proposed a deployment efficient reward-free exploration algorithm that can be trivially adapted to the parallel setting. However, without additional assumptions on the feature representations --- assuming that the so-called reachability coefficient is sufficiently large (see Definition 4.3 in~\citet{huang2021deployment}) --- only Algorithm 1 from~\citep{huang2021deployment} can be applied, which requires $K = {\cO}(1/(P\epsilon^{c_K}))$ rounds of interactions. In other words, using the results in~\citep{huang2021deployment} directly cannot achieve the optimality obtained in our setting, even when $P = \cO(K)$. 

{\par \textbf{Comparison to~\citet{AW-YC-MS-SSD-KJ:22}.}} \citet{AW-YC-MS-SSD-KJ:22} proposed a reward-free RL algorithm for linear MDPs that achieves a $\cO(d^2H^5/\epsilon^2)$ sample complexity, which improves the dependence on $d$ (the dimensionality of the linear feature representation $\phi$) with respect to the sequential counterpart of RF-POLSVI: from $d^3$ to $d^2$. However, we notice that RF-POLSVI has a sample complexity dependence on $d^4$ coming from the overhead term, and not from the base term. Therefore, we argue that implementing a \emph{parallelized} version of the algorithm by \citet{AW-YC-MS-SSD-KJ:22} using our single policy exploration method would still incur the same dependence on $d^4$ and no reduction in $d$ would be gained. Moreover, we remark that even in the sequential case, the improvement by \citet{AW-YC-MS-SSD-KJ:22} is in $d$ only and requires a significantly more complex algorithm and representation. Instead, RF-POLSVI achieves an almost linear speedup in $P$ under mild assumptions with an algorithm that has a less complex structure.

\section{Conclusion}
We formally proved that, for various RL problems, a simple way of aggregating information collected in parallel by agents in the exploration phase results in both an almost linear speedup term benefiting from the amount of parallelization $P$, and an additional complexity term polylogarithmic on such $P$. For the reward-free setting, our method is nearly minimax optimal. Moreover, we showed the benefits of parallel exploration in a MARL context. 

Our work gives rise to a host of open questions that remain to be answered. The success of more intricate coordinated exploration has been demonstrated in empirical studies, while our simplistic approach is provably near-optimal. What are the theoretical justifications in support of these coordinated exploration strategies? Are there settings under which coordinated approaches provably outperform our simplistic approach, i.e., by better matching the minimax lower bound? What could be the communication, computation, and sample complexity trade-offs between coordinated exploration and exploration using only a single policy?

\section*{Acknowledgments}
We are grateful to the anonymous reviewers and the meta-reviewer for their time and their comments to improve our paper. This work is partially supported by NSF III 2046795, IIS 1909577, CCF 1934986, NIH 1R01MH116226-01A, NIFA award 2020-67021-32799, the Alfred P. Sloan Foundation, and Google Inc. This work is also partially supported by
the William S. Fishman Faculty Research Fund at the University of Chicago Booth School of Business.

\bibliography{Parallel_MDP}

\begin{thebibliography}{45}
\providecommand{\natexlab}[1]{#1}
\providecommand{\url}[1]{\texttt{#1}}
\expandafter\ifx\csname urlstyle\endcsname\relax
  \providecommand{\doi}[1]{doi: #1}\else
  \providecommand{\doi}{doi: \begingroup \urlstyle{rm}\Url}\fi

\bibitem[Abbasi{-}Yadkori et~al.(2011)Abbasi{-}Yadkori, P{\'{a}}l, and
  Szepesv{\'{a}}ri]{YAY-DP-CS:11}
Yasin Abbasi{-}Yadkori, D{\'{a}}vid P{\'{a}}l, and Csaba Szepesv{\'{a}}ri.
\newblock Improved algorithms for linear stochastic bandits.
\newblock In John Shawe{-}Taylor, Richard~S. Zemel, Peter~L. Bartlett, Fernando
  C.~N. Pereira, and Kilian~Q. Weinberger, editors, \emph{Advances in Neural
  Information Processing Systems}, pages 2312--2320, 2011.
\newblock URL
  \url{https://proceedings.neurips.cc/paper/2011/hash/e1d5be1c7f2f456670de3d53c7b54f4a-Abstract.html}.

\bibitem[Agarwal et~al.(2020)Agarwal, Kakade, Krishnamurthy, and
  Sun]{agarwal2020flambe}
Alekh Agarwal, Sham~M. Kakade, Akshay Krishnamurthy, and Wen Sun.
\newblock {FLAMBE:} {S}tructural complexity and representation learning of low
  rank {MDP}s.
\newblock In Hugo Larochelle, Marc'Aurelio Ranzato, Raia Hadsell,
  Maria{-}Florina Balcan, and Hsuan{-}Tien Lin, editors, \emph{Advances in
  Neural Information Processing Systems}, 2020.
\newblock URL
  \url{https://proceedings.neurips.cc/paper/2020/hash/e894d787e2fd6c133af47140aa156f00-Abstract.html}.

\bibitem[Agarwal et~al.(2021)Agarwal, Ganguly, and
  Aggarwal]{agarwal2021communication}
Mridul Agarwal, Bhargav Ganguly, and Vaneet Aggarwal.
\newblock Communication efficient parallel reinforcement learning.
\newblock In Cassio~P. de~Campos, Marloes~H. Maathuis, and Erik Quaeghebeur,
  editors, \emph{Uncertainty in Artificial Intelligence}, volume 161, pages
  247--256, 2021.
\newblock URL \url{https://proceedings.mlr.press/v161/agarwal21a.html}.

\bibitem[Auer et~al.(2002)Auer, Cesa-Bianchi, and Fischer]{auer2002finite}
Peter Auer, Nicolo Cesa-Bianchi, and Paul Fischer.
\newblock Finite-time analysis of the multiarmed bandit problem.
\newblock \emph{Machine learning}, 47\penalty0 (2):\penalty0 235--256, 2002.

\bibitem[Bai et~al.(2019)Bai, Xie, Jiang, and Wang]{bai2019provably}
Yu~Bai, Tengyang Xie, Nan Jiang, and Yu{-}Xiang Wang.
\newblock Provably efficient {Q}-learning with low switching cost.
\newblock In Hanna~M. Wallach, Hugo Larochelle, Alina Beygelzimer, Florence
  d'Alch{\'{e}}{-}Buc, Emily~B. Fox, and Roman Garnett, editors, \emph{Advances
  in Neural Information Processing Systems}, volume~32, pages 8002--8011, 2019.
\newblock URL
  \url{https://proceedings.neurips.cc/paper/2019/hash/473803f0f2ebd77d83ee60daaa61f381-Abstract.html}.

\bibitem[Boutilier(1996)]{boutilier1996planning}
Craig Boutilier.
\newblock Planning, learning and coordination in multiagent decision processes.
\newblock In Yoav Shoham, editor, \emph{Theoretical Aspects of Rationality and
  Knowledge}, pages 195--210. Morgan Kaufmann, 1996.

\bibitem[Cesa{-}Bianchi et~al.(2013)Cesa{-}Bianchi, Dekel, and
  Shamir]{cesa2013online}
Nicol{\`{o}} Cesa{-}Bianchi, Ofer Dekel, and Ohad Shamir.
\newblock Online learning with switching costs and other adaptive adversaries.
\newblock In Christopher J.~C. Burges, L{\'{e}}on Bottou, Zoubin Ghahramani,
  and Kilian~Q. Weinberger, editors, \emph{Advances in Neural Information
  Processing Systems}, pages 1160--1168, 2013.
\newblock URL
  \url{https://proceedings.neurips.cc/paper/2013/hash/9cf81d8026a9018052c429cc4e56739b-Abstract.html}.

\bibitem[Chan et~al.(2021)Chan, Pacchiano, Tripuraneni, Song, Bartlett, and
  Jordan]{JC-AP-NT-YSS-PB-MIJ:21}
Jeffrey Chan, Aldo Pacchiano, Nilesh Tripuraneni, Yun~S Song, Peter Bartlett,
  and Michael~I Jordan.
\newblock Parallelizing contextual linear bandits.
\newblock \emph{arXiv preprint arXiv:2105.10590}, 2021.

\bibitem[Chen et~al.(2022)Chen, Zhou, and Gu]{chen2022almost}
Zixiang Chen, Dongruo Zhou, and Quanquan Gu.
\newblock Almost optimal algorithms for two-player zero-sum linear mixture
  {Markov} games.
\newblock In \emph{International Conference on Algorithmic Learning Theory},
  pages 227--261, 2022.

\bibitem[Dimakopoulou and Van~Roy(2018)]{dimakopoulou2018coordinated}
Maria Dimakopoulou and Benjamin Van~Roy.
\newblock Coordinated exploration in concurrent reinforcement learning.
\newblock In \emph{International Conference on Machine Learning}, pages
  1271--1279, 2018.

\bibitem[Dubey and Pentland(2021)]{AD-AP:21}
Abhimanyu Dubey and Alex Pentland.
\newblock Provably efficient cooperative multi-agent reinforcement learning
  with function approximation.
\newblock \emph{arXiv preprint arXiv:2103.04972}, 2021.

\bibitem[Dubey and Pentland(2020)]{dubey2020differentially}
Abhimanyu Dubey and Alex~'Sandy' Pentland.
\newblock Differentially-private federated linear bandits.
\newblock In Hugo Larochelle, Marc'Aurelio Ranzato, Raia Hadsell,
  Maria{-}Florina Balcan, and Hsuan{-}Tien Lin, editors, \emph{Advances in
  Neural Information Processing Systems}, volume~33, 2020.
\newblock URL
  \url{https://proceedings.neurips.cc/paper/2020/hash/4311359ed4969e8401880e3c1836fbe1-Abstract.html}.

\bibitem[Gao et~al.(2021)Gao, Xie, Du, and Yang]{MG-RX-SSD-LFY:21}
Minbo Gao, Tianle Xie, Simon~S Du, and Lin~F Yang.
\newblock A provably efficient algorithm for linear markov decision process
  with low switching cost.
\newblock \emph{arXiv preprint arXiv:2101.00494}, 2021.

\bibitem[Huang et~al.(2022)Huang, Chen, Zhao, Qin, Jiang, and
  Liu]{huang2021deployment}
Jiawei Huang, Jinglin Chen, Li~Zhao, Tao Qin, Nan Jiang, and Tie{-}Yan Liu.
\newblock Towards deployment-efficient reinforcement learning: Lower bound and
  optimality.
\newblock In \emph{International Conference on Learning Representations}, 2022.
\newblock URL \url{https://openreview.net/forum?id=ccWaPGl9Hq}.

\bibitem[Huang et~al.(2021)Huang, Wu, Yang, and Shen]{huang2021federated}
Ruiquan Huang, Weiqiang Wu, Jing Yang, and Cong Shen.
\newblock Federated linear contextual bandits.
\newblock In Marc'Aurelio Ranzato, Alina Beygelzimer, Yann~N. Dauphin, Percy
  Liang, and Jennifer~Wortman Vaughan, editors, \emph{Advances in Neural
  Information Processing Systems}, 2021.
\newblock URL
  \url{https://proceedings.neurips.cc/paper/2021/hash/e347c51419ffb23ca3fd5050202f9c3d-Abstract.html}.

\bibitem[Jin et~al.(2020{\natexlab{a}})Jin, Krishnamurthy, Simchowitz, and
  Yu]{jin2020reward}
Chi Jin, Akshay Krishnamurthy, Max Simchowitz, and Tiancheng Yu.
\newblock Reward-free exploration for reinforcement learning.
\newblock In \emph{International Conference on Machine Learning}, pages
  4870--4879, 2020{\natexlab{a}}.

\bibitem[Jin et~al.(2020{\natexlab{b}})Jin, Yang, Wang, and
  Jordan]{CJ-ZY-ZW-MIJ:20}
Chi Jin, Zhuoran Yang, Zhaoran Wang, and Michael~I Jordan.
\newblock Provably efficient reinforcement learning with linear function
  approximation.
\newblock In \emph{Conference on Learning Theory}, pages 2137--2143,
  2020{\natexlab{b}}.

\bibitem[Karbasi et~al.(2021)Karbasi, Mirrokni, and
  Shadravan]{karbasi2021Parallelizing}
Amin Karbasi, Vahab~S. Mirrokni, and Mohammad Shadravan.
\newblock Parallelizing {Thompson} sampling.
\newblock In Marc'Aurelio Ranzato, Alina Beygelzimer, Yann~N. Dauphin, Percy
  Liang, and Jennifer~Wortman Vaughan, editors, \emph{Advances in Neural
  Information Processing Systems}, volume~34, pages 10535--10548, 2021.
\newblock URL
  \url{https://proceedings.neurips.cc/paper/2021/hash/56f0b515214a7ec9f08a4bbf9a56f7ba-Abstract.html}.

\bibitem[Konda and Tsitsiklis(1999)]{konda1999actor}
Vijay~R. Konda and John~N. Tsitsiklis.
\newblock Actor-critic algorithms.
\newblock In Sara~A. Solla, Todd~K. Leen, and Klaus{-}Robert M{\"{u}}ller,
  editors, \emph{Advances in Neural Information Processing Systems}, volume~12,
  pages 1008--1014. The {MIT} Press, 1999.
\newblock URL \url{http://papers.nips.cc/paper/1786-actor-critic-algorithms}.

\bibitem[Kozuno et~al.(2021)Kozuno, M{\'{e}}nard, Munos, and
  Valko]{kozuno2021learning}
Tadashi Kozuno, Pierre M{\'{e}}nard, R{\'{e}}mi Munos, and Michal Valko.
\newblock Learning in two-player zero-sum partially observable {Markov} games
  with perfect recall.
\newblock In Marc'Aurelio Ranzato, Alina Beygelzimer, Yann~N. Dauphin, Percy
  Liang, and Jennifer~Wortman Vaughan, editors, \emph{Advances in Neural
  Information Processing Systems}, volume~34, pages 11987--11998, 2021.
\newblock URL
  \url{https://proceedings.neurips.cc/paper/2021/hash/646c9941d7fb1bc793a7929328ae3f2f-Abstract.html}.

\bibitem[Lattimore and Szepesv{\'{a}}ri(2020)]{TL-CS:20}
Tor Lattimore and Csaba Szepesv{\'{a}}ri.
\newblock \emph{Bandit Algorithms}.
\newblock Cambridge University Press, jul 2020.
\newblock \doi{10.1017/9781108571401}.

\bibitem[Lin et~al.(2019)Lin, Zhang, Yang, Wang, Basar, Sandhu, and
  Liu]{lin2019communication}
Yixuan Lin, Kaiqing Zhang, Zhuoran Yang, Zhaoran Wang, Tamer Basar, Romeil
  Sandhu, and Ji~Liu.
\newblock A communication-efficient multi-agent actor-critic algorithm for
  distributed reinforcement learning.
\newblock In \emph{Conference on Decision and Control}, pages 5562--5567.
  {IEEE}, 2019.
\newblock \doi{10.1109/CDC40024.2019.9029257}.

\bibitem[Mahajan et~al.(2019)Mahajan, Rashid, Samvelyan, and
  Whiteson]{mahajan2019maven}
Anuj Mahajan, Tabish Rashid, Mikayel Samvelyan, and Shimon Whiteson.
\newblock {MAVEN:} {M}ulti-agent variational exploration.
\newblock In Hanna~M. Wallach, Hugo Larochelle, Alina Beygelzimer, Florence
  d'Alch{\'{e}}{-}Buc, Emily~B. Fox, and Roman Garnett, editors, \emph{Advances
  in Neural Information Processing Systems}, pages 7611--7622, 2019.
\newblock URL
  \url{https://proceedings.neurips.cc/paper/2019/hash/f816dc0acface7498e10496222e9db10-Abstract.html}.

\bibitem[Nair et~al.(2015)Nair, Srinivasan, Blackwell, Alcicek, Fearon,
  De~Maria, Panneershelvam, Suleyman, Beattie, and Petersen]{nair2015Massively}
Arun Nair, Praveen Srinivasan, Sam Blackwell, Cagdas Alcicek, Rory Fearon,
  Alessandro De~Maria, Vedavyas Panneershelvam, Mustafa Suleyman, Charles
  Beattie, and Stig Petersen.
\newblock Massively parallel methods for deep reinforcement learning.
\newblock \emph{arXiv preprint arXiv:1507.04296}, 2015.

\bibitem[Perolat et~al.(2015)Perolat, Scherrer, Piot, and
  Pietquin]{perolat2015approximate}
Julien Perolat, Bruno Scherrer, Bilal Piot, and Olivier Pietquin.
\newblock Approximate dynamic programming for two-player zero-sum {Markov}
  games.
\newblock In \emph{International Conference on Machine Learning}, pages
  1321--1329, 2015.

\bibitem[Pham et~al.(2018)Pham, La, Feil{-}Seifer, and
  Nguyen]{pham2018cooperative}
Huy~Xuan Pham, Hung~Manh La, David Feil{-}Seifer, and Luan~Van Nguyen.
\newblock Cooperative and distributed reinforcement learning of drones for
  field coverage.
\newblock \emph{arXiv preprint arXiv:1803.07250}, 2018.
\newblock URL \url{http://arxiv.org/abs/1803.07250}.

\bibitem[Qiu et~al.(2021)Qiu, Ye, Wang, and Yang]{SQ-JY-ZW-ZY:22}
Shuang Qiu, Jieping Ye, Zhaoran Wang, and Zhuoran Yang.
\newblock On reward-free {RL} with kernel and neural function approximations:
  Single-agent {MDP} and {Markov} game.
\newblock In Marina Meila and Tong Zhang, editors, \emph{International
  Conference on Machine Learning}, volume 139 of \emph{Proceedings of Machine
  Learning Research}, pages 8737--8747, 2021.

\bibitem[Shamir(2013)]{shamir2013}
Ohad Shamir.
\newblock On the complexity of bandit and derivative-free stochastic convex
  optimization.
\newblock In Shai Shalev{-}Shwartz and Ingo Steinwart, editors,
  \emph{Conference on Learning Theory}, volume~30, pages 3--24, 2013.
\newblock URL \url{http://proceedings.mlr.press/v30/Shamir13.html}.

\bibitem[Shapley(1953)]{shapley1953SG}
L.~S. Shapley.
\newblock Stochastic games.
\newblock \emph{Proceedings of the National Academy of Sciences}, 39\penalty0
  (10):\penalty0 1095--1100, 1953.

\bibitem[Shi et~al.(2021)Shi, Shen, and Yang]{shi2021federated}
Chengshuai Shi, Cong Shen, and Jing Yang.
\newblock Federated multi-armed bandits with personalization.
\newblock In \emph{International Conference on Artificial Intelligence and
  Statistics}, pages 2917--2925, 2021.

\bibitem[Silver et~al.(2016)Silver, Huang, Maddison, Guez, Sifre, van~den
  Driessche, Schrittwieser, Antonoglou, Panneershelvam, Lanctot, Dieleman,
  Grewe, Nham, Kalchbrenner, Sutskever, Lillicrap, Leach, Kavukcuoglu, Graepel,
  and Hassabis]{alphaGo:16}
David Silver, Aja Huang, Chris~J. Maddison, Arthur Guez, Laurent Sifre, George
  van~den Driessche, Julian Schrittwieser, Ioannis Antonoglou, Veda
  Panneershelvam, Marc Lanctot, Sander Dieleman, Dominik Grewe, John Nham, Nal
  Kalchbrenner, Ilya Sutskever, Timothy Lillicrap, Madeleine Leach, Koray
  Kavukcuoglu, Thore Graepel, and Demis Hassabis.
\newblock Mastering the game of {Go} with deep neural networks and tree search.
\newblock \emph{Nature}, 2016.

\bibitem[Silver et~al.(2017)Silver, Schrittwieser, Simonyan, Antonoglou, Huang,
  Guez, Hubert, Baker, Lai, Bolton, Chen, Lillicrap, Hui, Sifre, van~den
  Driessche, Graepel, and Hassabis]{alphaZero:17}
David Silver, Julian Schrittwieser, Karen Simonyan, Ioannis Antonoglou, Aja
  Huang, Arthur Guez, Thomas Hubert, Lucas Baker, Matthew Lai, Adrian Bolton,
  Yutian Chen, Timothy Lillicrap, Fan Hui, Laurent Sifre, George van~den
  Driessche, Thore Graepel, and Demis Hassabis.
\newblock Mastering the game of {Go} without human knowledge.
\newblock \emph{Nature}, 2017.

\bibitem[Strehl et~al.(2009)Strehl, Li, and Littman]{strehl2009reinforcement}
Alexander~L Strehl, Lihong Li, and Michael~L Littman.
\newblock Reinforcement learning in finite {MDPs}: {PAC} analysis.
\newblock \emph{Journal of Machine Learning Research}, 10\penalty0 (11), 2009.

\bibitem[Suttle et~al.(2020)Suttle, Yang, Zhang, Wang, Ba{\c{s}}ar, and
  Liu]{suttle2020multi}
Wesley Suttle, Zhuoran Yang, Kaiqing Zhang, Zhaoran Wang, Tamer Ba{\c{s}}ar,
  and Ji~Liu.
\newblock A multi-agent off-policy actor-critic algorithm for distributed
  reinforcement learning.
\newblock \emph{{IFAC}-{PapersOnLine}}, 53\penalty0 (2):\penalty0 1549--1554,
  2020.
\newblock \doi{10.1016/j.ifacol.2020.12.2021}.

\bibitem[Tuynman and Ortner(2022)]{tuynman2022TransferRL}
Adrienne Tuynman and Ronald Ortner.
\newblock Transfer in reinforcement learning via regret bounds for learning
  agents.
\newblock \emph{arXiv preprint arXiv:2202.01182}, 2022.

\bibitem[Vinyals et~al.(2019)Vinyals, Babuschkin, Czarnecki, Mathieu, Dudzik,
  Chung, Choi, Powell, Ewalds, Georgiev, Oh, Horgan, Kroiss, Danihelka, Huang,
  Sifre, Cai, Agapiou, Jaderberg, Vezhnevets, Leblond, Pohlen, Dalibard,
  Budden, Sulsky, Molloy, Paine, G{\"{u}}l{\c{c}}ehre, Wang, Pfaff, Wu, Ring,
  Yogatama, W{\"{u}}nsch, McKinney, Smith, Schaul, Lillicrap, Kavukcuoglu,
  Hassabis, Apps, and Silver]{vinyals2019grandmaster}
Oriol Vinyals, Igor Babuschkin, Wojciech~M. Czarnecki, Micha{\"{e}}l Mathieu,
  Andrew Dudzik, Junyoung Chung, David~H. Choi, Richard Powell, Timo Ewalds,
  Petko Georgiev, Junhyuk Oh, Dan Horgan, Manuel Kroiss, Ivo Danihelka, Aja
  Huang, Laurent Sifre, Trevor Cai, John~P. Agapiou, Max Jaderberg,
  Alexander~Sasha Vezhnevets, R{\'{e}}mi Leblond, Tobias Pohlen, Valentin
  Dalibard, David Budden, Yury Sulsky, James Molloy, Tom~Le Paine, {\c{C}}aglar
  G{\"{u}}l{\c{c}}ehre, Ziyu Wang, Tobias Pfaff, Yuhuai Wu, Roman Ring, Dani
  Yogatama, Dario W{\"{u}}nsch, Katrina McKinney, Oliver Smith, Tom Schaul,
  Timothy~P. Lillicrap, Koray Kavukcuoglu, Demis Hassabis, Chris Apps, and
  David Silver.
\newblock Grandmaster level in {StarCraft} {II} using multi-agent reinforcement
  learning.
\newblock \emph{Nature}, 575\penalty0 (7782):\penalty0 350--354, 2019.
\newblock \doi{10.1038/s41586-019-1724-z}.

\bibitem[Wagenmaker et~al.(2022)Wagenmaker, Chen, Simchowitz, Du, and
  Jamieson]{AW-YC-MS-SSD-KJ:22}
Andrew Wagenmaker, Yifang Chen, Max Simchowitz, Simon~S Du, and Kevin Jamieson.
\newblock Reward-free {RL} is no harder than reward-aware {RL} in linear
  {Markov} decision processes.
\newblock In \emph{International Conference on Machine Learning}, pages
  22430--22456, 2022.

\bibitem[Wang et~al.(2020)Wang, Du, Yang, and Salakhutdinov]{RW-SD-LY-RS:20}
Ruosong Wang, Simon~S Du, Lin Yang, and Russ~R Salakhutdinov.
\newblock On reward-free reinforcement learning with linear function
  approximation.
\newblock In \emph{Advances in Neural Information Processing Systems},
  volume~33, pages 17816--17826, 2020.

\bibitem[Zhang et~al.(2020{\natexlab{a}})Zhang, Kakade, Basar, and
  Yang]{zhang2020model}
Kaiqing Zhang, Sham~M. Kakade, Tamer Basar, and Lin~F. Yang.
\newblock Model-based multi-agent {RL} in zero-sum {Markov} games with
  near-optimal sample complexity.
\newblock In Hugo Larochelle, Marc'Aurelio Ranzato, Raia Hadsell,
  Maria{-}Florina Balcan, and Hsuan{-}Tien Lin, editors, \emph{Advances in
  Neural Information Processing Systems}, volume~33, 2020{\natexlab{a}}.
\newblock URL
  \url{https://proceedings.neurips.cc/paper/2020/hash/0cc6ee01c82fc49c28706e0918f57e2d-Abstract.html}.

\bibitem[Zhang et~al.(2021)Zhang, Yang, and Ba{\c{s}}ar]{zhang2021multi}
Kaiqing Zhang, Zhuoran Yang, and Tamer Ba{\c{s}}ar.
\newblock Multi-agent reinforcement learning: A selective overview of theories
  and algorithms.
\newblock In \emph{Handbook of Reinforcement Learning and Control}, pages
  321--384. Springer International Publishing, 2021.
\newblock \doi{10.1007/978-3-030-60990-0_12}.

\bibitem[Zhang and Zavlanos(2019)]{zhang2019distributed}
Yan Zhang and Michael Zavlanos.
\newblock Distributed off-policy actor-critic reinforcement learning with
  policy consensus.
\newblock In \emph{Conference on Decision and Control}, pages 4674--4679, 2019.

\bibitem[Zhang et~al.(2020{\natexlab{b}})Zhang, Zhou, and Ji]{ZZ-ZY-XJ:20}
Zihan Zhang, Yuan Zhou, and Xiangyang Ji.
\newblock Almost optimal model-free reinforcement learning via
  reference-advantage decomposition.
\newblock In \emph{Advances in Neural Information Processing Systems},
  volume~33, pages 15198--15207, 2020{\natexlab{b}}.

\bibitem[Zhao et~al.(2022)Zhao, Tian, Lee, and Du]{zhao2021provably}
Yulai Zhao, Yuandong Tian, Jason~D. Lee, and Simon~S. Du.
\newblock Provably efficient policy optimization for two-player zero-sum
  {Markov} games.
\newblock In Gustau Camps{-}Valls, Francisco J.~R. Ruiz, and Isabel Valera,
  editors, \emph{International Conference on Artificial Intelligence and
  Statistics}, volume 151, pages 2736--2761, 2022.
\newblock URL \url{https://proceedings.mlr.press/v151/zhao22b.html}.

\bibitem[Zhu and Zhao(2022)]{zhu2020online}
Yuanheng Zhu and Dongbin Zhao.
\newblock Online minimax {Q} network learning for two-player zero-sum {Markov}
  games.
\newblock \emph{{IEEE} Transactions on Neural Networks and Learning Systems},
  33\penalty0 (3):\penalty0 1228--1241, mar 2022.
\newblock \doi{10.1109/tnnls.2020.3041469}.

\bibitem[Zhu et~al.(2021)Zhu, Zhu, Liu, and Liu]{zhu2021federated}
Zhaowei Zhu, Jingxuan Zhu, Ji~Liu, and Yang Liu.
\newblock Federated bandit: {A} gossiping approach.
\newblock In Longbo Huang, Anshul Gandhi, Negar Kiyavash, and Jia Wang,
  editors, \emph{{ACM} {SIGMETRICS} / International Conference on Measurement
  and Modeling of Computer Systems}, pages 3--4. {ACM}, 2021.
\newblock \doi{10.1145/3410220.3453919}.

\end{thebibliography}

\appendix
%\begin{appendices}
%
%

% \section{POLSVI Algorithm}
% \label{App:POLSVI}

\section{Further Explanation of Doubling Rounds}
\label{App:explain-dr} 

There are two ways to further understand when a doubling round may occur. 
\begin{itemize}
    \item If exploration is not very efficient across episodes, then we have that all the parallel agents are exploring similar directions or regions of the action-state space across episodes --- this could happen, for example, due to optimism biasing the agents in their exploration. Then, when the stochasticity of the environment surprises the agents by leading them to explore more diverse regions, the aggregation of this information in the newly computed covariance matrix will differentiate it more than usual with respect to the covariance matrix of the previous episode, thus triggering a doubling round. 
    \item If all agents were exploring the state-action space more efficiently across episodes by having more diversity in their  exploration, we would not expect such an abrupt change between information aggregation to happen since the novelty across episodes would be at a \emph{more stable level} and the stochasticity of the environment would have less effect in driving diversity in the exploration across the agents.
\end{itemize}
In either case, it is the inefficiency of exploration across parallel agents that introduces doubling rounds.

\section{RFMG-POLSVI Algorithm: Exploration Phase}
\label{App:RFMG-POLSVI}

\begin{algorithm}[h!]%[tb]
  \caption{Reward-Free Markov Game POLSVI (RFMG-POLSVI) --- Exploration phase}
  \label{alg:main_LIN_RFMG_POLSVI_exp}
\begin{algorithmic}[1]
    \STATE {\bfseries Input:} $P$, $K$, $\beta$, $\lambda$,
    \FOR{episode $k\in[K]$}
        % \STATE Receive initial state $s_0$
        \STATE $x_1^{k,p}\gets s_0$ for $p\in[P]$
        \STATE \# DONE BY CENTRAL SERVER:
        \STATE $Q_{H+1}^k(\cdot,\cdot,\cdot)\gets 0$
        \FOR{$h=H,\dots,1$}
            \STATE $\Lambda_{h}^k\gets\lambda I_d + \sum^P_{p=1}\sum^{k-1}_{\tau=1}\phi(x_h^{\tau,p},a_h^{\tau,p},b_h^{\tau,p})\phi(x_h^{\tau,p},a_h^{\tau,p},b_h^{\tau,p})^\top$
            \STATE 
        $u^k_h(\cdot,\cdot,\cdot)\gets\min\{\beta(\phi(\cdot,\cdot,\cdot)^\top(\Lambda_{h}^{k})^{-1}\phi(\cdot,\cdot,\cdot))^{1/2},H\}$
        \STATE $r^k_h(\cdot,\cdot,\cdot)\gets u^k_h(\cdot,\cdot,\cdot)/H$
        \STATE $w_{h}^k\gets(\Lambda_{h}^{k})^{-1} \sum^P_{p=1}\sum^{k-1}_{\tau=1}\phi(x_h^{\tau,p},a_h^{\tau,p},b_h^{\tau,p})\max_{(a,b)\in\A\times\cB} Q^{k}_{h+1}(x_{h+1}^{\tau,p},a,b)$
       \STATE $Q^k_{h}(\cdot,\cdot,\cdot)\gets \Pi_{[0,H]}[(w_{h}^k)^\top\phi(\cdot,\cdot,\cdot)+r^k_h(\cdot,\cdot,\cdot)+u^k_h(\cdot,\cdot,\cdot)]$
        \ENDFOR
        \STATE \# DONE BY EACH AGENT $p\in[P]$ IN PARALLEL:
        %\FOR{$p\in[P]$} %Not sure if put the p as in outer loop
            \FOR{$h\in[H]$}
                \STATE $(a_h^{k,p},b_h^{k,p}) \in \arg\max_{(a,b)\in\A\times\cB} Q^k_h(x_h^{k,p},a,b)$ \# GREEDY POLICY
                \STATE Observe $x_{h+1}^{k,p}$
            \ENDFOR
    \ENDFOR
%
%  \REPEAT
%  \STATE \# TRAIN 
%  \UNTIL{done}
\STATE {\bfseries Return} $\{(x^{k,p}_h,a^{k,p}_h,b^{k,p}_h)\}_{(h,k,p)\in[H]\times[K]\times[P]}$ \# COLLECTED TRAJECTORIES BY THE AGENTS
\end{algorithmic}
\end{algorithm}

%===================================
\section{Useful Technical Results}

For the rest of the appendix, let $\Z_{\geq 0}$ ($\Z_{\geq 1}$) be the set of non-negative (positive) integers.

\begin{proposition}[Linear Q-function with bounded parameters; Proposition~2.3 and Lemma~B.1 in~\citep{CJ-ZY-ZW-MIJ:20}]
\label{prop:lin-Q}
Consider a linear MDP $\M$. For any policy $\pi$, there exist paremeters $w^{\pi}_h\in\R^d$, $h\in[H]$, such that $Q_h^\pi(x, a) = \langle\phi (x, a),w^{\pi}_h\rangle$ for any $(x, a) \in \S\times\A$ and $\norm{w^\pi_h} \leq 2H\sqrt{d}$.
\end{proposition}

Proposition~\ref{prop:lin-Q} holds virtually the same for linear Markov games.

\begin{lemma}[Bound on quadratic form]
\label{lem:basic_ineq} Consider $\Lambda_{A,B} = \lambda I_d + \sum_{a=1}^A\sum_{b=1}^B \phi_{a,b} \phi_{a,b}^\top$ where $\phi_{a,b} \in \R^d$ and $\lambda > 0$. Then,
\begin{equation*}
\sum_{a=1}^{A}\sum_{b=1}^{B} \phi_{a,b}^\top (\Lambda_{A,B})^{-1} \phi_{a,b} \leq d.
\end{equation*}
\end{lemma}
\begin{proof}
Observe that $$\sum_{a=1}^{A}\sum_{b=1}^B \phi_{a,b}^\top (\Lambda_{A,B})^{-1} \phi_{a,b}
= \sum_{a=1}^{A}\sum_{b=1}^B\tr( \phi_{a,b}^\top (\Lambda_{A,B})^{-1} \phi_{a,b})
= \tr((\Lambda_{A,B})^{-1} \sum_{a=1}^{A}\sum_{b=1}^B\phi_{a,b}\phi_{a,b}^\top).$$
Given the eigenvalue decomposition  $\sum_{a=1}^A\sum_{b=1}^B\phi_{a,b}\phi_{a,b}^\top = U \diag(\lambda_1,  \ldots, \lambda_d)  U^\top$ with $\lambda_i\geq 0$, $i\in[d]$, we obtain $\Lambda_{A,B} = U \diag(\lambda_1+\lambda,  \ldots, \lambda_d+\lambda) U^\top$, and so $\tr((\Lambda_{A,B})^{-1} \sum_{a=1}^{A}\sum_{b=1}^B\phi_{a,b}\phi_{a,b}^\top)=
\sum_{i=1}^d \lambda_i/(\lambda_i + \lambda) \leq d$.
\end{proof}

We now introduce an auxiliary definition. 
\begin{definition}[Sequences $\L_P$, $\bar{\L}_P$ and their truncation]
We now define the infinite sequence $\L_P$, with integer $P\geq 1$, as %$(1,1),(1,2),\dots,(1,P)$, %$(2,1),(2,2),\dots,(2,P)$,$(3,1),\dots$;
$$
(1,1),(1,2),\dots,(1,P),(2,1),(2,2),\dots,(2,P),(3,1),\dots
$$
where the first term of the elements is an increasing sequence that takes values in $\Z_{\geq 1}$ but the second term only takes values in $[P]$ periodically. We can index sequences using $\L_P$, as for example in $\{\varphi_{(a,b)}\}_{(a,b)\in\L_P}$ where $\varphi_{(a,b)}\in \R^n$, $n\geq 1$. We denote by $\L_P(\bar{a},\bar{b})$ the finite sequence resulting from the truncation of the sequence $\L_P$ at its element $(\bar{a},\bar{b})\in\L_P$. 
Similarly, we define the infinite sequence $\bar{\L}_P$, with integer $P\geq 1$, by appending the sequence $(0,1),(0,2),\dots,(0,P)$ at the beginning of $\L_P$. %$(1,1),(1,2),\dots,(1,P)$,$(2,1),\dots$;
%$$
%(0,1),(0,2),\dots,(0,P),(1,1),(1,2),\dots,(1,P),(2,1),\dots\;.
%$$
Thus $\L_P\subsetneq\bar{\L}_P$.
For any appropriate element $(\bar{a},\bar{b})$ of $\L_P$ or $\bar{\L}_P$, we have that $(\bar{a},\bar{b})^{-k}$, $k\geq 1$, represents the $k$th previous element to $(\bar{a},\bar{b})$.
\end{definition}

Using the recently introduced definition, we present another technical lemma.
\begin{lemma}[Concentration bound for self-normalized processes]
\label{lem:self_norm_covering}
Let $B\in\Z_{\geq 1}$. 
Let $\{\F_{(a,b)}\}_{(a,b)\in\bar{\L}_B}$ be a filtration. Let $\{x_{(a,b)}\}_{(a,b)\in\L_B}$ be a stochastic process on $\S$ such that $x_{(a,b)}\in\F_{(a,b)}$, and let $\{\phi_{(a,b)}\}_{(a,b)\in\L_B}$ be an $\R^d$-valued stochastic process such that $\phi_{(a,b)} \in \F_{(a,b)^{-1}}$ and $\norm{\phi_{(a,b)}}\leq 1$. Let $\G$ be a function class of real-valued functions such that $\sup_{x\in\S} |g(x)| \leq H$ for any $g\in\G$, and with $\epsilon$-covering number $\mathcal{N}_{\epsilon}$ with respect to the distance $\mathrm{dist}(g, g') = \sup_{x\in S} |g(x) - g'(x)|$. 
Let $\Lambda_{A,B} = \lambda I_d + \sum_{a=1}^A\sum_{b=1}^B \phi_{(a,b)} \phi_{(a,b)}^\top$. Then for any $A\in\Z_{\geq 1}$ for every $A,B \in \Z_{\geq 1}$, every $g \in \G$, and any $\delta\in(0,1]$, we have that with probability at least $1-\delta$,
\begin{multline*} 
\norm{\sum_{a = 1}^A\sum_{b=1}^B \phi_{(a,b)} \{ g(x_{(a,b)}) - \E[g(x_{(a,b)})|\F_{(a,b)^{-1}}] \} }^2_{\Lambda_{A,B}^{-1}}\\
\leq 4H^2 \left[ \frac{d}{2}\log\biggl( \frac{\lambda+AB/d}{\lambda}\biggr )  + \log\frac{\mathcal{N}_{\epsilon}}{\delta}\right]  + \frac{8A^2B^2\epsilon^2}{\lambda}.
\end{multline*}
\end{lemma}
\begin{proof}
First, from our assumptions, for any $g \in \G$, there exists a $\tilde{g}$ in the $\epsilon$-covering such that
$g = \tilde{g} + \Delta_g$ with $\sup_{x\in\S} |\Delta_g(x)| \leq \epsilon$. Then,
\begin{equation}
\label{eq:conc-aux}
\begin{aligned}
&\norm{\sum_{a = 1}^A\sum_{b=1}^B \phi_{(a,b)} \{ g(x_\tau) - \E[g(x_\tau)|\F_{(a,b)^{-1}}] \}  }^2_{\Lambda_{A,B}^{-1}}\\
& \leq  2\underbrace{\norm{\sum_{a = 1}^A\sum_{b=1}^B \phi_{(a,b)} \{ \tilde{g}(x_{(a,b)}) - \E[\tilde{g}(x_{(a,b)})|\F_{(a,b)^{-1}}]\} }^2_{\Lambda_{A,B}^{-1}}}_{\textrm{(I)}}\\
&\quad+ 2\underbrace{\norm{\sum_{a = 1}^A\sum_{b=1}^B \phi_{(a,b)} \{ \Delta_g(x_{(a,b)}) - \E[\Delta_g(x_{(a,b)})|\F_{(a,b)^{-1}}] \} }^2_{\Lambda_{A,B}^{-1}}}_{\textrm{(II)}},
\end{aligned}
\end{equation}
where we used $\norm{a+b}\leq \norm{a}+\norm{b}\implies \norm{a+b}^2\leq \norm{a}^2+\norm{b}^2+2\norm{a}\norm{b}\leq 2\norm{a}^2+2\norm{b}^2$ for any $a,b\in\R^d$, and which actually holds for any weighted Euclidean norm. 

We start by analyzing the term (I) in equation~\eqref{eq:conc-aux}. Let $\varepsilon_{(a,b)}:=\tilde{g}(x_{(a,b)}) - \E[\tilde{g}(x_{(a,b)})|\F_{(a,b)^{-1}}]$. Now, we observe that 1) $\E[\varepsilon_{(a,b)}|\F_{(a,b)^{-1}}]=0$ and 2) $\varepsilon_{(a,b)}\in[-H,H]$ since $\tilde{g}(x_{(a,b)})\in[0,H]$. From these two facts we obtain that $\varepsilon_{(a,b)}|\F_{(a,b)^{-1}}$ is $H$-sub-Gaussian. Therefore we can apply the concentration bound of self-normalized processes from Theorem~1 of~\citep{YAY-DP-CS:11} along with a union bound over the $\epsilon$-covering of $\G$ to conclude that, with probability at least $1-\delta$,
\begin{multline}
\label{eq:upp-b-lem-aux1}
\textrm{(I)}=\norm{\sum_{a=1}^A\sum_{b=1}^B\phi_{(a,b)}\varepsilon_{(a,b)}}_{\Lambda_{A,B}^{-1}}^2\leq \log\left(\frac{\det(\Lambda_{A,B})^{1/2}\det(\lambda I_d)^{-1/2}}{\delta/\N_\epsilon}\right)\\
\overset{\textrm{(a)}}{\leq} 2H^2\left(
\frac{d}{2}\log\left(\frac{\lambda+AB/d}{\lambda}\right)+\log\left(\frac{\N_\epsilon}{\delta}\right)
\right),
\end{multline}
where (a) follows from $\det(\lambda I_d)=\lambda^d$ and from the determinant-trace inequality from Lemma~10 in~\citep{YAY-DP-CS:11} which let us obtain $\det(\Lambda_{A,B})\leq(\lambda+AB/d)^d$.

Now we analyze the term (II) in equation~\eqref{eq:conc-aux}. Let $\bar{\varepsilon}_{(a,b)}:=\Delta_g(x_{(a,b)}) - \E[\Delta_g(x_{(a,b)})|\F_{(a,b)^{-1}}]$. Then, 
\begin{multline*}
\norm{\sum_{a=1}^A\sum_{b=1}^B\phi_{(a,b)}\bar{\varepsilon}_{(a,b)}}\leq \sum_{a=1}^A\sum_{b=1}^B\norm{\phi_{(a,b)}\bar{\varepsilon}_{(a,b)}}\overset{\textrm{(a)}}{\leq} \sum_{a=1}^A\sum_{b=1}^B|\bar{\varepsilon}_{(a,b)}|\\
\leq \sum_{a=1}^A\sum_{b=1}^B|\Delta_g(x_{(a,b)})|+|\E[\Delta_g(x_{(a,b)})|\F_{(a,b)^{-1}}]|\leq \sum_{a=1}^A\sum_{b=1}^B 2\epsilon=2AB\epsilon,
\end{multline*}
where (a) follows from $\norm{\phi_{(a,b)}}\leq 1$. Thus, using this result, we obtain
$$
\textrm{(II)}\leq\frac{1}{\lambda}\norm{\sum_{a=1}^A\sum_{b=1}^B\phi_{(a,b)}\bar{\varepsilon}_{(a,b)}}^2\leq \frac{1}{\lambda}4A^2B^2\epsilon^2.$$

We finish the proof by multiplying by two the terms (I) and (II), and then adding them up to use them as an upper bound to~\eqref{eq:conc-aux} .
\end{proof}

%===================================

\section{Proof of Theorem~\ref{thm:main-paralel}}
\label{subsec:proof_main_parallel}
For simplicity, we will use the following notation: at episode $k$, we denote $\pi^{k,p}=\{\pi^{k,p}_h\}_{h\in[H]}$ as the greedy policy induced by $\{Q_h^k\}_{h=1}^H$ as performed by agent $p\in[P]$ (line 13 of Algorithm~\ref{alg:main_LIN_UCB_LSVI}), and we let the value function $V_h^{k,p}(x_h^{k,p})= Q_h^k(x_h^{k,p},\pi^{k,p}_h(x_h^{k,p}))=\max_{a\in\A} Q_h^k(x_h^{k,p},a)$, and with some abuse of notation
$V_h^{k,p}(x)= Q_h^k(x,\pi^{k,p}_h(x))$. We also set $\phi^{k,p}_h := \phi(x^{k,p}_h, a^{k,p}_h)$.

We now bound the parameters $\{w^{k}_h\}_{h\in[H],k\in[K]}$ from POLSVI (line 8 of algorithm~\ref{alg:main_LIN_UCB_LSVI}) using Lemma~\ref{lem:basic_ineq}.

\begin{lemma}[Parameter bound for POLSVI]
\label{lem:wn_estimate}
For any $(k, h) \in[K]\times[H]$, the parameter $w^k_h$ in the POLSVI algorithm satisfies
$\norm{w^k_h}\leq(1+H) \sqrt{\frac{d(k-1)P}{\lambda}}$.
\end{lemma}
\begin{proof}
For any vector $v \in \R^d$,
\begin{align*}
|v^\top w^k_h| & = |v^\top (\Lambda^k_h)^{-1} \sum_{\tau=1}^{k-1}\sum_{p=1}^P \phi^{\tau,p}_h [r^{\tau,p}_h + \max_{a\in\A} Q_{h+1}^k(x^{\tau,p}_{h+1}, a)]|\\
& \overset{\textrm{(a)}}{\leq}(1+H)\sum_{\tau = 1}^{k-1}\sum_{p=1}^P  |v^\top (\Lambda^k_h)^{-1} \phi^{\tau,p}_h|\\
&\overset{(b)}{\leq} (1+H)\sqrt{ \bigg[ \sum_{\tau = 1}^{k-1}\sum_{p=1}^P  v^\top (\Lambda^k_h)^{-1}v\bigg]  \biggl [ \sum_{\tau = 1}^{k-1}\sum_{p=1}^P  (\phi^{\tau,p}_h)^\top (\Lambda^k_h)^{-1}\phi^{\tau,p}_h\bigg] }\\
& \overset{\textrm{(c)}}{\leq} (1+H)\sqrt{d}\sqrt{\sum_{\tau = 1}^{k-1}\sum_{p=1}^P  v^\top (\Lambda^k_h)^{-1}v}\\
&\overset{\textrm{(d)}}{\leq}(1+H)\sqrt{\frac{d(k-1)P}{\lambda}}\norm{v}, %
\end{align*}
where (a) follows from the bounded rewards and $Q^k_{h+1}(\cdot,\cdot)\leq H$; (b) from applying Cauchy-Schwarz twice as in the following series of inequalities: given $q = (q_1,\dots,q_m)$ and $q = (p_1,\dots,p_m)$ where $q_i$ and $p_i$ are vectors of arbitrary dimension we have $\sum^m_{i=1}|q_i^\top p_i|\leq \sum^m_{i=1}\norm{q_i}\norm{p_i}\leq \sqrt{\sum^m_{i=1}\norm{q_i}}\sqrt{\sum^m_{i=1}\norm{p_i}}$ ; (c) follows from Lemma~\ref{lem:basic_ineq}; and (d) from $(\Lambda_h^k)^{-1}\preceq \lambda^{-1}I_d$. The proof concludes by considering that $\norm{w^k_h} = \max_{v:\norm{v} = 1} |v^\top w^k_h|$.
\end{proof}

Now we use Lemma~\ref{lem:self_norm_covering} to prove a useful concentration bound for POLSVI.

\begin{lemma}[Concentration bound on value functions for POLSVI] \label{lem:stochastic_term}
Consider the setting of Theorem~\ref{thm:main-paralel}. There exists an absolute constant $C$ independent of $c_{\beta}$ such that for any fixed $\delta\in(0, 1)$, the following event $\mathcal{E}$ holds with probability at least $1-\delta$,
\begin{multline*}
\forall (k, h)\in [K]\times [H]: \quad \norm{\sum_{\tau = 1}^{k-1}\sum_{p=1}^P \phi^{\tau,p}_h [V^{k}_{h+1}(x^{\tau,p}_{h+1}) - \Pe_h V^{k}_{h+1}(x_h^{\tau,p}, a_h^{\tau,p})]}_{(\Lambda^k_h)^{-1}}\\
\leq CdH\sqrt{\log [(c_\beta+1)dKHP/\delta]}. 
\end{multline*}
\end{lemma}
\begin{proof}
Define the function class 
%such that, for any $V\in\mathcal{V}$, we have  $V:\S\to\R$ with the form 
\begin{equation}
\label{eq:function_class}
\mathcal{V} = \left\{V:\S\to\R\;\Big|\; V(\cdot)=\min\left\{\max_{a\in\A} \bar{w}^\top\phi(\cdot, a) + \beta \sqrt{\phi(\cdot, a)^\top\bar{\Lambda}^{-1} \phi(\cdot, a)}, H \right\}\right\},
\end{equation}
where $\norm{\bar{w}}\leq L$, the minimum eigenvalue of $\bar{\Lambda}$ is greater or equal than $\lambda$, and $\norm{\phi(x, a)}\leq 1$ for all $(x,a)\in\S\times\A$. Let $\mathcal{N}_{\epsilon}$ be the $\epsilon$-covering number of $\mathcal{V}$ with respect to the distance $\operatorname{dist}(V, V') = \sup_{x\in\S} |V(x) - V'(x)|$. Then, we can use Lemma~D.6 of~\citep{CJ-ZY-ZW-MIJ:20} to obtain the bound
\begin{equation}
\label{eq:cover-aux_0}
\log \mathcal{N}_{\epsilon} \le d  \log (1+ 4L / \epsilon ) + d^2 \log \bigl [ 1 +  8 d^{1/2} \beta^2  / (\lambda\epsilon^2)  \bigr ].
\end{equation}

Now, let us go back to the POLSVI algorithm. We obtain that, with probability at least $1-\delta$, $\delta\in(0,1)$,
\begin{equation}
    \label{eq:bound_aux_long}
\begin{aligned}
&\norm{\sum_{\tau = 1}^{k-1}\sum_{p=1}^P \phi^{\tau,p}_h [V^{k}_{h+1}(x^{\tau,p}_{h+1}) - \Pe_h V^{k}_{h+1}(x_h^{\tau,p}, a_h^{\tau,p})]}_{(\Lambda^k_h)^{-1}}^2\\
&\overset{\textrm{(a)}}{\leq} 4H^2 \left[ \frac{d}{2}\log\biggl( \frac{\lambda+(k-1)P/d}{\lambda}\biggr )  + \log\N_\epsilon + \log\frac{1}{\delta}
\right]  + \frac{8(k-1)^2P^2\epsilon^2}{\lambda}\\
&\overset{(b)}{\leq} 4H^2 \left[ \frac{d}{2}\log\biggl( \frac{\lambda+(k-1)P/d}{\lambda}\biggr )  
+ d\log \left(1+ \frac{4(1+H) \sqrt{d(k-1)P}}{\epsilon\sqrt{\lambda}}\right) \right. \\
&\quad\left. + d^2 \log\left( 1 + \frac{8 d^{1/2}\beta^2}{\lambda\epsilon^2}\right) + \log\frac{1}{\delta}
\right]  + \frac{8(k-1)^2P^2\epsilon^2}{\lambda}
\end{aligned}
\end{equation}
where (a) is a direct application of Lemma~\ref{lem:self_norm_covering}; and (b) follows from the realization that from lines 9 and 13 in algorithm~\ref{alg:main_LIN_UCB_LSVI}, $V^k_{h+1}(\cdot)\in\mathcal{V}$ and so has covering number upper bounded as in~\eqref{eq:cover-aux_0} with $L=(1+H) \sqrt{\frac{d(k-1)P}{\lambda}}$ by using the bound from Lemma~\ref{lem:wn_estimate}.

Recalling that $\lambda = 1$ and $\beta=c_\beta dH\iota$ with $\iota=\log(dKHP/\delta)$ in the setting of Theorem~\ref{thm:main-paralel}, we claim that, after setting $\epsilon = \frac{dH}{KP}$ in our previous equation, there exists an absolute constant $C > 0$ independent of $c_\beta$ such that 
\begin{equation}
\label{eq:final-bound-1}
\norm{\sum_{\tau = 1}^{k-1}\sum_{p=1}^P \phi^{\tau,p}_h [V^{k}_{h+1}(x^{\tau,p}_{h+1}) - \Pe_h V^{k}_{h+1}(x_h^{\tau,p}, a_h^{\tau,p})]}_{(\Lambda^k_h)^{-1}}^2 \leq C d^2  H^2 \log ((c_\beta+1)dKHP/
\delta).
\end{equation}
Proving this claim would conclude the proof.

We first introduce a couple of useful results:
\begin{align}
\label{eq:iota_bound}
&\iota^2=\log\left(\frac{dKHP}{\delta}\right)\geq \log(dKHP)\geq \log(2\cdot 2)=\log(4)>1,\\ 
%\end{equation}
%and so
%\begin{equation}
\label{eq:upp_low_bound}
& \log \left(\frac{(c_\beta+1)dKHP}{\delta}\right)= 
\log(c_\beta+1)+\iota\geq \iota>1.
\end{align}
Now we start by replacing $\lambda =1$ in the right-hand side of~\eqref{eq:bound_aux_long},
\begin{align*}
&4H^2 \left[ \frac{d}{2}\log\left( 1+\frac{(k-1)P}{d}\right)  
+ d\log \left(1+ \frac{4(1+H) \sqrt{d(k-1)P}}{\epsilon}\right) \right.\\
&\quad \left. + d^2 \log\left(\frac{1}{\delta}\left( 1 + \frac{8 d^{1/2}\beta^2}{\epsilon^2}\right)\right) \right]  + 8(k-1)^2P^2\epsilon^2\\
&\leq 4H^2d^2 \left[\log\left( 1+\frac{(k-1)P}{d}\right)  
+ \log \left(1+ \frac{4(1+H) \sqrt{d(k-1)P}}{\epsilon}\right) \right.\\
&\quad \left. + \log\left(\frac{1}{\delta}\left( 1 + \frac{8 d^{1/2}\beta^2}{\epsilon^2}\right)\right) \right]  + 8(k-1)^2P^2\epsilon^2.
\end{align*}
Now we replace $\epsilon=\frac{dH}{KP}$ in the previous expression,
\begin{align*}
&4d^2H^2\left[\log\left( 1+\frac{(k-1)P}{d}\right)  
+ \log \left(1+ \frac{4(1+H) (k-1)^{1/2}KP^{3/2}d^{1/2}}{dH}\right) \right.\\
& \quad \left. + \log\left(\frac{1}{\delta}\left( 1 + \frac{8 d^{1/2}\beta^2 K^2P^2}{d^2H^2}\right)\right) \right] + 8d^2H^2\frac{(k-1)^2}{K^2}\\
&= 4d^2H^2\left[\log\left( 1+\frac{(k-1)P}{d}\right)  
+ \log \left(1+ \frac{4(1+H) (k-1)^{1/2}KP^{3/2}}{d^{1/2}H}\right) \right.\\
& \quad \left. + \log\left(\frac{1}{\delta}\left( 1 + \frac{8 \beta^2K^2P^2}{d^{3/2}H^2}\right)\right) \right] + 8d^2H^2\frac{(k-1)^2}{K^2}\\
&\leq 4d^2H^2\left[\log\left( 1+\frac{KP}{d}\right)  
+ \log \left(1+ \frac{4\left(1+\frac{1}{H}\right) K^{3/2}P^{3/2}}{d^{1/2}}\right) + \log\left(\frac{1}{\delta}\left( 1 + \frac{8 \beta^2K^2P^2}{d^{3/2}H^2}\right)\right) \right]\\
&\quad +8d^2H^2.\\
&\leq 4d^2H^2\left[\log\left( 1+\frac{KP}{d}\right)  
+ \log \left(1+ \frac{8K^{3/2}P^{3/2}}{d^{1/2}}\right) + \log\left(\frac{1}{\delta}\left( 1 + \frac{8 \beta^2K^2P^2}{d^{3/2}H^2}\right)\right) \right]\\
&\quad +8d^2H^2.
\end{align*}
Now we replace $\beta=c_\beta dH\sqrt{\iota}$ in the previous expression,
\begin{equation}
\label{eq:to_upper_bound}
\begin{aligned}
&4d^2H^2\left[\log\left( 1+\frac{KP}{d}\right)  
+ \log \left(1+ \frac{8K^{3/2}P^{3/2}}{d^{1/2}}\right) + \log\left(\frac{1}{\delta}\left( 1 + \frac{8 c_\beta^2d^2H^2\iota K^2P^2}{d^{3/2}H^2}\right)\right) \right]\\
&\quad +8d^2H^2\\
&=4d^2H^2\left[\log\left( 1+\frac{KP}{d}\right)  
+ \log \left(1+ \frac{8K^{3/2}P^{3/2}}{d^{1/2}}\right) + \log\left(\frac{1}{\delta}\left( 1 + 8 c_\beta^2d^{1/2}\iota K^2P^2\right)\right) \right]\\
&\quad +8d^2H^2\\
&\leq \underbrace{8d^2H^2\log \left(1+ \frac{8K^{3/2}P^{3/2}}{d^{1/2}}\right)}_{\textrm{(I)}} + 
\underbrace{4d^2H^2\log\left(\frac{1}{\delta}\left( 1 + 8 c_\beta^2d^{1/2}\iota K^2P^2\right)\right)}_{\textrm{(II)}}\\
&\quad 
+8d^2H^2\log\left(\frac{(c_\beta+1)dKHP}{\delta}\right)
\end{aligned}
\end{equation}
where the inequality has made use of~\eqref{eq:upp_low_bound}. We now upper bound the terms highlighted in~\eqref{eq:to_upper_bound}.
Then, 
\begin{align*}
\textrm{(I)}&\leq 8d^2H^2\log \left(1+ 8K^{3/2}P^{3/2}\right)\leq 8d^2H^2\log \left(9K^2P^2\right)\\
&\leq 8d^2H^2\log \left(9(1+c_\beta)^2(dKPH)^2\right)\\
&\overset{\textrm{(a)}}{\leq} 8d^2H^2\log \left(\frac{(1+c_\beta)^2(dKPH)^2}{\delta^2}\right)+8d^2H^2\log(9)\log \left(\frac{(c_\beta+1)dKHP}{\delta}\right)\\
&=(16+8\log(9))d^2H^2\log \left(\frac{(c_\beta+1)dKHP}{\delta}\right),
\end{align*}
where (a) follows from~\eqref{eq:upp_low_bound}. 
For the other term,
\begin{align*}
\textrm{(II)}
&\overset{\textrm{(a)}}{\leq} 
4d^2H^2\log\left(\frac{8(c_\beta+1)^2\iota(dKHP)^2}{\delta}\right)\\
&\overset{(b)}{\leq} 
4d^2H^2\log\left(\frac{(c_\beta+1)^2\iota(dKHP)^2}{\delta^2}\right)+4d^2H^2\log(8)\\
&=
4d^2H^2\log\left(\frac{(c_\beta+1)^2(dKHP)^2}{\delta^2}\right)+4d^2H^2\log(\iota)+4d^2H^2\log(8)\\
&\overset{\textrm{(c)}}{\leq} 8d^2H^2\log\left(\frac{(c_\beta+1)dKHP}{\delta}\right)+4d^2H^2\iota+4d^2H^2\log(8)\\
&\overset{\textrm{(d)}}{\leq} (12+4\log(8))d^2H^2\log\left(\frac{(c_\beta+1)dKHP}{\delta}\right)
\end{align*}
where (a) follows from $c_\beta>0$, (b) from $\delta^2<\delta$, (c) from $\log(\iota)<\iota$ (since $\iota>1$ from~\eqref{eq:iota_bound}), and (d) from $\iota\leq \log\left(\frac{(c_\beta+1)dKHP}{\delta}\right)$ and from~\eqref{eq:upp_low_bound}.

Now, joining the upper bounds for $(I)$ and $(II)$ in~\eqref{eq:to_upper_bound}, we finally obtain
\begin{multline*}
\norm{\sum_{\tau = 1}^{k-1}\sum_{p=1}^P \phi^{\tau,p}_h [V^{k}_{h+1}(x^{\tau,p}_{h+1}) - \Pe_h V^{k}_{h+1}(x_h^{\tau,p}, a_h^{\tau,p})]}_{(\Lambda^k_h)^{-1}}^2\\ \leq 
(36+8\log(9)+4\log(8))d^2H^2 \log ((c_\beta+1)dKHP/
\delta)
\end{multline*}
which proves the claim and thus the proof.
\end{proof}

The previous lemma will be used to upper bound the difference between the estimated Q-function by POLSVI before adding the optimism bonus (see line 9 in algorithm~\ref{alg:main_LIN_UCB_LSVI}) and the Q-function for any fixed policy. The optimism bonus will play an important role in such an upper bound.

\begin{lemma}[Bounded difference between the estimated Q-function before optimism and an arbitrary Q-function]
\label{lem:basic_relation} Consider the setting of Theorem~\ref{thm:main-paralel}. There exists an absolute constant $c_\beta>0$ such that for $\beta = c_\beta dH\sqrt{\iota}$ with $\iota = \log (dKHP/\delta)$ and
any fixed policy $\bar{\pi}$, given the event $\mathcal{E}$ defined in Lemma \ref{lem:stochastic_term}, we have for all $(x, a, h, k) \in \S\times\A\times[H]\times[K]$ that
\begin{equation*}
\langle\phi(x, a), w^k_h\rangle - Q_h^{\bar{\pi}}(x, a)  =  \Pe_h (V^{k}_{h+1} - V^{\bar{\pi}}_{h+1})(x, a) + \Delta^k_h(x, a),
\end{equation*}
for some $\Delta^k_h(x, a)$ such that $|\Delta^k_h(x, a)| \leq \beta \sqrt{\phi(x, a)^\top (\Lambda^k_h)^{-1}  \phi(x, a)}$.
\end{lemma}
\begin{proof}
For any $k\in[K]$, 
\begin{align*}
w^k_h -  w^{\bar{\pi}}_h
&= (\Lambda^k_h)^{-1} \sum_{\tau = 1}^{k-1}\sum^P_{p=1} \phi^{\tau,p}_h (r^{\tau,p}_h + V^{k}_{h+1}(x^{\tau,p}_{h+1}))- w^{\bar{\pi}}_h  \\
&\overset{\textrm{(a)}}{=} (\Lambda^k_h)^{-1} \sum_{\tau = 1}^{k-1}\sum^P_{p=1} \phi^{\tau,p}_h ({\phi^{\tau,p}_h}^\top w^{\bar{\pi}}_h - \Pe_h V^{\bar{\pi}}_{h+1}(x_h^{\tau,p}, a_h^{\tau,p}) + V^{k}_{h+1}(x^{\tau,p}_{h+1}))- w^{\bar{\pi}}_h  \\
&= (\Lambda^k_h)^{-1} \left(
\left(
\sum_{\tau = 1}^{k-1}\sum^P_{p=1}
\phi^{\tau,p}_h(\phi^{\tau,p}_h)^\top
-\Lambda_h^k\right)w^{\bar{\pi}}_h \right. \\
&\quad \left.  
 + \sum_{\tau = 1}^{k-1}\sum^P_{p=1} \phi^{\tau,p}_h \bigl (V^{k}_{h+1}(x^{\tau,p}_{h+1}) - \Pe_h V^{\bar{\pi}}_{h+1}(x_h^{\tau,p}, a_h^{\tau,p}) \bigr )\right) \\
&\overset{(b)}{=} (\Lambda^k_h)^{-1} \left(-\lambda w^{\bar{\pi}}_h + \sum_{\tau = 1}^{k-1}\sum^P_{p=1} \phi^{\tau,p}_h (V^{k}_{h+1}(x^{\tau,p}_{h+1}) - \Pe_h V^{\bar{\pi}}_{h+1}(x_h^{\tau,p}, a_h^{\tau,p}))\right) \\
&= \underbrace{-\lambda (\Lambda^k_h)^{-1} w^{\bar{\pi}}_h}_{\textrm{(I)}} + 
\underbrace{(\Lambda^k_h)^{-1} \sum_{\tau = 1}^{k-1}\sum_{p =1}^P \phi^{\tau,p}_h (V^{k}_{h+1}(x^{\tau,p}_{h+1}) - \Pe_h V^{k}_{h+1}(x_h^{\tau,p}, a_h^{\tau,p}))}_{\textrm{(II)}} \\
&\quad + \underbrace{(\Lambda^k_h)^{-1}\sum_{\tau = 1}^{k-1}\sum_{p =1}^P \phi^{\tau,p}_h \Pe_h (V^{k}_{h+1} - V^{\bar{\pi}}_{h+1})(x_h^{\tau,p}, a_h^{\tau,p})}_{\textrm{(III)}}. %
\end{align*}
where (a) follows from the fact that  for any $(x, a, h) \in \S \times \A \times [H]$, $Q^{\bar{\pi}}_h(x, a) :=  \langle \phi(x, a), w^{\bar{\pi}}_h\rangle = (r_h + \Pe_h V^{\bar{\pi}}_{h+1})(x, a) 
$ for some $w^{\bar{\pi}}_h\in\R^d$ (this follows from Proposition~\ref{prop:lin-Q} and the Bellman equation); and (b) follows from the definition of $\Lambda_h^k$ in the POLSVI algorithm. Since $\langle\phi(x, a), w^k_h\rangle - Q_h^{\bar{\pi}}(x, a)=\langle\phi(x, a), w^k_h-w^{\bar{\pi}}_h\rangle$ for any $(x,a)\in\S\times\A$, then we look to bound the inner product of each of the terms (I) -- (III) with the term $\phi(x, a)$.

We first analyze the term (I),
\begin{multline*}
|\langle\phi(x,a),\textrm{(I)}\rangle|=|\langle \phi(x, a),\lambda (\Lambda^k_h)^{-1} w^{\bar{\pi}}_h\rangle| = |\lambda \langle (\Lambda^k_h)^{-1/2}\phi(x, a), (\Lambda^k_h)^{-1/2} w^{\bar{\pi}}_h\rangle|\\
\leq \lambda \norm{w_h^{\bar{\pi}}}_{ (\Lambda^k_h)^{-1}} \sqrt{\phi(x, a)^\top (\Lambda^k_h)^{-1}  \phi(x, a)}
\leq \sqrt{\lambda} \norm{w_h^{\bar{\pi}}} \sqrt{\phi(x, a)^\top (\Lambda^k_h)^{-1}  \phi(x, a)}
\end{multline*}
where the last inequality follows from $\norm{\,\cdot\,}_{(\Lambda_h^k)^{-1}}\leq \frac{1}{\sqrt{\lambda}}\norm{\,\cdot\,}$.

For the term (II), since the event $\mathcal{E}$ from Lemma~\ref{lem:stochastic_term} is given and $\lambda=1$, we directly obtain
\begin{multline*}
|\langle\phi(x,a),\textrm{(II)}\rangle|=\left|\left\langle \phi(x, a), (\Lambda^k_h)^{-1} \sum_{\tau = 1}^{k-1}\sum_{p =1}^P \phi^{\tau,p}_h (V^{k}_{h+1}(x^{\tau,p}_{h+1}) - \Pe_h V^{k}_{h+1}(x_h^{\tau,p}, a_h^{\tau,p}))\right\rangle\right|
\\
\leq
\norm{\sum_{\tau = 1}^{k-1}\sum_{p =1}^P \phi^{\tau,p}_h (V^{k}_{h+1}(x^{\tau,p}_{h+1}) - \Pe_h V^{k}_{h+1}(x_h^{\tau,p}, a_h^{\tau,p}))}_{(\Lambda^k_h)^{-1}}\norm{\phi(x,a)}_{(\Lambda^k_h)^{-1}}\\
\leq %\frac{1}{\sqrt{\lambda}} 
C dH\sqrt{\log ((c_\beta+1)dKHP/\delta)} \sqrt{\phi(x, a)^\top (\Lambda^k_h)^{-1}  \phi(x, a)} 
\end{multline*}
where $C$ is an absolute constant independent of $c_\beta>0$.
%; where for the last inequality we used the assumption $\norm{\phi(x,a)}\leq 1$. 

For the term (III),
\begin{align*}
\langle\phi(x,a),\textrm{(III)}\rangle&=\left \langle \phi(x, a), (\Lambda^k_h)^{-1}\sum_{\tau = 1}^{k-1}\sum_{p =1}^P \phi^{\tau,p}_h \Pe_h (V^{k}_{h+1} - V^{\bar{\pi}}_{h+1})(x_h^{\tau,p}, a_h^{\tau,p}) \right \rangle\\
&= \bigg \langle \phi(x, a), (\Lambda^k_h)^{-1}\sum_{\tau = 1}^{k-1}\sum_{p =1}^P \phi^{\tau,p}_h (\phi^{\tau,p}_h)^\top \int_{\S} (V^{k}_{h+1} - V^{\bar{\pi}}_{h+1})(x') d \mu_h(x')\bigg\rangle\\
&\overset{\textrm{(a)}}{=} \underbrace{\bigg \langle \phi(x, a), \int_{\S} (V^{k}_{h+1} - V^{\bar{\pi}}_{h+1})(x') d \mu_h(x')\bigg \rangle}_{\textrm{(III.1)}}\\
&\quad 
\underbrace{-\lambda \bigg\langle \phi(x, a), (\Lambda^k_h)^{-1}\int_{\S} (V^{k}_{h+1} - V^{\bar{\pi}}_{h+1})(x') d \mu_h(x') \bigg\rangle}_{\textrm{(III.2)}}
\end{align*}
where (a) follows from the definition of $\Lambda_h^k$ in the POLSVI algorithm. We immediately see from our assumption on linear MDP that $\textrm{(III.1)}=\Pe_h (V^{k}_{h+1} - V^\pi_{h+1})(x, a)$ and 
\begin{multline*}
|\textrm{(III.2)}|
\leq \lambda\norm{\int_\S (V^{k}_{h+1} - V^{\bar{\pi}}_{h+1})(x')d \mu_h(x')}_{(\Lambda_h^k)^{-1}} \sqrt{\phi(x, a)^\top (\Lambda^k_h)^{-1}  \phi(x, a)}\\
\leq \sqrt{\lambda}\norm{\int_\S (V^{k}_{h+1} - V^{\bar{\pi}}_{h+1})(x')d \mu_h(x')} \sqrt{\phi(x, a)^\top (\Lambda^k_h)^{-1}  \phi(x, a)}\\
\overset{\textrm{(a)}}{\leq} \sqrt{\lambda}2H\int_\S \norm{\mu_h(x')} dx'  \sqrt{\phi(x, a)^\top (\Lambda^k_h)^{-1}  \phi(x, a)}
\overset{(b)}{\leq} 2 H \sqrt{d\lambda} \sqrt{\phi(x, a)^\top (\Lambda^k_h)^{-1}  \phi(x, a)}
\end{multline*}
where (a) follows from the value functions being bounded, and (b) from the definiton of the linear MDP.

Finally, putting it all together with $\lambda=1$, we conclude that, 
\begin{multline*}
|\langle\phi(x, a), w^k_h\rangle - Q_h^{\bar{\pi}}(x, a)  -  \Pe_h (V^{k}_{h+1} - V^{\bar{\pi}}_{h+1})(x, a)| \\
\leq \left(\norm{w_h^{\bar{\pi}}}+CdH \sqrt{\log ((c_\beta+1)dKHP/\delta)}+2H\sqrt{d}\right)\sqrt{\phi(x, a)^\top (\Lambda^k_h)^{-1}  \phi(x, a)}\\
\leq\left(4H\sqrt{d}+CdH \sqrt{\log ((c_\beta+1)dKHP/\delta)}\right)\sqrt{\phi(x, a)^\top (\Lambda^k_h)^{-1}  \phi(x, a)}
\end{multline*}
where the last inequality follows from Proposition~\ref{prop:lin-Q}.

Now, from equation~\eqref{eq:upp_low_bound} in Lemma~\ref{lem:stochastic_term}, we have $\sqrt{\log ((c_\beta+1)dKHP/\delta)}>1$ independently from $c_\beta>0$, and thus
\begin{multline*}
|\langle\phi(x, a), w^k_h\rangle - Q_h^{\bar{\pi}}(x, a)  -  \Pe_h (V^{k}_{h+1} - V^{\bar{\pi}}_{h+1})(x, a)|
\leq \bar{C} dH\sqrt{\log ((c_\beta+1)dKHP/\delta)} \sqrt{\phi(x, a)^\top (\Lambda^k_h)^{-1}  \phi(x, a)},
\end{multline*}
for an absolute constant $\bar{C}= C+4$ independent of $c_{\beta}$.

Finally, to prove this lemma, we only need to show that there exists a choice of the absolute constant $c_\beta>0$ so that
$
\bar{C}\sqrt{\log ((c_\beta+1)dKHP/\delta)}\leq c_\beta \sqrt{\iota}
$, which is equivalent to
\begin{equation} \label{eq:choice_beta_constant}
\bar{C}\sqrt{\iota + \log(c_\beta + 1)} \le c_\beta \sqrt{\iota}
\end{equation}
since  $\sqrt{\log\left(\frac{(1+c_\beta)dKHP}{\delta}\right)}=\sqrt{\log\left(\frac{dKHP}{\delta}\right)+\log(1+c_\beta)}=\sqrt{\iota+\log(1+c_\beta)}$.

Two facts are known: 1) $\iota \in [\log(2), \infty)$ by its definition and $d\geq 2$; and 2) $\bar{C}$ is an absolute constant independent of $c_\beta$. 

Since we know we are looking for  $c_\beta>0$ and using the bound $\log(x)\leq x-1$ for any positive $x\in\R$, we conclude that proving the following equation implies~\eqref{eq:choice_beta_constant},
\begin{equation} \label{eq:choice_beta_constant_2}
\bar{C}\sqrt{\iota+c_\beta}\leq c_\beta\sqrt{\iota}.
\end{equation}
Since both sides are nonnegative, we square them and obtain that it becomes equivalent to showing that $c_\beta$ satisfies $0\leq \iota c_\beta^2-\bar{C}^2c_\beta-\bar{C}^2\iota$, and solving this quadratic expression let us conclude that this is satisfied if $c_\beta\geq g(\iota)$ with $g(\iota)=\frac{\bar{C}^2}{2\iota}+\frac{1}{2}\sqrt{\frac{\bar{C}^4}{\iota^2}+4\bar{C}^2}$. We now observe that $\iota\mapsto g(\iota)$ is a non-increasing function for $\iota \in [\log(2), \infty)$; therefore, if we want~\eqref{eq:choice_beta_constant_2} (and so~\eqref{eq:choice_beta_constant} to hold for any $\iota \in [\log(2), \infty)$, it suffices to choose 
\begin{equation}
\label{eq:c_beta_lower}
c_\beta \geq \frac{\bar{C}^2}{2\log(2)}+\frac{1}{2}\sqrt{\frac{\bar{C}^4}{(\log(2))^2}+4\bar{C}^2}.    
\end{equation}
This finishes the proof.
\end{proof}
%%
%================================================================

We now continue with our proof of Theorem~\ref{thm:main-paralel}. Let us first condition on the event $\mathcal{E}$ defined in Lemma~\ref{lem:stochastic_term}.

We introduce the following notation: $\delta^{k,p}_h := V^k_h(x^{k,p}_{h}) - V^{\pi^{k,p}}_h(x^{k,p}_{h})$, and $\zeta^{k,p}_{h+1} := 
\Pe_h\delta^{k,p}_{h+1}(x^{k,p}_h, a^{k,p}_h) - \delta_{h+1}^{k,p}$. Then, for any $(p,h, k) \in [P]  \times [H] \times [K]$, we use Lemma~\ref{lem:basic_relation} (with $x=x_h^{k,p}$ and $a=a^{k,p}_h$ following the lemma's notation) to obtain, 
\begin{equation}
\label{eq:recursive-paral-aux}
\begin{aligned}
&Q^k_h(x_h^{k,p},a^{k,p}_h) - Q^{\pi^{k,p}}_h(x_h^{k,p}, a^{k,p}_h)\\
&\quad \leq \Pe_h (V^{k}_{h+1} - V^{\pi^{k,p}}_{h+1})(x_h^{k,p}, a^{k,p}_h)+ \beta \sqrt{(\phi^{k,p}_h)^\top (\Lambda^k_h)^{-1}  \phi^{k,p}_h}\\
\implies & V^k_h(x^{k,p}_{h}) - V^{\pi^{k,p}}_h(x^{k,p}_{h})\\
&\quad \leq (\Pe_h (V^{k}_{h+1} - V^{\pi^{k,p}}_{h+1})(x_h^{k,p},a_h^{k,p})) - \delta^{k,p}_{h+1})+ \delta^{k,p}_{h+1} 
+ \beta \sqrt{(\phi^{k,p}_h)^\top (\Lambda^k_h)^{-1}  \phi^{k,p}_h}\\
\implies & \delta^{k,p}_h\leq 
\zeta^{k,p}_{h+1} + \delta^{k,p}_{h+1}
+ \beta \sqrt{(\phi^{k,p}_h)^\top (\Lambda^k_h)^{-1}  \phi^{k,p}_h}.
\end{aligned}
\end{equation}
We define $\zeta_1^{k,p}=0$ for every $(k,p)\in[K]\times[P]$.

Now, let us focus on the regret performance metric. Now, 
% TODO: REMOVE ALL \bar{K} and just use 2. 
\begin{multline}
\label{eq:regret_prev_p}
\textnormal{Regret}(K,P) =  \sum_{k=1}^{K}\sum_{p=1}^{P} (V^\star_1(s_0) - V^{\pi^{k,p}}_1 (s_0))
\overset{\textrm{(a)}}{\leq}\sum_{k=1}^{K}\sum_{p=1}^{P} (V^{k}_1(s_0) - V^{\pi^{k,p}}_1 (s_0))\\
\overset{\textrm{(b)}}{=}\sum_{k=1}^{K}\sum_{p=1}^{P} \delta^{k,p}_1
\\
\overset{(c)}{\leq} \underbrace{\sum_{k=1}^{K}\sum_{p=1}^{P}\sum_{h=1}^H \zeta^{k,p}_{h}}_{\textrm{(I)}} + \underbrace{\beta \sum_{k=1}^{K}\sum_{p=1}^{P}\sum_{h=1}^H \sqrt{(\phi^{k,p}_h)^\top (\Lambda^k_h)^{-1}\phi^{k,p}_h}}_{\textrm{(II)}},
\end{multline}
where (a) follows from the optimism upper bound found in Lemma~B.5 from~\citep{CJ-ZY-ZW-MIJ:20} but using the event $\mathcal{E}$ from Lemma~\ref{lem:stochastic_term}; (b) follows since $x^{k,p}_1=s_o$ for every $(k,p)\in[K]\times[P]$; and (c) follows from the recursive formula in~\eqref{eq:recursive-paral-aux} and the fact that $\delta_{H+1}^{k,p}=\zeta_{H+1}^{k,p}=0$ and $\zeta^{k,p}_1=0$.

We first analyze the term (I) from~\eqref{eq:regret_prev_p}. 
Let us define the filtration $\{\F_{(k,h,p)}\}_{(k,h,p)\in\L^\star}$ where
$\L^\star$ is a sequence such that $\L^\star\subset\mathbb{Z}_{\geq 1}\times[H]\times[P]$ and its elements are arranged as follows. Firstly, we let the third coordinate take values from $1$ to $P$ and repeat this periodically \emph{ad infinitum}, so that each period has $P$ elements of $\L^\star$.
Secondly, the second coordinate takes the value $1$ for all elements in the first period of the third coordinate, then it takes the value $2$ for all elements of the second period of the third coordinate, and so on until taking the value of $H$ for all elements in the $H$-th period of the third coordinate --- this will constitute a period in the second coordinate --- after which the second coordinate takes the value $1$ again and continue describing periods (of $HP$ elements each) \emph{ad infinitum}. Finally, we let the third coordinate take the value corresponding to the number of periods so far in the second coordinate (thus the values of the first coordinate is unbounded). 
Consider any element $(k,h,p)\in\L^\star$. We denote by $(k,h,p)^{-1}$ its previous element in $\L^\star$. We let $\F_{(k,h,p)}$ contain the information of all states $x^{\bar{k},\bar{p}}_{\bar{h}}$ and actions $a^{\bar{k},\bar{p}}_{\bar{h}}$ whose indexes $(\bar{k},\bar{h},\bar{p})$ belong to the set $\L^\star$ up to the element $(k,h,p)\in\L^\star$.

We then can conclude that $\{\zeta^{k,p}_{h}\}_{(k,h,p)\in\L^\star}$ is a 
%Since the computation of $V_h^k$ is independent of the new observation $x^k_h$ at episode $k$, 
martingale difference sequence due to the following two properties:
\begin{enumerate}
    \item $\zeta^{k,p}_{h}\in\F_{(k,h,p)^{-1}}$. For $h=1$, $\E[\zeta^{k,p}_{h}|\F_{(k,h,p)^{-1}}]=0$ is trivial, so we focus on $h=2,\dots,H$.
    Note that since $x^{k,p}_h\sim\P_{h-1}(\cdot|x^{k,p}_{h-1},a^{k,p}_{h-1})$ (line 14 of POLSVI), we have that $\E[\delta^{k,p}_{h}|\F_{(k,h,p)^{-1}}]=\E[V^k_h(x^{k,p}_{h}) - V^{\pi^{k,p}}_h(x^{k,p}_{h})|\F_{(k,h,p)^{-1}}]=\E_{x'\sim\P_{h-1}(\cdot|x_{h-1}^{k,p},a_{h-1}^{k,p})}[V^k_h(x') - V^{\pi^{k,p}}_h(x')]=\Pe_{h-1}\delta_h^{k,p}(x_{h-1}^{k,p},a_{h-1}^{k,p})$, which immediately implies $\E[\zeta^{k,p}_{h}|\F_{(k,h,p)^{-1}}]=0$. 
    %    
    %and noticing that $\E[\delta^{k,p}_{h}|\F_{(k,h,p)^{-1}}]=\E[\delta^{k,p}_{h}|x^{p,q}_r, a^{p,q}_r, p\in[k], q\in[P],q\in[h-1]]=$ 
    %
    %\begin{align*}
    %\E[\zeta^{k,p}_{h}|\F_{(k,h,p)^{-1}}]&=\E_{x'\sim\P_{h-1}(\cdot|x_{h-1}^{k,p},a_{h-1}^{k,p})}[\delta^{k,p}_{h}]-\E[\delta^{k,p}_{h}|\F_{(k,h,p)^{-1}}]\\
    %&=\E_{x'\sim\P_{h-1}(\cdot|x_{h-1}^{k,p},a_{h-1}^{k,p})}[\delta^{k,p}_{h}]-\E[\delta^{k,p}_{h}|x_{h-1}^{k,p},a_{h-1}^{k,p}]=0
    %\end{align*}
    %(where we have the notation
    %$\E_{x'\sim\P_{h-1}(\cdot|x_{h-1}^{k,p},a_{h-1}^{k,p})}[g(x')]\equiv\E[g(x')|x_{h-1}^{k,p},a_{h-1}^{k,p}]$ for any function $g:\S\to\R$).
    \item $|\zeta^{k,p}_{h}|\leq 
    |\Pe_h\delta^{k,p}_{h+1}(x^{k,p}_h, a^{k,p}_h)| +|\delta_{h+1}^{k,p}|\leq 2H <\infty$
    since $V^k_h(x) - V^{\pi^{k,p}}_h(x)\in[-H,H]$ for any $x\in\S$.
\end{enumerate}
Therefore, we can use the Azuma-Hoeffding inequality to conclude that, for any $\epsilon > 0$,
\begin{equation*}
\Pr \left(\sum_{k=1}^{K}\sum_{p=1}^{P}\sum_{h=1}^H  \zeta^{k,p}_{h}> \epsilon \right) \leq \exp \bigg (\frac{-2 \epsilon^2 } {(KHP)(4H^2) } \bigg ).
\end{equation*}
We choose $\epsilon=\sqrt{2KH^3P\log\left(\frac{1}{\delta}\right)}$. Then, with probability at least $1 -\delta$,   
\begin{equation}
\label{eq:final2}
  \textrm{(I)}=\sum_{k=1}^{K}\sum^P_{p=1}\sum_{h=1}^H  \zeta^{k,p}_{h}\leq\sqrt{2KH^2HP\log\left(\frac{1}{\delta}\right)} \leq 2H\sqrt{KHP\iota}, 
\end{equation}
recalling that $\iota = \log\left(\frac{dKHP}{\delta}\right)$. We call $\bar{\mathcal{E}}$ the event such that~\eqref{eq:final2} holds.

Let $\mathcal{D}_{k,h}$ be the event that there is a doubling round at step $h\in[H]$ in episode $k\in[K]$, i,e., ``$\Lambda_h^{k+1}\succ 2\Lambda_h^k$".

We now analyze the term (II) from~\eqref{eq:regret_prev_p}. We split (II) according to the event of doubling rounds,
\begin{equation}
\label{eq:double_round_fact}
\textrm{(II)}= \beta \sum_{k=1}^{K}\sum_{h=1}^H\I[\mathcal{D}_{k,h}]\sum_{p=1}^{P} \sqrt{(\phi^{k,p}_h)^\top (\Lambda^k_h)^{-1}\phi^{k,p}_h} + \beta \sum_{k=1}^{K}\sum_{h=1}^H\I[\mathcal{D}_{k,h}^c]\sum_{p=1}^{P} \sqrt{(\phi^{k,p}_h)^\top (\Lambda^k_h)^{-1}\phi^{k,p}_h}.
\end{equation}
For the first term in~\eqref{eq:double_round_fact}, 
%\begin{equation}
%\label{eq:aux_last_0}
%\beta \sum_{k=1}^{K}\I[\mathcal{D}_t]\sum_{p=1}^{P}\sum_{h=1}^H \sqrt{(\phi^{k,p}_h)^\top (\Lambda^k_h)^{-1}\phi^{k,p}_h}
%\leq \frac{\beta}{\sqrt{\lambda}} \sum_{k=1}^{K}\I[\mathcal{D}_t]\sum_{p=1}^{P}\sum_{h=1}^H \norm{\phi^{k,p}_h}
%\leq \frac{\beta PH}{\sqrt{\lambda}}\sum_{k=1}^{K}\I[\mathcal{D}_t].
%\end{equation}
%
\begin{multline}
\label{eq:aux_last_0}
\beta \sum_{k=1}^{K}\sum_{h=1}^H\I[\mathcal{D}_{k,h}]\sum_{p=1}^{P} \sqrt{(\phi^{k,p}_h)^\top (\Lambda^k_h)^{-1}\phi^{k,p}_h}\\
\leq \frac{\beta}{\sqrt{\lambda}} \sum_{k=1}^{K}\sum_{h=1}^H\I[\mathcal{D}_{k,h}]\sum_{p=1}^{P} \norm{\phi^{k,p}_h}
\leq \frac{\beta P}{\sqrt{\lambda}}\sum_{k=1}^{K}\sum_{h=1}^H\I[\mathcal{D}_{k,h}].
\end{multline}

We now bound the second term in~\eqref{eq:double_round_fact}. Let us fix any $h\in[H]$. We can define the bounded sequence $\{\phi_h^{k,p}\}_{(k,p)\in\L_P}$ with $\|\phi_h^{k,p}\|\leq 1$ for any $(k,p)\in\mathbb{Z}_{\geq 0}\times[P]$. We also define
$$
\Lambda^{k,p}_h:=\lambda I_d + \sum^{k-1}_{\tau=1}\sum^P_{\bar{p}=1}\phi^{\tau,\bar{p}}(\phi^{\tau,\bar{p}}_h)^\top + \sum_{\bar{p}=1}^{p-1}
\phi^{k,\bar{p}}(\phi^{k,\bar{p}}_h)^\top
$$
and notice that $\Lambda^{k,1}_h=\Lambda^{k}_h$ and that $\Lambda^{k,p}_h\preceq \Lambda^{k+1,1}_h$ for any $p\in[P]$.

Now, consider we are in the case that there is no doubling round, i.e., $\I[\mathcal{D}^c_{k,h}]=1$. Then,  
$$
\Lambda^{k}_h\preceq\Lambda^{k,p}_h\preceq \Lambda^{k+1,1}_h=\Lambda^{k+1}_h\preceq 2\Lambda^{k}_h \implies
(\Lambda^{k}_h)^{-1}\succeq(\Lambda^{k,p}_h)^{-1}\succeq\frac{1}{2}(\Lambda^{k}_h)^{-1},
$$
which then implies that 
$$
\sum_{k=1}^{K}\sum_{p=1}^{P}(\phi^{k,p}_h)^\top (\Lambda^k_h)^{-1}\phi^{k,p}_h
\leq
2\sum_{k=1}^{K}\sum_{p=1}^{P}(\phi^{k,p}_h)^\top (\Lambda^{k,p}_h)^{-1}\phi^{k,p}_h
\overset{\textrm{(a)}}{\leq} 4\log\left[\frac{\det(\Lambda_h^{K,P})}{\det(\lambda I_d)}\right]
$$
%
%\beta \sum_{k=1}^{K}\I[\mathcal{D}_t^c]\sum_{p=1}^{P}\sum_{h=1}^H \sqrt{(\phi^{k,p}_h)^\top (\Lambda^k_h)^{-1}\phi^{k,p}_h}
%
where (a) follows from using~\citep[Lemma~11]{YAY-DP-CS:11}, whose conditions are satisfied from our bounded sequence $\{\phi_h^{k,p}\}_{(k,p)\in\L_P}$ (for fixed $h$) and the fact that the minimum eigenvalue of $\Lambda_h^k$ is lower bounded by $\lambda=1$ for every $k\in[K]$. Now, we have that $\Lambda_h^{K,P}=\Lambda_h^{K+1}$, which is a positive definite matrix whose maximum eigenvalue can be bounded as 
$\norm{\Lambda_h^{K+1}}\leq \norm{\sum_{k=1}^K\sum^P_{p=1}\phi^{k,p}_h(\phi^{k,p}_h)^\top}+\lambda\leq KP+\lambda$, and so $\det(\Lambda^{K+1}_h)\leq \det((KP+\lambda)I_d)=(KP+\lambda)^d$. We also have that $\det(\lambda I_d)=\lambda^d$. Then, using these results in our previous equation, we obtain that
\begin{equation}
\label{eq:aux_last_1}
\sum_{k=1}^{K}\sum_{p=1}^{P}(\phi^{k,p}_h)^\top (\Lambda^k_h)^{-1}\phi^{k,p}_h
\leq
4\log\left[\frac{KP+\lambda}{\lambda}\right]^d=
4d\log(KP+1)\leq 4d\iota,
\end{equation}
where the last inequality holds since $\log(KP+1)\leq \log\left(\frac{dKHP}{\delta}\right)=\iota$ for $d\geq 2$. 

Now, for a fixed $h\in[H]$, let $\mathcal{R}_h:=\{k\in[K]\mid \I[\mathcal{D}_{k,h}^c]=1\}$. 
Then, taking the second term in~\eqref{eq:double_round_fact},   
%\begin{multline}
%\label{eq:aux_last_2}
%\beta \sum_{k=1}^{K}\sum_{h=1}^H\I[\mathcal{D}_{k,h}^c]\sum_{p=1}^{P} \sqrt{(\phi^{k,p}_h)^\top (\Lambda^k_h)^{-1}\phi^{k,p}_h}  
%  %
%  \leq
%  \beta \sum_{k=1}^{K}\sum_{p=1}^{P}\sum_{h=1}^H \sqrt{(\phi^{k,p}_h)^\top (\Lambda^k_h)^{-1}\phi^{k,p}_h}\\
%  %
%  \overset{\textrm{(a)}}{\leq}\beta \sum_{h=1}^H  \sqrt{KP}\sqrt{ \sum_{k=1}^K\sum_{p=1}^P (\phi^{k,p}_h)^\top (\Lambda^k_h)^{-1}\phi^{k,p}_h}
%  %
%  \overset{(b)}{\leq} 2\beta H\sqrt{dKP\iota}, 
%\end{multline}
%
\begin{multline}
\label{eq:aux_last_2}
\beta \sum_{k=1}^{K}\sum_{h=1}^H\I[\mathcal{D}_{k,h}^c]\sum_{p=1}^{P} \sqrt{(\phi^{k,p}_h)^\top (\Lambda^k_h)^{-1}\phi^{k,p}_h}  
= \beta \sum_{h=1}^H\sum_{k\in\mathcal{R}_h}\sum_{p=1}^{P} \sqrt{(\phi^{k,p}_h)^\top (\Lambda^k_h)^{-1}\phi^{k,p}_h}\\
  \overset{\textrm{(a)}}{\leq}\beta \sum_{h=1}^H  \sqrt{|\mathcal{R}_h|P}\sqrt{ \sum_{k\in\mathcal{R}_h}\sum_{p=1}^P (\phi^{k,p}_h)^\top (\Lambda^k_h)^{-1}\phi^{k,p}_h}\\
\leq\beta \sum_{h=1}^H  \sqrt{|\mathcal{R}_h|P}\sqrt{ \sum_{k=1}^K\sum_{p=1}^P (\phi^{k,p}_h)^\top (\Lambda^k_h)^{-1}\phi^{k,p}_h}\\
\overset{(b)}{\leq} 2\beta\sqrt{d\iota} \sum_{h=1}^H  \sqrt{|\mathcal{R}_h|P}\leq 2\beta H\sqrt{dKP\iota}, 
%
%\overset{(b)}{\leq} 2\beta H\sqrt{dKP\iota},
\end{multline}
where (a) follows from the Cauchy-Schwartz inequality, and (b) from~\eqref{eq:aux_last_1}.
We now combine the results in~\eqref{eq:aux_last_0} and~\eqref{eq:aux_last_2}, i.e., the upper bounds for (II), and obtain, letting $\lambda=1$,
%$$
%\textrm{(II)}\leq \frac{\beta PH}{\sqrt{\lambda}}\sum_{k=1}^{K}\I[\mathcal{D}_t] +  \beta H\sqrt{2 d\bar{K}KP\iota}.
%$$
$$
\textrm{(II)}\leq\beta P\sum_{k=1}^{K}\sum_{h=1}^H\I[\mathcal{D}_{k,h}] + 2\beta H\sqrt{dKP\iota}.
$$
Using this last result along with~\eqref{eq:final2} in~\eqref{eq:regret_prev_p}, we conclude that,
%
%\begin{multline}
%\label{eq:regret_almost_l}
%\textnormal{Regret}(K,P) \leq  
%2H\sqrt{KHP\iota}
%+ 
%\beta H\sqrt{2 d\bar{K}KP\iota}
%+
%\frac{\beta PH}{\sqrt{\lambda}}\sum_{k=1}^{K}\I[\mathcal{D}_t]\\
%
%=2\sqrt{KH^3P\iota}
%+ 
%c_\beta\sqrt{2\bar{K}}\sqrt{ d^3T H^4P\iota^2}
%+
%c_\beta d PH^2\sum_{k=1}^{K}\I[\mathcal{D}_t]\\
%
%\leq \bar{C}\sqrt{d^3(KH)PH^3\iota^2}+
%c_\beta d PH^2\sum_{k=1}^{K}\I[\mathcal{D}_t]
%\end{multline}
%
\begin{multline}
\label{eq:regret_almost_l}
\textnormal{Regret}(K,P) \leq  
2H\sqrt{KHP\iota}
+ 
2\beta H\sqrt{dKP\iota}
+
\beta P\sum_{k=1}^{K}\sum_{h=1}^H\I[\mathcal{D}_{k,h}]\\
=2\sqrt{KH^3P\iota}
+ 
2c_\beta\sqrt{ d^3KH^4P\iota^2}
+
c_\beta d PH\sqrt{\iota}\sum_{k=1}^{K}\sum_{h=1}^H\I[\mathcal{D}_{k,h}]\\
\overset{\textrm{(a)}}{\leq} (2+2c_\beta)\sqrt{d^3KH^4P\iota^2}+
c_\beta d PH\sqrt{\iota}\sum_{k=1}^{K}\sum_{h=1}^H\I[\mathcal{D}_{k,h}]
\end{multline}
where (a) follows from $\sqrt{\iota}\leq \iota$ which follows from equation~\eqref{eq:iota_bound}.

We now bound the number of possible doubling rounds. Consider any $h\in[H]$. When there is a doubling round at episode $k$, we can find some $y\in\R^d$, $\norm{y}=1$, such that $y^\top\Lambda^{k+1}_h y>2y^\top\Lambda^{k}_h y$. Then, we can apply Lemma~12 from~\citep{JC-AP-NT-YSS-PB-MIJ:21} to obtain $\det(\Lambda^{k+1}_h)>2\det(\Lambda^{k}_h)$. If $n_h$ is the number of doubling rounds at step $h\in[H]$ across all episodes, then $\det(\Lambda^{K+1}_h)>(2)^{n_h}\det(\Lambda^{1}_h)$, and so $\log\left(\frac{\det(\Lambda^{K+1}_h)}{\det(\Lambda^{1}_h)}\right)>n_h\log(2)$. Now, since $\norm{\phi^{k,p}_h}\leq 1$ and $\lambda=1$, by Lemma~19.4 from~\citep{TL-CS:20}, we have that 
$$\log\left(\frac{\det(\Lambda^{K+1}_h)}{\det(\Lambda^{1}_h)}\right)\leq d\log\left(\frac{\tr(\lambda I_d)+KP}{d}\right)-\log(\det(\lambda I_d))=d\log\left(1+\frac{KP}{d}\right).$$
Then, our recently derived results let us conclude that $\sum^{K}_{k=1}\sum_{h=1}^H\I[\mathcal{D}_{k,h}]=\sum_{h=1}^H n_h<\frac{dH}{\log(2)}\log\left(1+\frac{KP}{d}\right)$; and using this bound in~\eqref{eq:regret_almost_l} let us obtain
\begin{equation*}
\textnormal{Regret}(K,P)
\leq (2+2c_\beta)\sqrt{d^3KH^4P\iota^2}+
\frac{c_\beta d^2H^2}{\log(2)}P\sqrt{\iota}\log\left(1+\frac{KP}{d}\right).
%< \bar{C}\sqrt{d^3(KH)PH^3\iota^2}+
%\frac{c_\beta d^2H^2}{\log(\bar{K})} P \log\left(1+\frac{KP}{d}\right),
%
\end{equation*}
Finally, since $\Pe[\text{not }  \mathcal{E}]\leq \delta$ and $\Pe[\text{not } \bar{\mathcal{E}}]\leq \delta$, a union bound lets us conclude that all the results proven so far hold with probability $1-2\delta
$. This finishes the proof of Theorem~\ref{thm:main-paralel}.
%
%================================================================
%================================================================  

\section{Proof of Theorem~\ref{thm:main-reward-free}}
\label{subsec:proof_main_rf_polsvi}

For simplicity, we will use the same notation as described at the beginning of section~\ref{subsec:proof_main_parallel} for the exploration phase with the following notation for the reward functions in the exploration phase: $r^k:=\{r_h^k\}_{h\in[H]}$ and $r^{k,p}_h:=r^k_h(x^{k,p}_h,a^{k,p}_h)$. For the planning phase, we denote the given reward $r:=\{r_h\}_{h\in[h]}$. For the planning phase, we define the value function $\hat{V}_h^k(\cdot)=\max_{a\in\A} \hat{Q}_h^k(\cdot,a)$. 
Additionally, if we take the underlying MDP $\mathcal{M}$ and replace its reward function by $\bar{r}=\{\bar{r}_h\}_{h\in[H]}$, then we denote the value functions of the newly created MDP by $V_h(\cdot,\bar{r})$, $h\in[H]$; and denote the optimal value functions by $V_h^*(\cdot,\bar{r})$, $h\in[H]$.

We introduce our first auxiliary lemma.

\begin{lemma}[Concentration bound on value functions in the exploration phase for RF-POLSVI]
\label{lem:concentration}
Consider the setting of Theorem~\ref{thm:main-reward-free}. There exists an absolute constant $C$ independent of $c_{\beta}$ such that for any fixed $\delta\in(0, 1)$, the following event $\mathcal{E}$ holds with probability at least $1-\delta$,
\begin{multline*}
\forall (k, h)\in [K]\times [H]: \quad \norm{\sum_{\tau = 1}^{k-1}\sum_{p=1}^P \phi^{\tau,p}_h [V^{k}_{h+1}(x^{\tau,p}_{h+1}) - \Pe_h V^{k}_{h+1}(x_h^{\tau,p}, a_h^{\tau,p})]}_{(\Lambda^k_h)^{-1}}\\
\leq CdH\sqrt{\log \left(\frac{(c_\beta+1)dKHP}{\delta}\right)}. 
\end{multline*}
\end{lemma}
\begin{proof}
First, notice that $\norm{w^k_h}\leq H\sqrt{dKP}$ by following the same proof as in Lemma~\ref{lem:wn_estimate} and choosing $\lambda=1$. Then, we notice that $(V^{k}_{h})_{(k,h)\in[K]\times[H]}$ 
belongs to the following function class 
\begin{multline}
\label{eq:function_class-MDP}
\mathcal{V} = \left\{V:\S\to\R\;\Big|\; V(\cdot)=
\min\left\{
\max_{a\in\A}\hat{w}^\top\phi(\cdot,a) \right.\right.\\
\left.\left. +\left(1+\frac{1}{H}\right)\min\{\beta\sqrt{\phi(\cdot,a)^\top\hat{\Lambda}^{-1}\phi(\cdot,a)},H\}, H
\right\}\right\},
\end{multline}
where $\norm{\hat{w}}\leq H\sqrt{dKP}$, the minimum eigenvalue of $\hat{\Lambda}$ is greater or equal than $\lambda$, and $\norm{\phi(x,a)}\leq 1$ for all $(x,a)\in\S\times\A$. For any $V,V'\in\mathcal{V}$, let $\dist(V,V')=\sup_{x\in\S} |V(x) - V'(x)|$. Set $\hat{u}(x,a):=\min\{\beta\sqrt{\phi(\cdot,a)^\top\hat{\Lambda}^{-1}\phi(\cdot,a)},H\}$. Then, 
\begin{equation}
\label{eq:dist-covering-MG-0}
\begin{aligned}
\dist(V,V')
&=\sup_{x\in\S}
\Big|\max_{a\in\A}\min\{\hat{w}^\top\phi(x,a)+(1+1/H)\hat{u}(x,a),H\}\\
&\quad-
\max_{a\in\A}\min\{(\hat{w}')^\top\phi(x,a)+(1+1/H)\hat{u}'(x,a),H\}\Big|\\
&\overset{\textrm{(a)}}{\leq}\sup_{x\in\S,a\in\A}
\Big|
(\hat{w}^\top\phi(x,a)+(1+1/H) \hat{u}(x,a))\\
&\quad-
((\hat{w}')^\top\phi(x,a)+(1+1/H)\hat{u'}(x,a))\Big|\\
%
%&\overset{\textrm{(a)}}{\leq}\sup_{x\in\S,a\in\A}
%\Big|
%(\hat{w}^\top\phi(x,a,b)+(1+1/H)\beta\sqrt{\phi(x,a,b)^\top\hat{\Lambda}^{-1}\phi(x,a,b)})\\
%&\quad-
%((\hat{w}')^\top\phi(x,a,b)+(1+1/H)\beta\sqrt{\phi(x,a,b)^\top(\hat{\Lambda}')^{-1}\phi(x,a,b)})\Big|\\
%
&\leq\sup_{\phi:\norm{\phi}\leq 1}\Big|
(\hat{w}-\hat{w}')^\top\phi\Big|
+\sup_{x\in\S,a\in\A}2\Big|
\hat{u}(x,a)-\hat{u'}(x,a)\Big|\\
&\overset{(b)}{\leq}\sup_{\phi:\norm{\phi}\leq 1}\Big|
(\hat{w}-\hat{w}')^\top\phi\Big|\\
&\quad +\sup_{\phi:\norm{\phi}\leq 1}2\beta\Big|
\sqrt{\phi^\top\hat{\Lambda}^{-1}\phi}-
\sqrt{\phi^\top(\hat{\Lambda}')^{-1}\phi}
\Big|\\
&\overset{\textrm{(c)}}{\leq} \norm{\hat{w}-\hat{w}'} + \sup_{\phi:\norm{\phi}\leq 1}
2\beta\sqrt{\left|\phi^\top(\hat{\Lambda}^{-1}-(\hat{\Lambda}')^{-1})\phi\right|}\\
&=\norm{\hat{w}-\hat{w}'} + 
2\beta\sqrt{\norm{\hat{\Lambda}^{-1}-(\hat{\Lambda}')^{-1}}}\\
&\leq\norm{\hat{w}-\hat{w}'} + 
2\beta\sqrt{\norm{\hat{\Lambda}^{-1}-(\hat{\Lambda}')^{-1}}_F}
\end{aligned}
\end{equation}
where (a) follows from the fact that the $\min\{\cdot,H\}$ operator is non-expansive and that the property $|\sup_{a\in\A}g(a)-\sup_{a\in\A}h(a)|\leq \sup_{a\in\A}|g(a)-h(a)|$ for any $g,h:\A\to\R$, and (b) follows from the inequality $|\sqrt{p}-\sqrt{q}|\leq\sqrt{|p-q|}$ for any $p,q\geq 0$. Now, we notice that~\eqref{eq:dist-covering-MG-0} is a bound of the same form as equation~(28) from Lemma~D.6 of~\citep{CJ-ZY-ZW-MIJ:20}, and so we can use this lemma to obtain that the $\epsilon$-covering number of $\mathcal{V}$, $\mathcal{N}_{\epsilon}$, with respect to the distance $\dist(\cdot,\cdot)$ can be upper bounded as 
\begin{equation*}
%\label{eq:cover-aux-0}
\log \mathcal{N}_{\epsilon} \le d  \log (1+ 4H\sqrt{dKP}/ \epsilon ) + d^2 \log \bigl [ 1 +  32 d^{1/2} \beta^2  / (\lambda\epsilon^2)  \bigr ].
\end{equation*}
Then, the proof of the lemma follows immediately from closely following the proof of 
Lemma~\ref{lem:stochastic_term}.
\end{proof}

We use the previous lemma to obtain a useful upper bound to the optimal value function resulting from an MDP whose rewards are the ones computed at the current episode of the exploration phase.

\begin{lemma}[Upper bound on the value function at the exploration phase]
\label{lem:sum_V}
Consider the setting of Theorem~\ref{thm:main-reward-free}.
With probability at least $1-2\delta$, for all $k\in[K]$, 
\begin{enumerate}
    \item\label{item:one_up} $V^*_1(s_0, r^k)\leq V^{k}_1(s_0)$,
\item\label{item:three_up} $\sum_{k=1}^{K}V^{k}_1(s_0)
\leq 2H\sqrt{\frac{KH\iota}{P}}+6\beta 
H\sqrt{\frac{dK\iota}{P}}
+ 10\beta d H\log\left(1+\frac{KP}{d}\right)
$.
\end{enumerate}
\end{lemma}
\begin{proof}
We first condition on the event $\mathcal{E}$ defined in Lemma~\ref{lem:concentration}, which holds with probability at least $1-\delta$.
Since we have a linear MDP, for any $(x,a)\in\S\times\A$, we have $\P_h(x'|x,a) = \langle \phi(x,a),\mu_h(x')\rangle$ and so $
\Pe_h V_{h+1}^k(x,a)=\langle \phi(x,a), \tilde{w}_h^k\rangle
$ with 
$
\tilde{w}^k_h:=\int_{x\in\S}V^k_{h+1}(x')d\mu_h(x')=\int_{x'\in\S}V^k_{h+1}(x')\mu_h(x')dx'
$. Then,
$$
\norm{\tilde{w}^k_h}\leq H\int_{x'\in\S}\norm{\mu_h(x')}dx'\leq H\sqrt{d}.
$$
Now, using this result, we have for every $(x,a,h,k)\in[\S]\times[\A]\times [H]\times[K]$, 
\begin{align*}
&\phi(x,a)^\top w^{k}_h -
\Pe_h V_{h+1}^k(x,a)\\
&=
\phi(x,a)^\top (\Lambda^{k}_h)^{-1}
\sum_{\tau=1}^{k-1}\sum_{p=1}^P\phi_h^{\tau,p}
V_{h+1}^k(x_{h+1}^{\tau,p}) - \Pe_h V_{h+1}^k(x,a)\\
&=
\phi(x,a)^\top (\Lambda^{k}_h)^{-1}\left(
\sum_{\tau=1}^{k-1}\sum_{p=1}^P\phi_h^{\tau,p}
V_{h+1}^k(x_{h+1}^{\tau,p}) 
- \Lambda^{k}_h \tilde{w}^{k}_{h}\right)\\
&=
\phi(x,a)^\top (\Lambda^{k}_h)^{-1}\left(
\sum_{\tau=1}^{k-1}\sum_{p=1}^P\phi_h^{\tau,p}
V_{h+1}^k(x_{h+1}^{\tau,p}) 
-  \lambda\tilde{w}^{k}_{h} 
- 
\sum_{\tau=1}^{k-1}\sum_{p=1}^P\phi_h^{\tau,p}(\phi_h^{\tau,p})^\top
\tilde{w}^{k}_{h} 
\right)\\
&=
\underbrace{\phi(x,a)^\top (\Lambda^{k}_h)^{-1}\left(
\sum_{\tau=1}^{k-1}\sum_{p=1}^P\phi_h^{\tau,p}
\left(V_{h+1}^k(x_{h+1}^{\tau,p}) - \Pe_h V^k_{h+1}(x^{\tau,p}_h,a^{\tau,p}_h)\right)\right)}_{\textrm{(I)}} -\lambda\underbrace{ \phi(x,a)^\top (\Lambda^{k}_h)^{-1}
\tilde{w}^{k}_{h}}_{\textrm{(II)}}
.
\end{align*}
Analyzing the term (I),
\begin{align*}
|\textrm{(I)}|&\le
\norm{\phi(x,a)}_{(\Lambda^{k}_h)^{-1}}
\norm{\sum_{\tau=1}^{k-1}\sum_{p=1}^P\phi_h^{\tau,p}
\left(V_{h+1}^k(x_{h+1}^{\tau,p}) - \Pe_h V_{h+1}^{k}(x^{\tau,p}_h,a^{\tau,p}_h)\right)}_{(\Lambda^{k}_h)^{-1}}\\
&\overset{\textrm{(a)}}{\leq} CdH\sqrt{\log \left(\frac{(c_\beta+1)dKHP}{\delta}\right)}\norm{\phi(x,a)}_{(\Lambda^{k}_h)^{-1}},
\end{align*}
where (a) follows from Lemma~\ref{lem:concentration}. For the term (II), $
|\textrm{(II)}|
\leq \norm{\tilde{w}^{k}_{h}}_{(\Lambda^{k}_h)^{-1}} \norm{\phi(s,a)}_{(\Lambda^{k}_h)^{-1}}
\leq \frac{H\sqrt{d}}{\sqrt{\lambda}} \|\phi(s,a)\|_{(\Lambda^{k}_h)^{-1}}$. Then, using these results with $\lambda=1$,
\begin{equation}
\label{eq:optimism-bound-val-RF}
\begin{aligned}
|\phi(x,a)^\top w^{k}_h -
\Pe_h V_{h+1}^k(x,a)|&\leq (CdH\sqrt{\log((c_\beta+1)dKHP/\delta)}+H\sqrt{d}) 
\norm{\phi(x,a)}_{(\Lambda^{k}_h)^{-1}}
\\
&\overset{\textrm{(a)}}{\leq} \bar{C}dH\sqrt{\log((c_\beta+1)dKHP/\delta)}\norm{\phi(x,a)}_{(\Lambda^{k}_h)^{-1}} 
\\
&\overset{(b)}{\leq}
c_{\beta}dH\sqrt{\log(dKHP/\delta)}
\norm{\phi(x,a)}_{(\Lambda^{k}_h)^{-1}}\\
&=\beta\norm{\phi(x,a)}_{(\Lambda^{k}_h)^{-1}},
\end{aligned}
\end{equation}
where (a) follows from the fact that 
$\sqrt{\log((c_\beta+1)dKHP/\delta)}>1$ independently from $c_\beta$ from equation~\eqref{eq:upp_low_bound} from the proof of Lemma~\ref{lem:stochastic_term} and using the absolute constant $\bar{C}=C+1$; and where (b) follows from $c_\beta$ being chosen as in the expression of equation~\eqref{eq:c_beta_lower} from the proof of Lemma~\ref{lem:basic_relation}.

Now we prove statement~\ref{item:one_up} of the lemma, which we do by induction. Consider any $x\in \S$. For step $H+1$, it trivially holds that $
V_{H+1}^{*}(x, r^k)\le V_{H+1}^k(x)$ since $V_{H+1}^{*}(x, r^k)=V_{H+1}^k(x)=0$.
Now, assume that at step $h\in [H]$ it holds that $
V_{h+1}^{*}(x,r^k)\le V_{h+1}^k(x)$. Then, using the Bellman equation,
\begin{multline*}
V_{h}^{*}(x, r^k)
= \max_{a\in\A}
\{r^k_h(x,a) + \Pe_h V_{h+1}^{*}(\cdot, r^k)(x,a)\}
\overset{\textrm{(a)}}{\leq} 
\max_{a\in\A}
\{r^k_h(x,a) + \Pe_h V_{h+1}^{k}(x,a)\}\\
\overset{(b)}{\leq} 
\max_{a\in\A}
\{r^k_h(x,a) + \phi(x,a)^\top w^{k}_h 
+ \beta\|\phi(x,a)\|_{(\Lambda^{k}_h)^{-1}}\}\\
\overset{\textrm{(c)}}{\leq}\min\left\{\max_{a\in\A}
\{r^k_h(x,a) + \phi(x,a)^\top w^{k}_h 
+ \beta\|\phi(x,a)\|_{(\Lambda^{k}_h)^{-1}}\},H\right\}
= V^k_h(x)
,
\end{multline*}
where (a) follows from the induction assumption, (b) follows from~\eqref{eq:optimism-bound-val-RF}, and (c) from $V^*_h(s,r^k)\leq H$. This finishes the proof by induction.

Now we start our proof for statement
% s~\ref{item:two_up} and
\ref{item:three_up} of the lemma. For every $(k,h,p)\in [K]\times[H-1]\times[P]$, let
$
\zeta_h^{k,p} = \Pe_h V_{h+1}^k(x_h^{k,p},a_h^{k,p}) - V_{h+1}^k(x_{h+1}^{k,p})
$ and $\zeta_H^{k,p} = 0$ for any $(k,p)\in[K]\times[P]$. 
Now, fix any $p\in[P]$ and condition on event $\mathcal{E}$ to obtain 
\begin{equation}
\label{eq:V-bound-RF}
\begin{aligned}
\sum_{k=1}^KV_1^k(s_0)
&\overset{\textrm{(a)}}{\leq} \sum_{k=1}^K\left((\phi^{k,p}_1)^\top w^{k}_1 
+ r^k_1
+ \beta\|\phi(x_1^{k,p},a_1^{k,p}) \|_{(\Lambda^{k}_1)^{-1}}\right)\\
&\leq 
\sum_{k=1}^K
\left((\phi^{k,p}_1)^\top w^{k}_1 
+\beta(1+1/H)\|\phi(x_1^{k,p},a_1^{k,p}) \|_{(\Lambda^{k}_1)^{-1}}\right)\\
&\overset{(b)}{\leq} \sum_{k=1}^K
\left( \Pe_1 V_{2}^k(x_1^{k,p},a_1^{k,p})
+ \beta(2+1/H)\|\phi(x_1^{k,p},a_1^{k,p}) \|_{(\Lambda^{k}_1)^{-1}}\right)\\
&=
\sum_{k=1}^K\left(\zeta_1^{k,p}+ V_{2}^k(x_2^{k,p})
+ \beta(2+1/H)\|\phi(x_1^{k,p},a_1^{k,p}) \|_{(\Lambda^{k}_1)^{-1}}\right)\\
&\leq \ldots\\
&\leq 
\sum_{k=1}^K\sum_{h=1}^{H-1}\zeta_h^{k,p} +\beta(2+1/H) 
\sum_{k=1}^K\sum_{h=1}^H\norm{\phi(x_h^{k,p},a_h^{k,p})}_{(\Lambda^{k}_h)^{-1}}\\
&=\sum_{k=1}^K\sum_{h=1}^{H}\zeta_h^{k,p} +\beta(2+1/H) 
\sum_{k=1}^K\sum_{h=1}^H\norm{\phi(x_h^{k,p},a_h^{k,p})}_{(\Lambda^{k}_h)^{-1}}.
\end{aligned}
\end{equation}
where (a) follows from the fact that all agents are equally initialized (line 3 of algorithm~\ref{alg:main_LIN_RF_POLSVI_exp}) which then implies $Q_1^k(x^{k,p_1}_1,a_1^{k,p_1})=Q_1^k(x^{k,p_2}_1,a_1^{k,p_2})$ for every $p_1,p_2\in[P]$ since $x^{k,p_1}_1 = x^{k,p_2}_1$; 
and (b) follows from~\eqref{eq:optimism-bound-val-RF}.

%Note that for each $h\in[H - 1]$,
Now, using the sequence $\L^*$ defined in the proof of Theorem~\ref{thm:main-paralel} and following that proof itself, we can deduce that $\{\zeta^{k,p}_h\}_{(k,h,p)\in\L^*}$ is a martingale difference sequence and use the Azuma-Hoeffding inequality to conclude that, for any $\epsilon > 0$,
\begin{equation*}
\Pr \left(\sum_{k=1}^{K}\sum_{p=1}^{P}\sum_{h=1}^{H}  \zeta^{k,p}_{h}> \epsilon \right) \leq \exp \bigg (\frac{-2 \epsilon^2 } {(KHP)(4H^2) } \bigg ).
\end{equation*}
We choose $\epsilon=\sqrt{2KH^3P\log\left(\frac{1}{\delta}\right)}$. Then, with probability at least $1 -\delta$,   
\begin{equation}
\label{eq:mds-bound-RF}
  \sum_{k=1}^{K}\sum_{p=1}^P\sum_{h=1}^{H}  \zeta^{k,p}_{h}\leq\sqrt{2KH^3P\log\left(\frac{1}{\delta}\right)} \leq 2H\sqrt{KHP\iota}, 
\end{equation}
recalling that $\iota = \log\left(\frac{dKHP}{\delta}\right)$. We call $\bar{\mathcal{E}}$ the event such that~\eqref{eq:mds-bound-RF} holds. 

Now, let $\mathcal{D}_{k,h}$ be the event that there is a doubling round at step $h\in[H]$ of episode $k\in[K]$, i.e., ``$\Lambda_h^{k+1}\succ 2\Lambda_h^k$". Following a similar analysis to Theorem~\ref{thm:main-paralel}, we can conclude that \begin{equation}
    \label{eq:bound_dr}
\begin{aligned}
\sum_{k=1}^K\sum^P_{p=1}\sum_{h=1}^H \|\phi(x_h^{k,p},a_h^{k,p}) \|_{(\Lambda^{k}_h)^{-1}}
\leq
2H\sqrt{dKP\iota}
+ \frac{dH}{\log(2)}P\log\left(1+\frac{KP}{d}\right).
\end{aligned}
\end{equation}
Now, we go back to~\eqref{eq:V-bound-RF}, sum over the agents $p\in[P]$ both sides of the inequality and divide then by $P$ both sides respectively, and on this result then use both~\eqref{eq:bound_dr} and~\eqref{eq:mds-bound-RF} to obtain an upper bound, which results in  
\begin{equation*}
\begin{aligned}
\sum_{k=1}^{K}V^{k}_1(s_0)
\leq 2H\sqrt{\frac{KH\iota}{P}}+6\beta 
H\sqrt{\frac{dK\iota}{P}}
+ \frac{3\beta dH}{\log(2)}\log\left(1+\frac{KP}{d}\right).
\end{aligned}
\end{equation*}
This finishes the proof of statement~\ref{item:three_up} as we note $3/\log(2) < 10$ and thereby completes our proof.

% Finally, statement~\ref{item:one_up} of the lemma holds with probability at least $1-\delta$, and statements~\ref{item:two_up} and~\ref{item:three_up} hold with probability $\Pe[\mathcal{E}\cap\bar{\mathcal{E}}]\geq 1-2\delta$. This finishes the proof.
\end{proof}

Assume we take the underlying MDP $\mathcal{M}$ and decide to use $\{u_h/H\}_{h\in[H]}$ as the underlying reward, then, the next lemma provides an upper bound of its associated optimal value function. 

\begin{lemma}[Bounding the optimal value function with rewards taken from the planning phase]
\label{lem:small_V}
With probability $1-2\delta$,
$$
V^{*}_1\left(s_0, \left\{\frac{u_h}{H}\right\}_{h\in[H]}
\right)
\leq (2+6c_\beta)\sqrt{\frac{d^3H^4\iota^2}{KP}}
+ 10c_\beta\frac{\sqrt{d^4H^4\iota}}{K}\log\left(1+\frac{KP}{d}\right)
% +2\sqrt{\frac{H^2\iota}{K}}
$$
where $\{u_h\}_{h\in[H]}$ is as described in the planning phase of RF-POLSVI.
\end{lemma}
\begin{proof}

Observe we have that for every $k\in[K]$, $\Lambda_h\succeq \Lambda_h^k \implies \Lambda_h^{-1}\preceq (\Lambda_h^k)^{-1}$ (remember that $\Lambda_h$ is defined in the planning phase of RF-POLSVI). Then, for every $(k,h)\in [K]\times[H]$, 
$r^k_h(\cdot, \cdot) \geq u_h(\cdot, \cdot)/H$, and so for any $x\in\S$,
\begin{equation}
\label{eq:value-rew-RF}
V_1^*\left(x, \left\{\frac{u_h}{H}\right\}_{h\in[H]}\right) \leq 
V_1^*(x, \{r_h^k\}_{h\in[H]}).
\end{equation}
%
% We then have
% \begin{align*}
% \E\big[
% V_1^*(s_0, \{u_h/H\}_{h\in[H]})
% \big]
% % &\overset{\textrm{(a)}}{\leq} \frac{1}{K}
% % \E_{x\sim \mu}\left[
% % \sum_{k=1}^KV_1^*(x, r^k)
% % \right]
% % \overset{(b)}{\leq} \frac{1}{K}\sum_{k=1}^KV_1^*(s_0, r^k)
% % +\sqrt{2(1+\log(2))}\sqrt{\frac{H^2\iota}{K}}\\
% % &\overset{\textrm{(c)}}{\leq} 
% & \leq 2H\sqrt{\frac{H\iota}{KP}}+3\beta H\sqrt{\frac{2d\iota}{K}} + 2\sqrt{\frac{H^2\iota}{K}}\\
% &\leq 2\sqrt{\frac{H^3\iota}{KP}}+(3\sqrt{2}c_\beta + 2)
% \sqrt{\frac{d^3H^4\iota^2}{K}}
% \end{align*}
% %
% where (a) follows from~\eqref{eq:value-rew-RF}, (b) from~\eqref{eq:value-MDS-bound}, and (c) from statement~\ref{item:two_up} of Lemma~\ref{lem:sum_V}. 

% Now to prove statement~\ref{item:two_up_bound} of the lemma,
% %
Then,
\begin{align*}
V_1^*(s_0, \{u_h/H\}_{h\in[H]})
% &\overset{\textrm{(a)}}{\leq} \frac{1}{K}
% \E_{x\sim \mu}\left[
% \sum_{k=1}^KV_1^*(x, r^k)
% \right]
&\overset{\textrm{(a)}}{\leq} \frac{1}{K}\sum_{k=1}^KV_1^*(s_0, r^k)\\
% +\sqrt{2(1+\log(2))}\sqrt{\frac{H^2\iota}{K}}\\
&\overset{(b)}{\leq} 2H\sqrt{\frac{H\iota}{KP}}+6\beta 
H\sqrt{\frac{d\iota}{KP}}
+ 10\beta dH\frac{1}{K}\log\left(1+\frac{KP}{d}\right)\\
% +\sqrt{2(1+\log(2))}\sqrt{\frac{H^2\iota}{K}}\\
&\leq 
(2+6c_\beta)\sqrt{\frac{d^3H^4\iota^2}{KP}}
+ 10c_\beta\frac{\sqrt{d^4H^4\iota}}{K}\log\left(1+\frac{KP}{d}\right)
% +\sqrt{2(1+\log(2))}\sqrt{\frac{H^2\iota}{K}}
\end{align*}
where (a) follows from~\eqref{eq:value-rew-RF} and (b) follows from Lemma~\ref{lem:sum_V}. This finishes the proof.
\end{proof}

The following lemma shows that optimism can be used to bound the optimal Q-function. 
\begin{lemma}[Action-value bounds using optimism]\label{lem:confidence_planning}
With probability $1-\delta$, 
for the reward $r$ given in the planning phase of RF-POLSVI and for any $h \in [H]$,
\begin{equation}
\label{eq:Q-funct_opt_bound}
Q^*_h(\cdot,\cdot; r)\le \hat{Q}_h(\cdot, \cdot) 
\le
r_h(\cdot, \cdot) 
+ 
\Pe_h\hat{V}_{h+1}(\cdot,\cdot)
%
%\sum_{x'\in\S}\P_{h}(x' \mid \cdot, \cdot)\hat{V}_{h+1}(x')
%
+
2u_h(\cdot, \cdot)
.
\end{equation}
\end{lemma}
\begin{proof}
We first prove the right inequality in~\eqref{eq:Q-funct_opt_bound}. 
Carefully following the same derivation of equation~\eqref{eq:optimism-bound-val-RF} of Lemma~\ref{lem:sum_V} (using $\hat{V}_h$ instead of $V^k_h$ for every $h\in[H]$), we obtain that, with probability at least $1 - \delta$, 
for every $(x,a,h)\in\S\times\A\times [H]$, 
\begin{equation}
\label{eq:optimism-bound-val-RF_1}     
\left|
\phi(x,a)^\top \hat{w}_h - 
\Pe_h\hat{V}_{h+1}(x,a)
%\sum_{x' \in \S} \P_h(x' \mid x,a) \hat{V}_{h+1}(x')
\right| \leq \beta \|\phi(x,a)\|_{\Lambda_h^{-1}},
\end{equation}
and so 
\begin{align*}
\hat{Q}_h(x, a) 
&\leq  \phi(x, a)^\top \hat{w}_h + r_{h}(x, a) +  u_h(x, a)\\
&\leq  r_h(x, a) + 
\Pe_h\hat{V}_{h+1}(x,a)
%
%\sum_{x' \in \S} \P_h(x' \mid x,a) V_{h+1}(x') 
%
+ u_h(x,a) + \beta\|\phi(x,a)\|_{\Lambda_h^{-1}}\\
\overset{\textrm{(a)}}{\implies}\hat{Q}_h(x, a) &\leq\min\{r_h(x, a) 
+ \Pe_h\hat{V}_{h+1}(x,a)
+
2u_h(x, a),H\}\\
&\leq r_h(x, a) 
+ \Pe_h\hat{V}_{h+1}(x,a)
+
2u_h(x, a)
\end{align*}
where (a) follows from the fact $\hat{Q}_h(s, a) \le H$ and that $u_h(\cdot, \cdot)=\min\left\{\beta\norm{\phi(x,a)}_{\Lambda_h^{-1}}, H\right\}$.

We now prove the left inequality in~\eqref{eq:Q-funct_opt_bound}, which we do by induction on $h$. Consider any $(x,a)\in\S\times\A$. For the time step $h = H + 1$ the inequality trivially holds since the value functions are equal to zero. Suppose now that for some $h \in [H]$, $Q^*_{h + 1}(x,a; r)\leq \hat{Q}_{h + 1}(x, a)$. 
Then,
\begin{align*}
\phi(x,a)^\top \hat{w}_h 
\overset{\textrm{(a)}}{\geq}& 
\Pe_h 
%\sum_{x' \in \S}\P_h(x' \mid x,a) 
\hat{V}_{h+1}
%(x')
(x,a)
- \beta\|\phi(x,a)\|_{\Lambda_h^{-1}}\\
\overset{(b)}{\geq}& \Pe_h V_{h+1}^*(x,a; r)  - \beta  \|\phi(x,a)\|_{\Lambda_h^{-1}}.
\end{align*}
where (a) follows from~\eqref{eq:optimism-bound-val-RF_1}, and (b) from the induction hypothesis. %
Finally, using the Bellman equation along with the previous expression,
 \begin{align*}
Q_h^*(x, a; r)& =  r_h(x, a) + \Pe_h V_{h+1}^*(x,a; r) \\
& \leq  r_{h}(x, a) + \phi(x, a)^\top \hat{w}_h +  \beta\|\phi(x,a)\|_{\Lambda_h^{-1}}\\
\overset{(a)}{\implies} Q_h^*(x, a; r)&\leq \min\{r_{h}(x, a) + \phi(x, a)^\top \hat{w}_h +  \beta\|\phi(x,a)\|_{\Lambda_h^{-1}},H\}\\
&=\hat{Q}_h(x,a),
\end{align*}
where (a) follows since $Q^*_{h + 1}(x,a; r) \le H$. This finishes the proof by induction.
\end{proof}

We now continue with the proof of Theorem~\ref{thm:main-reward-free}, conditioning on the events defined in Lemma~\ref{lem:sum_V} and Lemma~\ref{lem:confidence_planning} which hold altogether with probability at least $1 - 3\delta$.

Then, 
\begin{align*}
&V^{*}_1(s_0; r)
- 
V^{\pi}_1(s_0; r)\\
&\overset{\textrm{(a)}}{\leq}\hat{V}_1(s_0)- V^{\pi}_1(s_0; r)\\
&= \hat{Q}_1(s_0,\pi_1(s_0))- Q^{\pi}_1(s_0, \pi_1(s_0); r)\\
&\leq \E_{x_2 \sim \P_1(\cdot \mid s_0, \pi_1(s_0))}[r_1(s_0, \pi_1(s_0)) + \hat{V}_2(x_2) + u_1(s_0, \pi(s_0)) - r_1(s_0, \pi_1(s_0)) - V^{\pi}_2(x_2;r)]\\
&= \E_{x_2 \sim \P_1(\cdot \mid s_0, \pi_1(s_0))}[\hat{V}_2(x_2) + u_1(s_0, \pi(s_0)) - V^{\pi}_2(x_2;r)]\\
&\leq \E_{x_2 \sim \P_1(\cdot \mid s_0, \pi_1(s_0)), x_3 \sim \P_2(\cdot, \mid x_2, \pi_2(x_2))}[u_1(s_0, \pi(s_0)) + u_2(x_2, \pi(x_2)) +  \hat{V}_3(x_3)- V^{\pi}_3(x_3;r)]\\
&= \ldots\\
&\leq\E\left[\sum_{h=1}^H u_h(x_h,\pi(x_h))\mid x_1=s_0\right]\\
&= V^{\pi}_1(s_0, \{u_h\}_{h\in{H}})\\
&\overset{(b)}{\leq}V^{*}_1(s_0,\{u_h\}_{h\in{H}})\\
&=HV^{*}_1(s_0, \{u_h/H\}_{h\in{H}}),
\end{align*}
where (a) follows from $\hat{V}_1(x) = \max_{a \in \A} \hat{Q}_1(x, a) \geq \max_{a \in \A} Q_1^*(x, a; r) = V_1^*(x; r)$ for any $x\in\S$ by Lemma~\ref{lem:confidence_planning}, and (b) follows from the definition of $V^{*}_1(s_0,\{u_h\}_{h\in{H}})$.

% Then, for the first bound of the first bound of the theorem, we use statement~\ref{item:one_up_bound} from Lemma~\ref{lem:small_V} to obtain
% \begin{align*}
% \E_{x_1\sim \mu}[V^{*}_1(x_1; r)
% - 
% V^{\pi}_1(x_1; r)]\leq  2\sqrt{\frac{H^5\iota}{KP}}+(3\sqrt{2}c_\beta + \sqrt{2(1+\log(2))})
% \sqrt{\frac{d^3H^6\iota^2}{K}};
% \end{align*}
% and for the second bound of the theorem, we use statement~\ref{item:two_up_bound} from Lemma~\ref{lem:small_V} to obtain
Finally, from Lemma~\ref{lem:small_V} we obtain 
\begin{align*}
V^{*}_1(x_1; r)
- 
V^{\pi}_1(x_1; r)\leq  
(2+6c_\beta)\sqrt{\frac{d^3H^6\iota^2}{KP}}
+ 10c_\beta\frac{\sqrt{d^4H^6\iota}}{K}\log\left(1+\frac{KP}{d}\right)
% +\sqrt{2(1+\log(2))}\sqrt{\frac{H^4\iota}{K}}
,
\end{align*}
which finishes the proof.
%
%By taking $K = c_K \cdot d^3H^6\log(dH  \delta^{-1}\varepsilon^{-1}) / \varepsilon^2$ for a sufficiently large constant $c_K > 0$, we have
%\[
%\EE_{x_1\sim \mu}[V^{*}_1(x_1; r)
%- 
%V^{\pi}_1(x_1; r)] \le H \cdot\EE_{s\sim \mu}[V^{*}_1(s, u / H)]\le c'H\sqrt{d^3H^4\cdot \log(dKH/\delta)/K} \le \varepsilon,
%\]
%which implies $\pi$ is $\varepsilon$-optimal with respect to $r$.

%==========================================
%==========================================

%
%================================================================
%================================================================  

\section{Proof of Theorem~\ref{thm:main-reward-free_MG}}
\label{subsec:proof_main_rf_polsvi_MG}

For simplicity, we will use the same notation as described at the beginning of section~\ref{subsec:proof_main_parallel} for the exploration phase, with the modification of including an extra argument due to the action of the second player, e.g., $r^{k,p}_h:=r^k_h(x^{k,p}_h,a^{k,p}_h,b^{k,p}_h)$. For the planning phase, we define the value functions $\overline{V}_h(x)=\E_{a\sim\pi_{h+1}(x),b\sim\overline{D}(x)}[\overline{Q}_{h}(x,a,b)]$ and $\underline{V}_h(x)=\E_{a\sim\underline{D}(x),b\sim\nu_{h+1}(x)}[ \underline{Q}_{h}(x,a,b)]$, where $\pi,\overline{D},\nu$ and $\underline{D}$ are described in lines 11-12 of algorithm~\ref{alg:main_LIN_RFMG_POLSVI_plan}. 
%     
%$\hat{V}_h^k(\cdot)=\max_{a\in\A} \hat{Q}_h^k(\cdot,a)$. 
%Additionally, if we take the underlying MDP $\mathcal{M}$ and replace its reward functions by $\bar{r}=\{\bar{r}_h\}_{h\in[H]}$; then, we can denote the value functions of the newly created MDP by $V_h(\cdot,\bar{r})$, $h\in[H]$; and denote the optimal value functions by $V_h^*(\cdot,\bar{r})$, $h\in[H]$

Let $\bar{a}=(a,b)$, $a\in\A$, $b\in\cB$, so that $\bar{a}\in\bar{A}:=\A\times\cB$. Then, we can reduce the planning phase of the RFMG-POLSVI algorithm as if it were the one from RF-POLSVI over an underlying MDP $\mathcal{M}'=(\S,\bar{A},H,\P; r)$ with the kernel being considered as $\P_h(\cdot\mid x,\bar{a})$ 
and the reward function as $r_h(x,\bar{a})$ with $\bar{a}\in\A$ and $x\in\S$ for any step $h\in[H]$ --- in other words, the  \emph{action} that a single agent takes in the MDP $\mathcal{M}'$ is equivalent to the \emph{joint action} of the two players of the Markov Game. 

Therefore, using this reduction, we use Lemma~\ref{lem:concentration} to conclude that with probability $1-\delta$, $\delta\in(0,1)$, the following event $\mathcal{E}$ holds  
\begin{multline}
\label{eq:aux-RFMG-POLSVI-one}
\forall (k, h)\in [K]\times [H]: \quad \norm{\sum_{\tau = 1}^{k-1}\sum_{p=1}^P \phi^{\tau,p}_h [V^{k}_{h+1}(x^{\tau,p}_{h+1}) - \Pe_h V^{k}_{h+1}(x_h^{\tau,p}, a_h^{\tau,p},b_h^{\tau,p})]}_{(\Lambda^k_h)^{-1}}\\
\leq CdH\sqrt{\log \left(\frac{(c_\beta+1)dKHP}{\delta}\right)}, 
\end{multline} 
where $C$ is an absolute constant independent of $c_{\beta}$. Likewise, we can use~\eqref{eq:aux-RFMG-POLSVI-one} and Lemma~\ref{lem:sum_V} to conclude that, with probability at least $1-2\delta$ for all $k\in[K]$, 
\begin{equation}
\label{item:one_up_MG} 
V^*_1(s_0, r^k)\leq V^{k}_1(s_0),
\end{equation}
and that 
\begin{align}
%\label{item:two_up_MG} &\sum_{k=1}^{K}V^{k}_1(s_0)
%\leq 2H\sqrt{\frac{KH\iota}{P}}+3\beta H\sqrt{2dK\iota}
\label{item:three_up_MG} &\sum_{k=1}^{K}V^{k}_1(s_0)
\leq 2H\sqrt{\frac{KH\iota}{P}}+6\beta 
H\sqrt{\frac{d\bar{K}K\iota}{P}}
+ 10\beta dH\log\left(1+\frac{KP}{d}\right).
\end{align}

\begin{lemma}[Concentration bounds using optimism] \label{lem:bonus_concentrate_plan_game} 
Consider the setting of Theorem~\ref{thm:main-reward-free_MG}. Then, the following event $\tilde{\cE}$ holds with probability at least $1-\delta$: $\forall (x,a,b,h) \in \S \times \A \times \cB \times [H]$,
\begin{align}
&|\Pe_h \Vo_{h+1}(x,a,b) - \wo_h^\top\phi(x,a,b)|  \leq %u_h(x,a,b), 
\beta\norm{\phi(x,a,b)}_{\Lambda_h^{-1}},
\label{eq:bonus_concentrate_game1}\\
&|\Pe_h \Vu_{h+1}(x,a,b) - \wu_h^\top\phi(x,a,b)|  \leq %u_h(x,a,b).
\beta\norm{\phi(x,a,b)}_{\Lambda_h^{-1}}.
\label{eq:bonus_concentrate_game2}
\end{align}
\end{lemma}
\begin{proof} 
We first focus on proving equation~\eqref{eq:bonus_concentrate_game1}.
Closely following the exact proof as in Lemma~\ref{lem:wn_estimate}, we conclude that 
$\norm{\wo_h}\leq H\sqrt{dKP}$ for any $h\in[H]$.
Then, we notice that $\Vo_h$, $h\in[H]$, belongs to the following function class 
\begin{multline}
\label{eq:function_class-MG}
%\begin{aligned}
\mathcal{V} = \left\{V:\S\to\R\;\Big|\; V(\cdot)=
\min\left\{
\max\left\{
\max_{\pi'\in\Delta(\A)}\min_{\nu'\in\Delta(\cB)}\E_{a\sim\pi',b\sim\nu'}[\hat{w}^\top\phi(\cdot,a,b)+r(\cdot,a,b) \right.\right.\right.\\
%&\quad\quad\quad\quad\quad\quad\quad\quad\quad\quad 
\left.\left.\left. +\min\{\beta\sqrt{\phi(\cdot,a,b)^\top\hat{\Lambda}^{-1}\phi(\cdot,a,b)},H\}]
,0\right\}, H
\right\}\right\},
%\end{aligned}
\end{multline}
where $\norm{\hat{w}}\leq H\sqrt{dKP}$, the minimum eigenvalue of $\hat{\Lambda}$ is greater or equal than $\lambda$, and $\norm{\phi(x, a,b)}\leq 1$ for all $(x,a,b)\in\S\times\A\times\cB$. For any $V,V'\in\mathcal{V}$, let $\dist(V,V')=\sup_{x\in\S} |V(x) - V'(x)|$. Set $\hat{u}(x,a,b):=\min\{\beta\sqrt{\phi(\cdot,a,b)^\top\hat{\Lambda}^{-1}\phi(\cdot,a,b)},H\}$. Then, 
\begin{equation}
\label{eq:dist-covering-MG}
\begin{aligned}
\dist(V,V')
&=\sup_{x\in\S}
\Big|\max_{\pi'\in\Delta(\A)}\min_{\nu'\in\Delta(\cB)}\E_{a\sim\pi',b\sim\nu'}\left[\Pi_{[0,H]}[\hat{w}^\top\phi(x,a,b)+r(x,a,b)+\hat{u}(x,a,b)]\right]\\
&\quad-
\max_{\pi'\in\Delta(\A)}\min_{\nu'\in\Delta(\cB)}\E_{a\sim\pi',b\sim\nu'}\left[\Pi_{[0,H]}[(\hat{w}')^\top\phi(x,a,b)+r(x,a,b)+\hat{u}'(x,a,b)]\right]\Big|\\
&\overset{\textrm{(a)}}{\leq}\sup_{x\in\S,a\in\A,b\in\cB}
\Big|\Pi_{[0,H]}[\hat{w}^\top\phi(x,a,b)+r(x,a,b)+\hat{u}(x,a,b)]\\
&\quad-
\Pi_{[0,H]}[(\hat{w}')^\top\phi(x,a,b)+r(x,a,b)+\hat{u}'(x,a,b)]\Big|\\
&\overset{(b)}{\leq}\sup_{x\in\S,a\in\A,b\in\cB}\Big|
(\hat{w}^\top\phi(x,a,b)+r(x,a,b)+\hat{u}(x,a,b))\\
&\quad-
((\hat{w}')^\top\phi(x,a,b)+r(x,a,b)+\hat{u}'(x,a,b))\Big|\\
&\leq\sup_{\phi:\norm{\phi}\leq 1}\Big|
(\hat{w}-\hat{w}')^\top\phi\Big|
+\sup_{\phi:\norm{\phi}\leq 1}\Big| \hat{u}(x,a,b)-\hat{u}'(x,a,b)\Big|\\
&\overset{\textrm{(c)}}{\leq}\sup_{\phi:\norm{\phi}\leq 1}\Big|
(\hat{w}-\hat{w}')^\top\phi\Big|\\
&\quad 
+\sup_{\phi:\norm{\phi}\leq 1}\beta\Big|
\sqrt{\phi^\top\hat{\Lambda}^{-1}\phi}-
\sqrt{\phi^\top(\hat{\Lambda}')^{-1}\phi}
\Big|\\
&\overset{\textrm{(d)}}{\leq} \norm{\hat{w}-\hat{w}'} + \sup_{\phi:\norm{\phi}\leq 1}
\beta\sqrt{\left|\phi^\top(\hat{\Lambda}^{-1}-(\hat{\Lambda}')^{-1})\phi\right|}\\
&=\norm{\hat{w}-\hat{w}'} + 
\beta\sqrt{\norm{\hat{\Lambda}^{-1}-(\hat{\Lambda}')^{-1}}}\\
&\leq\norm{\hat{w}-\hat{w}'} + 
\beta\sqrt{\norm{\hat{\Lambda}^{-1}-(\hat{\Lambda}')^{-1}}_F}
\end{aligned}
\end{equation}
where (a) follows from the fact that the max-min operator is non-expansive and the fact that given a set $\mathcal{G}$ and function $g:\mathcal{G}\to \R$, $\max_{\mu\in\Delta(\mathcal{G})}\E_{y\sim \mu}[g(y)]\leq \max_{y\in\mathcal{G}}g(y)$; (b) follows from the fact that the clipping %truncation 
operator is non-expansive; (c) follows from the min operator being non-expansive; and (d) follows from the inequality $|\sqrt{p}-\sqrt{q}|\leq\sqrt{|p-q|}$ for any $p,q\geq 0$. Now, we notice that~\eqref{eq:dist-covering-MG} is a bound of the same form as equation~(28) from~\citep[Lemma~D.6]{CJ-ZY-ZW-MIJ:20}, and so we can use this lemma to conclude that the $\epsilon$-covering number of $\mathcal{V}$, $\mathcal{N}_{\epsilon}$, with respect to the distance $\dist(\cdot,\cdot)$ can be upper bounded as 
\begin{equation}
\label{eq:cover-aux}
\log \mathcal{N}_{\epsilon} \le d  \log (1+ 4H\sqrt{dKP}/ \epsilon ) + d^2 \log \bigl [ 1 +  8 d^{1/2} \beta^2  / (\lambda\epsilon^2)  \bigr ].
\end{equation}
Then, we can closely follow the derivation of the equation in Lemma~\ref{lem:stochastic_term} to obtain that the following event holds with probability at least $1-\delta$,
\begin{multline*}
\forall h\in [H]: \quad \norm{\sum_{\tau = 1}^{K}\sum_{p=1}^P \phi^{\tau,p}_h [\Vo_{h+1}(x^{\tau,p}_{h+1}) - \Pe_h \Vo_{h+1}(x_h^{\tau,p}, a_h^{\tau,p},b_h^{\tau,p})]}_{\Lambda_h^{-1}}
\\\leq CdH\sqrt{\log [(c_\beta+1)dKHP/\delta]}. 
\end{multline*}
for some absolute constant $C$ independent of $c_\beta>0$.
We now condition on this event. Then, we can closely follow the same proof procedure as the one given in Lemma~\ref{lem:sum_V} to derive equation~\eqref{eq:optimism-bound-val-RF} and thus obtain, for any $(x,a,b)\in\S\times\A\times\cB$ with $\lambda=1$,
$$
|\phi(x,a,b)^\top\wo_h-\Pe_h \Vo_{h+1}(x,a,b)|\leq \beta\norm{\phi(x,a,b)}_{\Lambda_h^{-1}},
$$
where $\beta$ can be chosen in exactly the same way as in Theorem~\ref{thm:main-paralel},i.e., exactly as %equations~\eqref{item:two_up_MG}  and
in~\eqref{item:three_up_MG}.  This finishes the proof for inequality~\eqref{eq:bonus_concentrate_game1}. 

To prove inequality~\eqref{eq:bonus_concentrate_game2}, it is easy to prove that $\Vu_h$, $h\in[H]$, belongs to a very similar function class as $\mathcal{V}$ in~\eqref{eq:function_class-MG}, and so we can use the same proof for inequality~\eqref{eq:bonus_concentrate_game1}.
\end{proof}

\begin{lemma}[Bounds on value functions for the planning phase] \label{lem:bonus_plan_game}
Consider the setting of Theorem~\ref{thm:main-reward-free_MG}. Conditioned on the event $\tilde{\cE}$ defined in Lemma \ref{lem:bonus_concentrate_plan_game}, we have that for every $(x,h)\in\S\times[H]$
\begin{align}
&V_h^\dagger(x; r) \leq \Vo_h(x) \leq \E_{a\sim \pi_h, b\sim \bre_2(\pi)_h} [(\PP_h\overline{V}_{h+1}  + r_h + 2u_h)(x,a,b)], \label{eq:optimism_game_1} \\
&V_h^\dagger(x; r) \geq \underline{V}_h(x)\geq \EE_{a\sim \bre_1(\nu)_h, b\sim \nu_h} [(\PP_h\underline{V}_{h+1}  - r_h - 2u_h)(x,a,b)]. \label{eq:optimism_game_2} 
\end{align}
\end{lemma}
\begin{proof} 
We first prove the leftmost  inequality~\eqref{eq:optimism_game_1}, which we do by induction. The case for step $H+1$ is trivial since $V_{H+1}^\dagger(x;r)=V_{H+1}(x)=0$ for any $x\in \S$. We now consider the induction step $V_{h+1}^\dagger(x; r)\leq\overline{V}_{h+1}(x)$. Then, \begin{align*}
Q_h^\dagger(x,a,b; r)
&\overset{\textrm{(a)}}{=} r_h(x,a,b) + \Pe_h V_{h+1}^\dagger(x,a,b; r)\\
 &\overset{(b)}{\leq} r_h(x,a,b) + \Pe_h \Vo_{h+1}(x,a,b; r) \\
 &\overset{\textrm{(c)}}{\leq}r_h(x,a,b) + \bar{w}_h^\top\phi(x,a,b)+\beta\norm{\phi(x,a,b; r)}_{\Lambda^{-1}_h}  \\
 \overset{\textrm{(d)}}{\implies} Q_h^\dagger(x,a,b; r)&\leq \min\{r_h(x,a,b) + \Pe_h \Vo_{h+1}(x,a,b; r)+\beta\norm{\phi(x,a,b; r)}_{\Lambda^{-1}_h},H\}\\
 &\overset{\textrm{(e)}}{=}\min\{r_h(x,a,b) + \bar{w}_h^\top\phi(x,a,b)+u_h(x,a,b),H\}\\
 &=\overline{Q}_h(x,a,b)
\end{align*}
where (a) follows from the Bellman equation, (b) follows from the induction step, (c) follows from Lemma \ref{lem:bonus_concentrate_plan_game}, (d) follows from 
$Q_h^\dagger(x,a,b; r)\in[0,H]$, and (e) from the fact that $u_h(\cdot,\cdot,\cdot)=\min\{\beta\norm{\phi(\cdot,\cdot,\cdot)}_{\Lambda_{h}^{-1}},H\}$. 
Then, using this result leads to %
\begin{align*}
V_h^\dagger(x; r) &= \max_{\pi_h'\in\Delta(\A)}\min_{\nu_h'\in\Delta(\cB)}\E_{a\sim \pi_h', b\sim \nu_h'} [Q_h^\dagger(x, a, b; r)] \\
&\leq \max_{\pi_h'\in\Delta(\A)}\min_{\nu_h'\in\Delta(\cB)} \EE_{a\sim \pi_h', b\sim \nu_h'} [\overline{Q}_h(x, a, b)]\\
&=\Vo_h(x). 
\end{align*}
This finishes the proof by induction of the leftmost inequality in~\eqref{eq:optimism_game_1}.

Now we analyze the rightmost inequality in~\eqref{eq:optimism_game_1} follows from
\begin{align*}
\overline{V}_h(x) &=  \min_{\nu'\in\Delta(\cB)} \EE_{a\sim \pi_h, b\sim \nu'} \overline{Q}_h(x,a, b)\\
&=\E_{a\sim \pi_h, b\sim \bre_2(\pi)_h} \overline{Q}_h(x,a, b)\\
&=\E_{a\sim \pi_h, b\sim \bre_2(\pi)_h} \min\{\max\{(\wo_h^\top\phi + r_h + u_h)(x,a,b),0\}, H\} \\
&\overset{\textrm{(a)}}{\leq}\E_{a\sim \pi_h, b\sim \bre_2(\pi)_h} \min\{\max\{(\Pe_h \overline{V}_{h+1}  + r_h +  \beta\norm{\phi(\cdot,\cdot,\cdot)}_{\Lambda_h^{-1}}+u_h)(x,a,b),0\},H\} \\
\overset{(b)}{\implies} \overline{V}_h(x) &\leq\E_{a\sim \pi_h, b\sim \bre_2(\pi)_h} \min\{\max\{(\Pe_h \overline{V}_{h+1}  + r_h + 2u_h)(x,a,b),0\}, H\} \\
&\overset{(c)}{\leq}\E_{a\sim \pi_h, b\sim \bre_2(\pi)_h} [(\PP_h\overline{V}_{h+1}  + r_h + 2u_h)(x,a,b)],
\end{align*}
where (a) follows from  Lemma~\ref{lem:bonus_concentrate_plan_game}; (b) follows from $\overline{V}_h(x)\leq H$ and the definition of $u_h$; and (c) follows from $(\Pe_h\overline{V}_{h+1}  + r_h + 2u_h)(s,a,b)\geq 0$. This finishes the proof for the rightmost inequality in~\eqref{eq:optimism_game_1}.
	
Finally, the proof for the inequalities in~\eqref{eq:optimism_game_2} can be obtained by carefully following and modifying all the proof just done for the inequalities in~\eqref{eq:optimism_game_1}.
\end{proof}

\begin{lemma}[Bounding best-responses with rewards taken from the planning phase] \label{lem:explore_plan_connect_game} We have that for any initial state $s_0\in\S$,
\begin{align*}
& V_1^{\pi, \bre_2(\pi)}(s_0, u/H) \leq \frac{1}{K} \sum_{k=1}^K V_1^*(s_0, r^k),\\
&V_1^{\bre_1(\nu) , \nu }(s_0, u/H) \leq \frac{1}{K} \sum_{k=1}^K V_1^*(s_0, r^k).
\end{align*}
\end{lemma}
\begin{proof} 
\begin{equation}
\begin{aligned}
V_1^{\pi, \bre_2(\pi)}(s_0, u/H)
&\overset{\textrm{(a)}}{\leq} V_1^*(s_0, u/H)\overset{(b)}{\leq} V_1^*(s_0,r^k)\\
\implies V_1^{\pi, \bre_2(\pi)}(s_0, u/H)&\leq \frac{1}{K}\sum_{k=1}^K V_1^*(s_0,r^k)
\end{aligned}
\end{equation}
where (a) follows from the definition of the optimal value function, (b) from the same arguments shown in deriving equation~\eqref{eq:value-rew-RF} for any $k\in[K]$. 

The second inequality can be obtained similarly.
\end{proof}

We now continue with the proof of Theorem~\ref{thm:main-reward-free_MG}, conditioning on the event such that equations~\eqref{item:one_up_MG}--\eqref{item:three_up_MG} hold and on the event defined in Lemma~\ref{lem:bonus_concentrate_plan_game}, which altogether hold with probability at least $1 - 3\delta$.

Then, 
\begin{align*}
&V_1^\dagger(s_0;r) - V_1^{\pi, \bre_2(\pi)}(s_0;r)\\
&\overset{\textrm{(a)}}{\leq} \overline{V}_1(s_0) - V_1^{\pi, \bre_2(\pi)}(s_0; r)\\
&\overset{(b)}{\leq}\E_{a_1\sim \pi_1, b_1\sim \bre_2(\pi)_1} [(r_1 + \Pe_1\Vo_{2} + 2u_1)(s_0,a_1,b_1)]- V_1^{\pi, \bre_2(\pi)}(s_0; r) \\
&\overset{\textrm{(c)}}{=}\E_{a_1\sim \pi_1, b_1\sim \bre_2(\pi)_1} [(r_1+ \Pe_1 \overline{V}_{2}+ 2u_1)(s_0,a_1, b_1) - r_1(s_0,a_1,b_1)\\
&\quad - \Pe_1 V_{2}^{\pi, \bre_2(\pi)}(s_0,a_1,b_1; r) ]\\
&= \E_{a_1\sim \pi_1, b_1\sim \bre_2(\pi)_h} [ \Pe_1 \overline{V}_{2}(s_0,a_1,b_1) - \Pe_1 V_{2}^{\pi,\bre_2(\pi)}(s_0,a_1,b_1; r) + 2u_1(s_0,a_1,b_1)]\\
&=\E_{a_1\sim \pi_1,b_1\sim\bre_2(\pi)_1, x_{2}\sim\P_1} [ \overline{V}_{2}(x_{2}) -  V_{2}^{\pi, \bre_2(\pi)}(x_{2}; r) + 2u_1(s_0,a_1,b_1)]\\
&\leq \ldots\\
&\leq \E_{\forall h\in [H]: ~a_h\sim \pi_h, b_h\sim \bre_2(\pi)_h, x_{h+1}\sim\P_h}\left[\sum_{h=1}^H 2u_h(x_h, a_h, b_h)\,\Bigg|\, x_1=s_0 \right]\\
&=2H V_1^{\pi, \bre_2(\pi)}(s_0, u/H)\\
&\overset{\textrm{(d)}}{\leq}
2H\frac{1}{K}\sum_{k=1}^K V_1^*(s_0, r^k)\\
&\overset{\textrm{(e)}}{\leq}
2H\frac{1}{K}\sum_{k=1}^K V_1^k(s_0)
\end{align*}
where (a) and (b) follows from Lemma~\ref{lem:bonus_plan_game}, (c) comes from the Bellman equation, (d) comes from Lemma~\ref{lem:explore_plan_connect_game}, and (e) follows from~\eqref{item:one_up_MG}. 
Then, from~\eqref{item:three_up_MG}
\begin{align}
\label{item:part-two-MG} &V_1^\dagger(x_1; r) - V_1^{\pi, \bre_2(\pi)}(x_1; r)
\leq 4H^2\sqrt{\frac{H\iota}{KP}}+12\beta 
H^2\sqrt{\frac{d\iota}{KP}}
+ 20\beta dH^2\frac{1}{K}\log\left(1+\frac{KP}{d}\right).
\end{align}
%for any $\bar{K}>1$.

Next, we prove the upper bound of the term $V_1^{\bre_1(\nu), \nu}(s_0; r)  - V_1^\dagger(s_0; r)$.
\begin{align*}
&V_1^{\bre_1(\nu), \nu}(s_0; r)  - V_1^\dagger(s_0; r)\\
&\overset{\textrm{(a)}}{\leq} V_1^{\bre_1(\nu), \nu}(s_0; r) - \Vu_1^{\pi, \bre_2(\pi)}(s_0; r)\\
&\overset{(b)}{\leq} V_1^{\bre_1(\nu), \nu}(s_0; r)-\E_{a_1\sim \bre_1(\nu)_1, b_1\sim \nu_1} [(\Pe_1\Vu_{2} - r_1 - 2u_1)(s_0,a_1,b_1)]\\
&\overset{\textrm{(c)}}{=}\E_{a_1\sim \bre_1(\nu)_1, b_1\sim \nu_1}[\Pe_1 V_{2}^{\bre_1(\nu),\nu}(s_0,a_1,b_1; r)
- \Pe_1 \Vu_{2}(s_0,a_1, b_1) +2u_1(x_1,a_1,b_1; r)]\\
&= \E_{a_1\sim \bre_1(\nu)_1, b_1\sim \nu_1,x_2\sim\P_1}[
V_{2}^{\bre_1(\nu),\nu}(x_2; r)
-\Vu_{2}(x_2) + 2u_1(x_1,a_1,b_1)]\\
&\leq \ldots\\
&\leq \E_{\forall h\in [H]: ~a_h\sim \bre_1(\nu)_h,b_h\sim\nu_h, x_{h+1}\sim\P_h}\left[\sum_{h=1}^H 2u_h(x_h, a_h, b_h)\,\Bigg|\, x_1=s_0 \right]\\
&=2H V_1^{\bre_1(\nu),\nu}(s_0, u/H)\\
&\overset{\textrm{(d)}}{\leq}
2H\frac{1}{K}\sum_{k=1}^K V_1^*(s_0, r^k)\\
&\overset{\textrm{(e)}}{\leq}
2H\frac{1}{K}\sum_{k=1}^K V_1^k(s_0)
\end{align*}
where (a) and (b) follows from Lemma~\ref{lem:bonus_plan_game}, (c) comes from the Bellman equation, (d) comes from Lemma~\ref{lem:explore_plan_connect_game}, and (e) follows from~\eqref{item:one_up_MG}. %
Then, from~\eqref{item:three_up_MG} 
\begin{align}
\label{item:part-two-MG_l} &V_1^{\bre_1(\nu), \nu}(x_1; r)  - V_1^\dagger(x_1; r)
\leq 4H^2\sqrt{\frac{H\iota}{KP}}+12\beta 
H^2\sqrt{\frac{d\iota}{KP}}
+ 20\beta dH^2\frac{1}{K}\log\left(1+\frac{KP}{d}\right).
\end{align}
%for any $\bar{K}>1$.

Now, adding 
\eqref{item:part-two-MG} with~\eqref{item:part-two-MG_l}, we obtain,
\begin{align*}
V_1^{\bre_1(\nu), \nu}(s_0; r)-V_1^{\pi, \bre_2(\pi)}(s_0; r)&\leq 8H^2\sqrt{\frac{H\iota}{KP}}+12\beta 
H^2\sqrt{\frac{d\iota}{KP}}
+ 20\beta dH^2\frac{1}{K}\log\left(1+\frac{KP}{d}\right)\\
&=8\sqrt{\frac{H^5\iota}{KP}}+12c_\beta\sqrt{\frac{d^3H^6\iota^2}{KP}}
+ 20c_\beta\frac{\sqrt{d^4H^6}}{ K}\log\left(1+\frac{KP}{d}\right)\\
&\leq (8+12c_\beta)\sqrt{\frac{d^3H^6\iota^2}{KP}}+20c_\beta\frac{\sqrt{d^4H^6}}{K}\log\left(1+\frac{KP}{d}\right)
\end{align*}
This finishes the proof.

\section{Proof of Theorem~\ref{thm:lower_bound}}

As mentioned in the main document, our proof technique mimics that of~\citep{AW-YC-MS-SSD-KJ:22}.

We introduce some notation:  for a vector $a\in\R^d$, let $a_i$ denote its $i$-th component; for any positive integer $A$, let $[A]=\{1,\dots,A\}$.

We first define the following \emph{parallel} linear bandit setting.
\begin{definition}[Parallel linear bandit setting]
\label{def:def-bandit-parallel}
Let $K,P>0$ be arbitrary and fixed strictly positive integers, 
$\Phi = \{\phi \in \RR^d: \|\phi\|_2 = 1\}$ 
be the $d$-dimensional Euclidean norm 
sphere
%ball, 
and $\Theta = \{-\mu, \mu\}^d$ for some $\mu \in (0, \frac{1}{20\sqrt{d}}]$. For some some fixed $\theta \in \Theta$, consider the query model where at every step $k = 1, \dots, K$ we choose a batch of query values $\{\phi^{k, 1}, \dots,  \phi^{k, P}\} \subseteq \Phi$ 
and observe, for each $p \in [P]$ independently, the reward
\begin{equation}
    y^{k, p} \sim \textrm{Bernoulli}(1/2 + \langle \theta, \phi^{k, p} \rangle).
    \label{eq:y_bern}
\end{equation}
After choosing the last batch of query values at step $K$, the query model outputs a final query value $\phi^{K+1}$. 
We call $K$ the number of episodes for which we run the query model. 
\end{definition}

Note that the parallel linear bandit setting in Definition~\ref{def:def-bandit-parallel} becomes the \emph{single} linear bandit setting when $P=1$.

Also, note that any query strategy $\pi$ which produces a final output according to our query model induces a distribution over the set $\Phi$ in its output, thus, for a parallel linear bandit setting with $K$ episodes we have that $\phi^{K+1}\sim\pi$.

We begin by producing a lower bound on parallel adaptive linear regression.

\begin{lemma}
\label{lem:hard_linear_bandit}
Consider the parallel linear bandit setting in Definition~\ref{def:def-bandit-parallel} with $K$ episodes. 
Then,
\[
    \inf_{\hat{\theta}, \pi}\max_{\theta \in \Theta}\EE_{\theta}[\|\hat{\theta} - \theta\|_2^2] \geq \frac{d\mu^2}{2}\rbr{1 - \sqrt{\frac{20KP\mu^2}{d}}}
\] 
where the infimum is taken over all measurable estimators $\hat{\theta}$ and measurable (but potentially adaptive) query policies $\pi$, and $\EE_\theta$ denotes the expectation 
with respect to the randonmness in the observations and queries 
% under the distribution induced 
when $\theta \in \Theta$ is the true parameter. 
\end{lemma}
\begin{proof}
Our proof is similar in spirit to that of Theorem 5 in~\citep{AW-YC-MS-SSD-KJ:22}. Consider a given estimator $\hat{\theta}$ produced by a query strategy.  We immediately have
\begin{multline}
\label{eq:prim-lowebound}
\max_{\theta \in \Theta}\EE_{\theta}[\|\hat{\theta} - \theta\|^2] 
\geq \EE_{\theta \sim \textrm{Uniform}(\Theta)}\EE_{\theta}[\|\hat{\theta} - \theta\|^2] \\
= \EE_{\theta \sim \textrm{Uniform}(\Theta)}\EE_{\theta}[\sum^d_{i=1}(\hat{\theta}_i-\theta_i)^2]
\overset{(a)}{\geq} \EE_{\theta \sim \textrm{Uniform}(\Theta)}\EE_{\theta}\sbr{\mu^2\sum_{i = 1}^d \I[\theta_i\hat{\theta}_i < 0]},
\end{multline}
where (a) follows from the fact that $(\theta_i-\hat{\theta}_i)^2\in\{0,4\mu^2\}$ with $(\theta_i-\hat{\theta}_i)^2=4\mu^2$ if and only if $\theta_i\hat{\theta}_i<0$.

We now closely follow and adapt the proof of~\citep{shamir2013}[Lemma~4] to our setting until the derivation of equation~\eqref{eq:fin-expr}.
We then have 
\begin{align*}
&\EE_{\theta \sim \textrm{Uniform}(\Theta)}\EE_{\theta}\sbr{\sum_{i = 1}^d \I[\theta_i\hat{\theta}_i < 0]}\\
&\qquad = \sum^d_{i=1}\bar{\Pr}[\theta_i\hat{\theta}_i<0]\\
&\qquad\overset{\textrm{(a)}}{=}\sum_{i = 1}^d \rbr{\frac{1}{2}\Pr[\hat{\theta}_i\theta_i < 0 | \theta_i > 0] + \frac{1}{2}\Pr[\hat{\theta}_i\theta_i < 0 | \theta_i < 0]}\\
&\qquad=\frac{1}{2}\sum_{i = 1}^d \rbr{\Pr[\hat{\theta}_i < 0 | \theta_i > 0] + \Pr[\hat{\theta}_i > 0 | \theta_i < 0]}\\
&\qquad = \frac{1}{2}\sum_{i = 1}^d \rbr{1 - (\Pr[\hat{\theta}_i > 0 | \theta_i > 0] - \Pr[\hat{\theta}_i > 0 | \theta_i < 0])}\\
&\qquad \geq \frac{d}{2}\rbr{1 - \frac{1}{d}\sum_{i = 1}^d |\Pr[\hat{\theta}_i > 0 | \theta_i > 0] - \Pr[\hat{\theta}_i > 0 | \theta_i < 0]|}\\
&\qquad\overset{\textrm{(b)}}{\geq}\frac{d}{2}\rbr{1 - \sqrt{\frac{1}{d}\sum_{i = 1}^d(\Pr[\hat{\theta}_i > 0 | \theta_i > 0] - \Pr[\hat{\theta}_i > 0 | \theta_i < 0])^2}},
\end{align*}
where (a) follows from $\theta \sim \textrm{Uniform}(\Theta)$, (b) follows from $\sum_{i=1}^d|a_i|\leq \sqrt{d}\norm{a}$ for any $a\in\R^d$, and where $\bar{\Pr}$ is a measure with respect to the distribution of $\theta$ and with respect to both the observations and query strategy, whereas $\Pr$ is only with respect to both the observations and query strategy. From now on, we will abuse the notation $\Pr$ to denote different types of measures across our derivations.

Now, for an arbitrary and fixed $i' \in [d]$, 
\begin{align*}
    &(\Pr[\hat{\theta}_{i'} > 0 | \theta_{i'} > 0] - \Pr[\hat{\theta}_{i'} > 0 | \theta_{i'} < 0])^2\\
    & \qquad \overset{\textrm{(a)}}{=}\Big( \sum_{u\in\{-\mu,\mu\}^{d-1}}\Pr[(\theta_{j}: j \neq i')=u]\\
    & \qquad \qquad \times (\Pr[\hat{\theta}_{i'} > 0 | \theta_{i'} > 0, (\theta_{j}: j \neq i')=u] - \Pr[\hat{\theta}_{i'} > 0 | \theta_{i'} < 0, (\theta_{j}: j \neq i')=u]\Big)^2\\
    & \qquad \leq \sum_{u\in\{-\mu,\mu\}^{d-1}}\Pr[(\theta_{j}: j \neq i')=u]\\
    & \qquad \qquad \times (\Pr[\hat{\theta}_{i'} > 0 | \theta_{i'} > 0, (\theta_{j}: j \neq i')=u] - \Pr[\hat{\theta}_{i'} > 0 | \theta_{i'} < 0, (\theta_{j}: j \neq i')=u])^2\\
    & \qquad \leq \max_{u\in\{-\mu,\mu\}^{d-1}}(\Pr[\hat{\theta}_{i'} > 0 | \theta_{i'} > 0, (\theta_{j}: j \neq i')=u] - \Pr[\hat{\theta}_{i'} > 0 | \theta_{i'} < 0, (\theta_{j}: j \neq i')=u])^2.
\end{align*}
where (a) follows from the law of total probability. Now, observe that the last term is the square of the total variation distance, and so, by Pinsker's inequality we obtain 
\begin{align*}
    &\max_{u\in\{-\mu,\mu\}^{d-1}}(\Pr[\hat{\theta}_{i'} > 0 | \theta_{i'} > 0, (\theta_{j}: j \neq i')=u] - \Pr[\hat{\theta}_{i'} > 0 | \theta_{i'} < 0, (\theta_{j}: j \neq i')=u])^2\\
    &\qquad \leq \frac{1}{2}\KL\rbr{\Pr[\hat{\theta}_{i'} > 0 | \theta_{i'} > 0,\{\theta_{j}: j \neq i'\}] ||\Pr[\hat{\theta}_{i'} > 0 | \theta_{i'} < 0,\{\theta_{j}: j \neq i'\}]}
\end{align*} where $\KL(p || q)$ denotes the Kullback-Leibler (KL) divergence between distributions $p$ and $q$.
        
As any randomized query strategy $\pi$ which outputs $\phi^{K+1}$ can be characterized by a distribution over deterministic query strategies, we assume without loss of generality that the query strategy is deterministic and show that our lower bound holds uniformly over all possible deterministic query strategies. Then, as $\hat{\theta}=\phi^{K+1}$, we are assuming that $\hat{\theta}$ is a deterministic function of $y^{k,p},(k,P)\in[K]\times[P]$ (and that, consequently, $\phi^{k,p}$ is a deterministic function function of $y^{\hat{k},\hat{p}}$, $(\hat{k},\hat{p})\in[P]\times[k-1]$). 

Under this assumption, 
\begin{align*}
    &\KL\rbr{\Pr[\hat{\theta}_{i'} > 0 | \theta_{i'} > 0,\{\theta_{j}\}_{j \neq i'}] ||\Pr[\hat{\theta}_{i'} > 0 | \theta_{i'} < 0,\{\theta_{j}\}_{j \neq i'}]}\\
    & \overset{(a)}{=} \KL\rbr{\Pr[y^{k,p},(k,p)\in[K]\times[P] | \theta_{i'} > 0,\{\theta_{j}\}_{j \neq i'}] ||\Pr[y^{k,p},(k,p)\in[K]\times[P] | \theta_{i'} < 0,\{\theta_{j}\}_{j \neq i'}]}\\
    & \overset{\textrm{(b)}}{\leq}
    %\sum_{k = 1}^K\KL\rbr{\Pr[y^{k,p},p\in[P] | \theta_{i'} > 0, \{\theta_{j}\}_{j \neq i'}, \{y^{k',p}\}_{k'=1,p=1}^{k-1,P}] ||\Pr[y^{k,p},p\in[P]| \theta_{i'} < 0 , \{\theta_{j}\}_{j \neq i'},\{y^{k',p}\}_{k'=1,p=1}^{k-1,P}]}\\
    %& \overset{\textrm{(c)}}{=} 
    \sum_{k = 1}^K \sum_{p = 1}^P {\KL\rbr{\Pr[y^{k, p} | \theta_{i'} > 0, \{\theta_{j}\}_{j \neq i'},(y^{k',p'})_{(k',p')\in\cL_P(k,p)}] ||\Pr[y^{k, p} | \theta_{i'} < 0 , \{\theta_{j}\}_{j \neq i'},(y^{k',p'})_{(k',p')\in\cL_P(k,p)}]}}\\
    & \overset{\textrm{(c)}}{=}
    \sum_{k = 1}^K \sum_{p = 1}^P \KL\rbr{\Pr[y^{k, p} | \theta_{i'} > 0, \{\theta_{j}\}_{j \neq i'},(y^{k',p'})_{(k',p')\in[k-1]\times[P]]}] ||\Pr[y^{k, p} | \theta_{i'} < 0 , \{\theta_{j}\}_{j \neq i'},(y^{k',p'})_{(k',p')\in[k-1]\times[P]}]}
\end{align*}
where (a) follows from the fact that $\hat{\theta}_{i'}$ is a deterministic function of all the (random) observations $\{y^{k,p}:(k,p)\in[K]\times[P]\}$ and so describes a new distribution over the observations (since we do not know the form of the exact dependency we decided to simply list ``$y^{k,p},(k,p)\in[K]\times[P]$" on the argument); (b) follows from the chain rule for KL divergence; and (c) follows from the fact that for a given episode, the actions taken by each agent are independent (we use the notation that when $k=1$, there is no conditioning on any variable depending on any output reward).%, and (c) also considers the fact that for any $k\in[K]$, the set $\{y^{k,1},\dots,y^{k,P}\}$ contains mutually independent random variables.

Joining all the results herein derived (and under some renaming of variable indices), we obtain 
        \begin{equation}
        \begin{aligned}
           \EE_{\theta \sim \textrm{Uniform}(\Theta)}\EE_{\theta}\sbr{\sum_{i = 1}^d \I[\theta_i\hat{\theta}_i < 0]}
           &\geq
           \frac{d}{2}\rbr{1 - \sqrt{\frac{1}{d}\sum_{i = 1}^d
          \sum_{k=1}^K\sum_{p=1}^P U_{k,p,i}
           }}
        \end{aligned}
            \label{eq:fin-expr}
        \end{equation}
where
$$
U_{k,p,i}=\sup_{\{\theta_{i'}: i' \neq i\}} \KL\rbr{\Pr[y^{k, p} | \theta_{i} > 0, \{\theta_{i'}\}_{i' \neq i},(y^{k',p})_{(k',p')\in\cL_P(k,p)}] ||\Pr[y^{k, p} | \theta_{i} < 0 , \{\theta_{i'}\}_{i' \neq i},(y^{k',p})_{(k',p')\in\cL_P(k,p)}]}
$$

The rest of the proof directly follows from the proof of Theorem 5 of~\citep{AW-YC-MS-SSD-KJ:22}, where we directly calculate the KL divergence between two Bernoulli distributions and use the fact that $\phi^{k, p} \in \Phi$ for all $(k, p) \in [K]\times[P]$ to obtain the inequality
$$
\EE_{\theta \sim \textrm{Uniform}(\Theta)}\EE_{\theta}\sbr{\sum_{i = 1}^d \mu^2\I[\theta_i\hat{\theta}_i < 0]}
\geq
\frac{d\mu^2}{2}\rbr{1 - \sqrt{\frac{20KP\mu^2}{d}}}
$$
which we replace back in~\eqref{eq:prim-lowebound} to obtain 
$$
\max_{\theta \in \Theta}\EE_{\theta}[\|\hat{\theta} - \theta\|^2]
\geq
\frac{d\mu^2}{2}\rbr{1 - \sqrt{\frac{20KP\mu^2}{d}}}.
$$
Finally, since this uniform lower bound holds for any deterministic query strategy which outputs the estimator $\hat{\theta}$, it will also hold for any randomized one (which, as we explained before, can be seen as a distribution over deterministic query strategies), and thus we can take the infimum as in the lemma's statement. This completes the proof.
\end{proof}

We introduce some terminology. Let $\phi^{\star}(\theta)=\sup_{\phi\in\Phi}\langle\theta,\phi\rangle$ denote the best arm for any arbitrary $\theta \in \Theta$. Now, given some arbitrary fixed parameter $\theta\in\Theta$ and given a query strategy $\pi$ 
(which we know outputs a single query value or arm according to some induced distribution over the set $\Phi$), we define $\phi^{\pi}=\EE_{\phi\sim\pi}[\phi]$ and say that $\pi$ is $\epsilon$-optimal if $\langle\theta,\phi^\star(\theta)\rangle-\langle\theta,\phi^{\pi}\rangle\leq \epsilon$. 

We now establish the following result on our parallel linear bandits setting, which is basically~\citep{AW-YC-MS-SSD-KJ:22}[Theorem~2] and which we include for completeness since we do adapt some notation and terminology according to our particular setting and presentation of the problem. We also provide additional detail where it is needed in the proof.

\begin{theorem}
Let $\epsilon>0$, $P>0$, $d>1$, and $KP\geq d^2$. Consider the parallel linear bandit setting in Definition~\ref{def:def-bandit-parallel} with $\mu=\sqrt{\frac{d}{700KP}}$ and also consider running a (potentially adaptive) parallel algorithm with batches of size $P$ for $K$ episodes which stops at a possibly random stopping time $\tau$ and outputs a policy $\hat{\pi}$ which is a guess at an $\epsilon$-optimal policy. 
Then, there is a universal constant $c > 0$ such that unless $KP \geq c (dH/\epsilon)^2$, there exists $\theta\in\Theta$ for which $\Pr_{\theta}[\{\tau > K \textrm{ or }\hat{\pi} \textrm{ is not $\epsilon$-optimal}\}] \geq 0.1$; i.e., with constant probability either $\hat{\pi}$ is not $\epsilon$-optimal or more than $K$ batches are collected. 
\end{theorem}
\begin{proof}
Since $\phi^{\star}(\theta)=\sup_{\phi\in\Phi}\langle\theta,\phi\rangle$ for any $\theta \in \Theta$, then we easily see  that
\[
    \phi^{\star}(\theta) = \frac{\theta}{\norm{\theta}}= \frac{\theta}{\sqrt{d}\mu},\quad (\phi^{\star}(\theta))^T\theta = \sqrt{d}\mu.
\] 
Now, consider a fixed $\theta\in\Theta$. Let $\hat{\pi}$ be an $\epsilon$-optimal query policy for $\theta$. Then,  
\begin{align*}
    1 &= \EE_{\phi \sim \hat{\pi}}[\|\phi\|^2] \overset{\textrm{(a)}}{\geq} \|\phi^{\hat{\pi}}\|^2\\
    &= \|\phi^*(\theta)\|^2 + 2\rbr{\phi^{\hat{\pi}} - \phi^{\star}(\theta)}^\top(\phi^*(\theta)) + \nbr{\phi^{\hat{\pi}} - \phi^{\star}(\theta)}^2\\
    &\overset{\textrm{(b)}}{=}1 + \frac{2}{\sqrt{d}\mu}\rbr{\phi^{\hat{\pi}} - \phi^{\star}(\theta)}^\top\theta + \nbr{\phi^{\hat{\pi}} - \phi^{\star}(\theta)}^2
\end{align*} where (a) comes from Jensen's inequality and (b) from the fact that $\phi^{\star}(\theta) = \frac{\theta}{\sqrt{d}\mu}$. 
Now, from the definition of $\epsilon$-optimality, we have $\rbr{ \phi^{\star}(\theta) - \phi^{\hat{\pi}} }^\top\theta \leq \epsilon$, and thus we obtain 
\begin{equation}
    \nbr{\phi^{\hat{\pi}} - \phi^{\star}(\theta)}^2 \leq \frac{2\epsilon}{\sqrt{d}\mu}.
    \label{eq:up-bound-phi}
\end{equation}

Now, letting $\hat{\pi}$ be any query policy, we show how $\phi^{\hat{\pi}}$ may be used to construct an estimator of $\theta$. With a slight abuse of notation, we define the set $\Theta^{\hat{\pi}} \subseteq \Theta$, where
\[
    \Theta^{\hat{\pi}} = \cbr{\theta' \in \Theta: \|\theta' - \sqrt{d}\mu\phi^{\hat{\pi}}\|^2 \leq 2\sqrt{d}\mu \epsilon}.
\] As $\phi^{\star}(\theta) = \frac{\theta}{\sqrt{d}\mu}$ and $\nbr{\phi^{\hat{\pi}} - \phi^{\star}(\theta)}^2 \leq \frac{2\epsilon}{\sqrt{d}\mu}$ if $\bar{\pi}$ is $\epsilon$-suboptimal, we know that
%\[ 
$    \nbr{\theta - \sqrt{d}\mu \phi^{\bar{\pi}}}^2 \leq 2\sqrt{d}\mu\epsilon,
$
%\]
thereby ensuring the set $\Theta^{\hat{\pi}}$ is non-empty as long as ${\hat{\pi}}$ is $\epsilon$-suboptimal.

Then, we define $\hat{\theta}_{\hat{\pi}}$ to be an estimator of $\theta$ as follows
\[
    \hat{\theta}_{\hat{\pi}} = \begin{cases}
        \textrm{any }\theta' \in \Theta^{\hat{\pi}} &\textrm{ if }
        \textrm{query policy $\hat{\pi}$ is $\epsilon$-optimal}
        \\
        \textrm{any }\theta' \in \Theta &\textrm{ otherwise}.
    \end{cases}
\]
Therefore, if $\hat{\pi}$ is $\epsilon$-optimal for $\theta$, it follows that
\begin{equation}
\label{eq:phi-epsopt}
\begin{aligned}
    \|\hat{\theta}_{\hat{\pi}} - \theta\|^2 &\leq 2\|\hat{\theta}_{\hat{\pi}} - \sqrt{d}\mu\phi^{\hat{\pi}}\|^2 + 2 \|\theta - \sqrt{d}\mu\phi^{\hat{\pi}}\|^2\\
    &\overset{(a)}{\leq} 4\sqrt{d}\mu\epsilon + 2d \mu^2 \nbr{\phi^{\hat{\pi}} - \phi^{\star}(\theta)}^2 \overset{(b)}{\leq} 8\sqrt{d}\mu\epsilon.
\end{aligned}
\end{equation}
where (a) follows from $\hat{\theta}_{\hat{\pi}}\in\Theta^{\hat{\pi}}$ and (b) from~\eqref{eq:up-bound-phi}.
In other words, if we have an $\epsilon$-optimal query policy, we can construct an estimator $\hat{\theta}$ such that $\|\hat{\theta} - \theta\|^2 \leq 8\sqrt{d}\mu\epsilon$.

Now, taking again the arbitrary query policy $\hat{\pi}$, we have
\begin{equation}
\label{eq:eq-0E}
\begin{aligned}
    \EE_{\theta}[\|\hat{\theta}_{\hat{\pi}} - \theta\|^2] 
    &= \EE_\theta[\|\hat{\theta}_{\hat{\pi}} - \theta\|^2\I[\hat{\pi}\textrm{ is $\epsilon$-optimal for $\theta$}]+\|\hat{\theta}_{\hat{\pi}} - \theta\|^2\I[\hat{\pi}\textrm{ is not $\epsilon$-optimal for $\theta$}]]\\
    &\overset{(a)}{\leq} 8\sqrt{d}\mu\epsilon+4d\mu^2\;\EE_\theta[\I[\hat{\pi}\textrm{ is not $\epsilon$-optimal for $\theta$}]]\\
    &\overset{(b)}{=} \frac{8d\epsilon}{\sqrt{700KP}} + \frac{4d^2}{700KP} {\Pr}_{\theta}[\{\hat{\pi}\textrm{ is not  $\epsilon$-optimal for $\theta$}\}],
\end{aligned}
\end{equation}
where (a) follows from~\eqref{eq:phi-epsopt} in the first term and $\norm{\theta_1-\theta_2}\leq2\sqrt{d}\mu$ for any $\theta_1,\theta_2\in\Phi$ in the second term; and where (b) follows from plugging in $\mu = \sqrt{d/700KP}$.

Now, from Lemma~\ref{lem:hard_linear_bandit}, we have that our algorithm, since it runs until the stopping time $\tau$, $\inf_{\hat{\theta}, \pi}\max_{\theta \in \Theta}\EE_{\theta}[\|\hat{\theta} - \theta\|_2^2]
\geq\frac{d\mu^2}{2}\rbr{1 - \sqrt{\frac{20\tau P\mu^2}{d}}}$. Now, if we assume that $\tau\leq K$, i.e., that we collect no more than $K$ batches of $P$ arms, we obtain
\begin{equation}
\label{eq:eq-1E}
\begin{aligned}
\max_{\theta \in \Theta}\EE_{\theta}[\|\hat{\theta}_{\hat{\pi}} - \theta\|_2^2] 
\geq\inf_{\hat{\theta}, \pi}\max_{\theta \in \Theta}\EE_{\theta}[\|\hat{\theta}_{\hat{\pi}} - \theta\|_2^2]
\geq\frac{d\mu^2}{2}\rbr{1 - \sqrt{\frac{20KP\mu^2}{d}}}   \overset{\textrm{(a)}}{\geq} 0.00059d^2/KP.
\end{aligned}
\end{equation}
where (a) follows from using $KP \geq d^2$ and $\mu = \sqrt{d/700KP}$. 

Then, assuming that $\theta=\theta^\star\in \arg\max_{\theta \in \Theta}\EE_{\theta}[\|\hat{\theta}_{\hat{\pi}} - \theta\|_2^2]$, 
\eqref{eq:eq-0E} and~\eqref{eq:eq-1E} would not be contradicted if the following holds,  
\[
    0.000059\frac{d^2}{KP} \leq \frac{8d\epsilon}{\sqrt{700KP}} + \frac{4d^2}{700KP}{\Pr}_{\theta}[\{\hat{\pi}\textrm{ is not  $\epsilon$-optimal for $\theta^\star$}\}].
\] 
Manipulating this inequality shows
that if 
$$
\left(\frac{0.00325}{2\sqrt{700}}\right)^2\frac{d^2}{\epsilon^2}\geq KP
$$
then 
$$
{\Pr}_{\theta}[\{\hat{\pi}\textrm{ is not  $\epsilon$-optimal for $\theta^\star$}\}]\geq 0.1.
$$
In other words, unless $KP= \Omega(d^2/\epsilon^2)$, we have  $\Pr[\{\hat{\pi}\textrm{ is not  $\epsilon$-optimal for $\theta^\star$}\}] \geq 0.1$.
\end{proof}

We now must relate the best arm identification problem in the  linear bandit setting to linear MDPs. 
For this, we can relate any instance of 
the linear bandit problem (defined by an instance of the parameter $\theta\in\Theta$) %considered here is an instance of linear MDP. to
to an instance of a linear MDP. 
The idea is that the task of identifying  the optimal policy in the constructed linear MDP reduces to identifying the best query strategy in the linear bandit setting. 
We refer readers interested in the exact construction of the linear MDP to the work~\citep{AW-YC-MS-SSD-KJ:22}[Lemma D.3] and summarize its results in the following lemma.

\begin{lemma}[Lemma D.3, D.4 from \citep{AW-YC-MS-SSD-KJ:22}, restated]
\label{lemma:reduction-lin-mdp-to-lin-bandit}

Let $\epsilon>0$, $P>0$, $d>1$, and $KP\geq d^2$. Also, let $\theta \in \Theta = \{-\mu, \mu\}^d$ 
with $\mu=\sqrt{\frac{d}{700KP}}$
%be arbitrary and fixed for some $\mu \in \left(0, \frac{1}{20\sqrt{d}}\right]$ 
. There exists a $(d + 1)$-dimensional linear MDP with horizon $H$ such that, starting at some particular initial state, any $\epsilon$-suboptimal policy for the MDP can be converted to an $\epsilon/H$-suboptimal query policy in the single linear bandit setting. 
\end{lemma}
\begin{proof}
We refer readers interested in a detailed proof to Appendix D.1 of~\citep{AW-YC-MS-SSD-KJ:22}.
\end{proof}

Note that Lemma~\ref{lemma:reduction-lin-mdp-to-lin-bandit} constructs the MDP for the \emph{single} linear bandit setting; however, in the same lemma's statement we have that the parameter $\mu$ has $P$ in the denominator (instead of $1$ as in the original work by~\cite{AW-YC-MS-SSD-KJ:22}).  
This is not a problem since having $P\geq1$ does not change the validity of the construction of the MDP stated in~\citep{AW-YC-MS-SSD-KJ:22} -- indeed, in the construction by~\cite{AW-YC-MS-SSD-KJ:22}, the denominator of $\mu$ serves only as a scaling to the unknown signed measure $\mu_1$ (which defines a linear MDP; see Section~\ref{sec:preliminaries}) so that it satisfies the condition $\norm{\mu_1(\S)}\leq\sqrt{d}$.

We now describe a  parallel algorithm with $P$ agents for a given policy at iteration $k\in[K]$. At the beginning of the iteration, we let each agent $p\in[P]$ access the MDP in parallel from the specific initial condition stated by~\cite{AW-YC-MS-SSD-KJ:22} and take an action according to the given policy (the policy is free to define how each agent takes its own action but must make sure each agent takes actions independently)  -- since all start in the same initial state, they will be playing the underlying parallel linear bandit and thus obtain the rewards $\{y^{k,p}\}_{p\in[P]}$ (because the construction by~\citep{AW-YC-MS-SSD-KJ:22} only allows access to the underlying (single) linear bandit only for the transition out of the initial state). Each agent continues playing their MDP independently according to the construction by~\cite{AW-YC-MS-SSD-KJ:22} for all the rest of the episodes (which requires no further access to the underlying linear bandit). At the end of the episode, the parallel algorithm can decide whether to make a change to the policy based on the results obtained by all the parallel agents and use it for the next iteration. After iterating all the $K$ episodes, the algorithm outputs some policy, whose decision for the transition out of the initial state for all agents would be equivalent to $\phi^{K+1}$, i.e., the output of a query model for the parallel linear bandit after $k$ steps. 

With Lemma~\ref{lemma:reduction-lin-mdp-to-lin-bandit} in mind, we know that identifying $\epsilon$-suboptimal policies in the MDP is at least as hard as identifying $\epsilon/H$-suboptimal query strategies in the single linear bandit. Previously, we showed that identifying a $\epsilon$-suboptimal query strategy in the parallel linear bandit setting with $P$ processors and $K$ steps requires $KP = \Omega(d^2/\epsilon^2)$. Replacing the $\epsilon$ with $\epsilon/H$ gives us the lower bound for linear MDP.

\end{document}